\documentclass[twoside,11pt]{article}

\usepackage{blindtext}
\usepackage{amsmath}
\usepackage{bm}
\usepackage{bbm}

\usepackage{newtxtext}
\usepackage[subscriptcorrection]{newtxmath}

\usepackage[utf8]{inputenc}
\usepackage[T1]{fontenc}
\usepackage[english]{babel}
\usepackage{natbib}
\usepackage{graphicx}
\usepackage{enumerate}
\usepackage{url}
\usepackage{multirow}
\usepackage{url}
\usepackage{booktabs}
\usepackage{nicefrac}
\usepackage{microtype}
\usepackage{xcolor}
\usepackage{mathtools}
\usepackage{mathrsfs}
\usepackage{algorithm,algorithmic}

\newtheorem{theorem}{Theorem}
\newtheorem{lemma}{Lemma} 
\newtheorem{proposition}[theorem]{Proposition}
\newtheorem{remark}{Remark}

\newtheorem{definition}{Definition}

\usepackage[preprint]{jmlr2e}

\usepackage{lastpage}

\ShortHeadings{ReLU integral probability metric and its applications}{ReLU integral probability metric and its applications}
\firstpageno{1}

\begin{document}

\title{ReLU integral probability metric and its applications}

\author{\name Yuha Park \email yuha.park@uni-hamburg.de \\
       \addr Department of Mathematics\\
       University of Hamburg\\
       Mittelweg 177, Hamburg, Germany
       \AND
       \name Kunwoong Kim \email kwkim.online@gmail.com \\
       \addr Department of Statistics\\
       Seoul National University\\
       1 Gwanak-ro, Gwanak-gu, Seoul, Republic of Korea
       \AND
       \name Insung Kong \email insung.kong@utwente.nl \\
       \addr Department of Applied Mathematics\\
       University of Twente\\
       Drienerlolaan 5, 7522 NB Enschede, Netherlands
       \AND
       \name Yongdai Kim \email ydkim0903@gmail.com \\
       \addr Department of Statistics\\
       Seoul National University\\
       1 Gwanak-ro, Gwanak-gu, Seoul, Republic of Korea
       }

\editor{}

\maketitle

\begin{abstract}
We propose a parametric integral probability metric (IPM) to measure the discrepancy between two probability measures. The proposed IPM leverages a specific parametric family of discriminators, such as single-node neural networks with ReLU activation, to effectively distinguish between distributions, making it applicable in high-dimensional settings. By optimizing over the parameters of the chosen discriminator class, the proposed IPM demonstrates that its estimators have good convergence rates and can serve as a surrogate for other IPMs that use smooth nonparametric discriminator classes.
We present an efficient algorithm for practical computation, offering a simple implementation and requiring fewer hyperparameters. Furthermore, we explore its applications in various tasks, such as covariate balancing for causal inference and fair representation learning. Across such diverse applications, we demonstrate that the proposed IPM provides strong theoretical guarantees, and empirical experiments show that it achieves comparable or even superior performance to other methods.
\end{abstract}



\section{Introduction}

The concept of measuring the distance between two probability measures is useful across various domains, including statistical inference, machine learning, and probability theory.
The \textit{Integral Probability Metrics} (IPMs), in particular, have emerged as powerful tools \citep{muller1997integral, Sriperumbudur2009IPM, sriperumbudur2012} for assessing discrepancies between two probability measures, enabling significant advancements in various applications such as generative modeling, statistical testing, and domain adaptation. For example, the Wasserstein distance is used in image generation \citep{arjovsky2017wasserstein, gulrajani2017improved}, image retrieval \citep{rubner2000earth}, reinforcement learning \citep{bellemare2017distributional}, distributionally robust optimization \citep{blanchet2019robust}, and optimal transport problems \citep{villani2009optimal, NIPS2013_af21d0c9}. The maximum mean discrepancy (MMD) is  used in statistical testing \citep{gretton2012kernel}, domain adaptation \citep{NIPS2012_ca8155f4, 8099590}, generative modeling \citep{pmlr-v37-li15, dziugaite2015training, li2017mmd} and causal inference \citep{Wong2017sobolev, pmlr-v202-kong23d}.
The H\"{o}lder-IPM is used for statistical testing \citep{Wang2023manifold, tang2023} and the Sobolev-IPM is used for image generation \citep{mroueh2017sobolev} and causal inference \citep{Wong2017sobolev}. Additionally, a parametric IPM (an IPM whose discrimnator class is parametric)
is applied to fair representation learning tasks \citep{pmlr-v162-kim22b} and causal effect estimation \citep{pmlr-v202-kong23d}.

Let $\mathcal{F}$ be a class of (real-valued) bounded measurable functions $f$ from $\mathcal{X}$ to $\mathbb{R}$. For a given class $\mathcal{F}$, the IPM between two given probability measures $\mathcal{P}$ and $\mathcal{Q}$, denoted as $d_{\mathcal{F}}(\mathcal{P}, \mathcal{Q})$,  is defined as the maximum average difference between $\mathcal{P}$ and $\mathcal{Q}$ on the class $\mathcal{F}$ over ${\mathcal {X}}$,
\begin{equation} \label{equ:IPM}
    d_{\mathcal{F}}(\mathcal{P}, \mathcal{Q}) \coloneqq \sup_{f \in \mathcal{F}} \left| \int f(d\mathcal{P}-d\mathcal{Q}) \right| = \sup_{f \in \mathcal{F}} \Big| \mathbb{E}_{\mathbf{X}} f(\mathbf{X})-\mathbb{E}_{\mathbf{Y}} f(\mathbf{Y}) \Big|,
\end{equation}
when $\mathbf{X}\sim \mathcal{P}$ and $\mathbf{Y}\sim\mathcal{Q}.$
Here, the functions $f$ are referred to as \textit{discriminators}, \textit{critics} or \textit{adversaries}, depending on the context of their specific applications.

Different choices of the discriminator class $\mathcal{F}$ result in different IPMs.
For example, the MMD is defined by selecting a reproducing kernel Hilbert space (RKHS) 
whereas the Wasserstein distance uses the set of 1-Lipschitz functions as its discriminator class. 
The H\"{o}lder-IPM and Sobolev-IPM select the H\"{o}lder and Sobolev spaces, respectively, as their discriminator classes. 
In turn, different IPMs affect the final estimator differently. A larger discriminator class can identify more complex differences between two probability measures, and thus more complex patterns could be learned to yield less bias in the final estimator.
In contrast, as usual, large discriminator classes would lead to over-fitting and thus higher variance. 
Moreover, numerical computation of IPMs with a larger discriminator class is more difficult.
Thus, there is a need to find a discriminator class which is small but is able to identify complex patterns between two probability measures.

\citet{pmlr-v162-kim22b} consider a parametric discriminator class
$\{\textup{sig}(\boldsymbol{\theta}^{\top}\mathbf{z} + \mu): \mu\in \mathbb{R}, \boldsymbol{\theta} \in \mathbb{R}^d\},$
where $\textup{sig}(\cdot)$ is the sigmoid function (i.e. $\textup{sig}(z)=(1+\exp(-z))^{-1}$),
whose corresponding IPM is called the sigmoid-IPM (SIPM).
An interesting property of the SIPM is that it is theoretically equivalent to an IPM with a certain nonparametric discriminator class (e.g. the set of all infinitely differentiable functions).
\citet{pmlr-v162-kim22b} and \citet{pmlr-v202-kong23d} empirically demonstrate that the SIPM competes well with existing nonparametric IPMs for fair representation learning and causal inference, respectively.

The aim of this paper is to propose a new parametric IPM, called the ReLU-IPM, where
the discriminator class is given as $\{ (\boldsymbol{\theta}^{\top}\mathbf{z} + \mu)_+: \mu \in [-1,1], \, \boldsymbol{\theta} \in \mathbb{S}^{d-1} \},$
where $(\cdot)_+ = \max(\cdot, 0)$ denotes the ReLU activation function and $\mathbb{S}^{d-1}$ is the unit sphere in $\mathbb{R}^d.$
We provide an interesting theoretical relation between the ReLU-IPM and the H\"{o}lder-IPM in that the ReLU-IPM is an upper bound of the H\"{o}lder-IPM. Moreover, it is proved that the ReLU-IPM and H\"{o}lder-IPM are asymptotically equivalent when the H\"{o}lder discriminators are sufficiently smooth. These theoretical results imply that the ReLU-IPM is essentially a nonparametric IPM even if the discriminator class is parametric.

An advantage of the H\"{o}lder-IPM is to control the size of the discriminator class by
controlling the smoothness of the H\"{o}lder discriminators.
In contrast, the discriminator classes of the Wasserstein distance, which include all of differentiable functions, would be too large while the discriminator class of the MMD, which consists of only very smooth functions
(e.g. infinitely differentiable functions for the MMD with the Gaussian kernel), would be too small.
However, computation of the H\"{o}lder-IPM is difficult since H\"{o}lder smooth discriminators should be approximated
by computable models such as deep neural networks \citep{Wang2023manifold}. The choice of an appropriate architecture of DNNs for approximating H\"{o}lder smooth discriminators would be hard, which makes
the H\"{o}lder-IPM be underappreciated. Our theoretical results suggest that 
the ReLU-IPM can be used as a surrogate IPM for the H\"{o}lder-IPM.

To illustrate the empirical superiority of the ReLU-IPM,  we consider the causal inference problem and fair representation learning. For the causal problem, we use the ReLU-IPM in the covariate balancing (CB) estimator proposed by \citet{pmlr-v202-kong23d} and derive the convergence rate of the estimated average treatment effect on treated (ATT).
In particular, when the output regression model is sufficiently smooth, the CB estimator of ATT with the ReLU-IPM achieves a parametric rate.
By analyzing simulated data, we demonstrate that the CB estimator with the ReLU-IPM outperforms other competitors, including the SIPM.

For fair representation learning, we consider using the ReLU-IPM to measure the degree of fairness of a learned fair representation. We prove that the final prediction model constructed on the top of the learned fair representation is guaranteed to be fair provided that the prediction head (a prediction model from the representation space to the output space) is H\"{o}lder-smooth. A similar result for the SIPM has been obtained by \citet{pmlr-v162-kim22b} when the prediction head is infinitely differentiable. In this view, the ReLU-IPM can be considered as an extension of the SIPM for H\"{o}lder smooth prediction heads.
By analyzing several benchmark datasets, we demonstrate that the fair representation learning algorithm with the ReLU-IPM is compared favorably with various baseline algorithms.

The remainder of this paper is organized as follows: Section \ref{sec:pre} provides a detailed review of IPMs, discussing their definitions, applications, and limitations in the context of machine learning. We introduce our proposed metrics, ReLU-IPM, in Section \ref{sec:ReLUIPM} and present its applications to the following tasks: causal inference and fair representation learning in Sections \ref{sec:app3} and \ref{sec:app2}, respectively.
Finally, Section \ref{sec:conc} concludes the paper and suggests directions for future research.

\vspace{2pt}
\begin{paragraph}{Notations}
    We use bold letters to denote vectors and matrices. 
    Let $\mathbb{N}=\{1,2,\ldots \}$ be the set of natural numbers.
    For a given two sequence $\{a_n\}$ and $\{b_n\}$, the notation $a_n \lesssim b_n$ and $a_n \gtrsim b_n$ indicates that there exists a positive constant $c$ such that for all $n$, $a_n \leq cb_n$ and $a_n \geq cb_n$, respectively. 
    When $a_n \lesssim b_n$ and $a_n \gtrsim b_n$, we write $a_n \asymp b_n$.
    For $a \in \mathbb{R}$, let $\lfloor a \rfloor$ and $\lceil a \rceil$ denote the floor and ceiling functions, which round $a$ to the next strictly smaller and larger integer, respectively.
    We denote $\mathbf{0}$ and $\mathbf{1}$ as the vectors of all zeros and all ones, respectively, in the corresponding dimensional space.
    For a given condition or event $A$, $\mathbb{I}(A)$ denotes the indicator function, where $\mathbb{I}(A) = 1$ if $A$ is true and $\mathbb{I}(A) = 0$ if $A$ is false. 
\end{paragraph}


\section{Review of existing IPMs}
\label{sec:pre}

In this section, we review popularly used IPMs.

\begin{paragraph}{Wasserstein distance}
    Given two distributions $\mathcal{P}$ and $\mathcal{Q}$, the p-Wasserstein distance is defined as
    \begin{equation} \label{eq:wass}
    W_p(\mathcal{P}, \mathcal{Q}) \coloneqq \left( \inf_{\pi \in \Pi(\mathcal{P}, \mathcal{Q})} \int_{\mathcal{X} \times \mathcal{X}} d(\mathbf{x}, \mathbf{y})^p d\pi(\mathbf{x}, \mathbf{y}) \right)^{1/p},
    \end{equation}
    where $d(\mathbf{x}, \mathbf{y})$ is a distance between two points $\mathbf{x}$ and $\mathbf{y}$, and $\Pi(\mathcal{P}, \mathcal{Q})$ is the set of all joint probability measures (couplings) $\pi(\mathbf{x},\mathbf{y})$ with the marginals $\mathcal{P}$ and $\mathcal{Q}$.
    By the Kantorovich-Rubinstein duality \citep{kantorovich1958space, villani2009optimal}, the 1-Wasserstein distance can be expressed as the IPM with a discriminator class containing all bounded 1-Lipschitz functions.
    That is, the 1-Wasserstein distance can be redefined as $d_{\mathcal{F}_{L}}(\mathcal{P}, \mathcal{Q}) \coloneqq \sup_{f \in \mathcal{F}_{L}} | \mathbb{E}_{\mathbf{X}} f(\mathbf{X})-\mathbb{E}_{\mathbf{Y}} f(\mathbf{Y})|$, where $\mathcal{F}_{L} \coloneqq \left\{ f: \|f\|_L \leq 1 \right\}$
    \footnote{ $\|f\|_L \coloneqq \sup_{\mathbf{x} \neq \mathbf{y}}\frac{|f(\mathbf{x})-f(\mathbf{y})|}{\|\mathbf{x}-\mathbf{y}\|}$ is the Lipschitz semi-norm of a real-valued function $f: \mathcal{X} \rightarrow \mathbb{R}$, where $\|\cdot\|$ is a certain norm defined on $\mathcal{X}$ (i.e., $\|f\|_L$ is the Lipschitz constant of the function $f$). }.
    Originally derived from the theory of optimal transport, it measures the minimum cost required to transform one distribution into another, considering both the distance and the amount of mass that must be moved. 
    Thus, the Wasserstein distance captures geometric differences between two distributions.
    
    Recently, the Wasserstein distance has gained popularity in the domain of image generation \citep{arjovsky2017wasserstein, NEURIPS2019_c9e1074f}, where it helps to improve the stability and convergence of the training process by addressing issues such as mode collapse.
    Furthermore, the Wasserstein distance has been used in various tasks, including domain adaptation \citep{Shen_Qu_Zhang_Yu_2018, xu2019wasserstein}, object detection \citep{xu2018multi, NEURIPS2020_fe7ecc4d, pmlr-v139-yang21l}, statistical testing \citep{63640154-a4b4-3ee5-9180-1c20fb2afd23, 10.1214/21-EJS1816}, and representation learning \citep{NEURIPS2019_f9209b78}.
    The Wasserstein distance can be computed using an algorithm that employs a DNN to approximate 1-Lipschitz functions, as suggested by \citet{arjovsky2017wasserstein} and \citet{gulrajani2017improved}. However, approximating 1-Lipschitz functions is challenging since the choice of architectures would not be easy and the optimization would be unstable.
\end{paragraph}

\begin{paragraph}{Maximum mean discrepancy (MMD)}
    The discriminator class for MMD is $\mathcal{H}_{k} \coloneqq \left\{  f: \|f\|_{\mathcal{H}_k} \leq 1 \right\}$, where $\mathcal{H}_k$ is the reproducing kernel Hilbert space associated with a given kernel $k$.
    That is, the MMD is defined as $d_{\mathcal{H}_{k}}(\mathcal{P}, \mathcal{Q}) \coloneqq \sup_{f \in \mathcal{H}_{k}} | \mathbb{E}_{\mathbf{X}} f(\mathbf{X})-\mathbb{E}_{\mathbf{Y}} f(\mathbf{Y})|$. 
    An important advantage of the MMD is that it can be calculated in a closed form.
    This leads to the well-known expression of the MMD in terms of a given kernel function. i.e., 
    $$
    d_{\mathcal{H}_{k}}(\mathcal{P}, \mathcal{Q}) 
    = \mathbb{E}k(\mathbf{X}, \mathbf{X}^{\prime}) - 2\mathbb{E}k(\mathbf{X}, \mathbf{Y}) + \mathbb{E}k(\mathbf{Y}, \mathbf{Y}^{\prime}).
    $$
    The MMD is used in statistical testing \citep{gretton2012kernel}, generative models \citep{pmlr-v37-li15, dziugaite2015training, li2017mmd} such as generative adversarial networks (GANs), and domain adaptation \citep{NIPS2012_ca8155f4, 8099590}.
    The MMD is computationally efficient, easy to implement, and generally performs well for high-dimensional data. However, it is sensitive to the choice of kernel functions. The Gaussian kernel is popularly used \citep{vert2004primer, gretton2012kernel, NIPS2016_5055cbf4, garreau2018largesampleanalysismedian, li2024optimality},
    but the corresponding RKHS, which includes only infinitely differentiable functions \citep{micchelli2006universal, kroese2019data}, would be too small and the choice of bandwidths in practice would not be easy. 
\end{paragraph}

\begin{paragraph}{H\"{o}lder-IPM}
    The H\"{o}lder-IPM \citep{Wang2023manifold, tang2023} uses the H\"{o}lder function class,
    $
    \mathcal{H}_{d}^{\beta} \coloneqq \{ f: \|f\|_{\mathcal{H}_{d}^{\beta}} \leq 1 \},
    $
    for the discriminator class, 
    where the H\"{o}lder norm is defined as
    $$
    \|f\|_{\mathcal{H}_{d}^{\beta}} = \sum_{\|\bm{\alpha}\|_1=\lfloor \beta \rfloor} \max_{x,y \in \Omega, x \neq y} \frac{|f^{(\bm{\alpha})}(x)-f^{(\bm{\alpha})}(y)|}{\|x-y\|_2^{\beta - \lfloor \beta \rfloor}} + \sum_{\|\bm{\alpha}\|_1 \leq \lfloor \beta \rfloor} \max_{x \in \Omega} |f^{(\bm{\alpha})}(x)|,
    $$
    where $f^{(\bm{\alpha})}=\frac{\partial^{\|\bm{\alpha}\|_1} f}{\partial x_1^{\alpha_1} \partial x_2^{\alpha_2} \ldots \partial x_d^{\alpha_d}}$ for $\bm{\alpha} = (\alpha_1, \alpha_2, \ldots, \alpha_d)^{\top}$.
    That is, the H\"{o}lder-IPM is defined as $d_{\mathcal{H}_{d}^{\beta}}(\mathcal{P}, \mathcal{Q}) \coloneqq \sup_{f \in {\mathcal{H}_{d}^{\beta}}} | \mathbb{E}_{\mathbf{X}} f(\mathbf{X})-\mathbb{E}_{\mathbf{Y}} f(\mathbf{Y})|$.
    \citet{Wang2023manifold} considers the H\"{o}lder-IPM for statistical testing problems.
    \citet{Wang2023manifold} states that practical applications of the H\"{o}lder-IPM become feasible through the approximation of H\"{o}lder smooth functions by the use of neural networks because it is well known that neural networks are good at approximating H\"{o}lder smooth functions \citep{ohn2019smooth, Yang2024}.
    However, the choice of neural network architectures is difficult.
\end{paragraph}

\begin{paragraph}{Sigmoid-IPM (SIPM)}
    The SIPM introduced by \citet{pmlr-v162-kim22b} is an IPM with a parametric discriminator class, where the discriminator class $\mathcal{F}_{\textup{sig}}$ is given as 
    $$
    \mathcal{F}_{\textup{sig}} \coloneqq \left\{ f: f(\mathbf{z}) = \operatorname{sig}\left(\boldsymbol{\theta}^{\top}\mathbf{z} + \mathbf{\mu}\right), \, \boldsymbol{\theta} \in \mathbb{R}^{d}, \, \mathbf{\mu} \in \mathbb{R} \right\}.
    $$ 
    That is, the SIPM is defined as $d_{\mathcal{F}_{\textup{sig}}}(\mathcal{P}, \mathcal{Q}) \coloneqq \sup_{f \in \mathcal{F}_{\textup{sig}}} | \mathbb{E}_{\mathbf{X}} f(\mathbf{X})-\mathbb{E}_{\mathbf{Y}} f(\mathbf{Y})|$.
    The development and application of the SIPM are well-studied in fair representation learning (FRL) \citep{pmlr-v162-kim22b} and  causal inference \citep{pmlr-v202-kong23d}. An interesting theoretical property of the SIPM obtained by \citet{pmlr-v162-kim22b} is that the SIPM is equivalent to the IPM whose discriminator class is the set of all infinitely differentiable functions. That is, even if the SIPM uses a parametric discriminator class,
    it covers a nonparametric discriminator class.
\end{paragraph}


\section{ReLU-IPM: New parametric IPM}
\label{sec:ReLUIPM}

Our goal is to develop an IPM with a parametric discriminator class $\mathcal{F}$ that is simple but effective. For $\mathcal{F}$ to be a suitable choice, it must be computationally efficient while maintaining sufficient flexibility to capture a broad range of discriminators larger than $\mathcal{F}.$
An example of such an IPM is the SIPM.
In this section, we propose a new parametric discriminator class that is similar to $\mathcal{F}_{\textup{sig}}$, but the sigmoid activation function is replaced by the ReLU activation function.
In the context of deep neural networks, it is well-known that the ReLU activation not only facilitates faster and more effective training of deep neural networks, especially on large and complex datasets, but also mitigates the vanishing gradient problem. Moreover, it exhibits several advantageous properties, including \textit{sparsity}, \textit{scale-invariance}, and \textit{efficient computation} \citep{Glorot2011, Ramachandran2017}.

\subsection{Definition}
\label{sec:def_relu}

We consider the discriminator class $\mathcal{F}_{\textup{ReLU}}$ defined on the unit ball $\mathbb{B}^d = \{\mathbf{z} \in \mathbb{R}^d : \|\mathbf{z}\|_2 \leq 1 \}$ of $\mathbb{R}^d$ given as
\begin{equation}
    \mathcal{F}_{\textup{ReLU}} \coloneqq \left\{ f :\mathbb{B}^d \rightarrow \mathbb{R} \, | \, f(\mathbf{z}) = (\boldsymbol{\theta}^{\top}\mathbf{z} + \mu)_+, \, \mu \in [-1,1], \, \boldsymbol{\theta} \in \mathbb{S}^{d-1} \right\},
\end{equation}
where  $\mathbb{S}^{d-1} = \{\mathbf{z} \in \mathbb{R}^d : \|\mathbf{z}\|_2 = 1 \}$ is the unit sphere in $\mathbb{R}^{d}.$ 
Then we define the ReLU-IPM as
\begin{equation*}
    d_{\mathcal{F}_{\textup{ReLU}}}(\mathcal{P},\mathcal{Q}) = \sup_{f \in \mathcal{F}_{\textup{ReLU}}} \Big| \mathbb{E}_{\mathbf{X} \sim \mathcal{P}} f(\mathbf{X})-\mathbb{E}_{\mathbf{Y} \sim \mathcal{Q}} f(\mathbf{Y}) \Big|
\end{equation*}
for the two probability measures $\mathcal{P}$ and $\mathcal{Q}$ defined on $\mathbb{B}^d$.

By replacing the population expectations with the empirical counterparts based on the samples $\mathbf{X}^{n}=\{\mathbf{X}_1,\ldots,\mathbf{X}_n\}$ and $\mathbf{Y}^{m}=\{\mathbf{Y}_1,\ldots,\mathbf{Y}_m\}$, it is natural to define the empirical estimate of the ReLU-IPM as
\begin{equation} \label{ReLUIPM}
d_{\mathcal{F}_{\textup{ReLU}}}(\widehat{\mathcal{P}}_n, \widehat{\mathcal{Q}}_m) \coloneqq \sup_{f \in \mathcal{F}_{\textup{ReLU}}} \left| \frac{1}{n} \sum_{i=1}^n f(\mathbf{X}_i) - \frac{1}{m} \sum_{i=1}^m f(\mathbf{Y}_i) \right|,
\end{equation}
where $\widehat{\mathcal{P}}_n(\cdot) = \frac{1}{n} \sum_{i=1}^n \delta_{\mathbf{X}_i}(\cdot)$ and $\widehat{\mathcal{Q}}_m(\cdot) = \frac{1}{m} \sum_{i=1}^m \delta_{\mathbf{Y}_i}(\cdot)$.
Here, $\delta_x$ denotes the Dirac measure on $x$.

\subsection{Theoretical Studies of ReLU-IPM}
\label{sec:theo_relu}

The first theoretical property is that the ReLU-IPM can perfectly distinguish any two probability measures which is stated in the following proposition. The proof of Proposition \ref{prop:discriminative} is deferred to Appendix \ref{pf:prop1}.

\begin{proposition} \label{prop:discriminative}
$d_{\mathcal{F}_{\textup{ReLU}}}(\mathcal{P}, \mathcal{Q})=0 \text{ if and only if } \mathcal{P} \equiv \mathcal{Q}$ for two probability measures $\mathcal{P}$ and $\mathcal{Q}$.
\end{proposition}

\begin{remark} \label{remark:1}
    It would not be very surprising that a parametric IPM perfectly distinguishes any two probability measures.
    The characteristic function distance (CFD), which is the IPM with discriminator class
$\mathcal{F}_{\textup{cf}} = { e^{i \mathbf{t}^{\top}\mathbf{x}}: \mathbf{t} \in \mathbb{R}^d }$, is introduced by \citet{ansari2020characteristic}.    
    It is well known that $d_{\mathcal{F}_{\textup{cf}}}(\mathcal{P}, \mathcal{Q})=0$ if and only if $\mathcal{P} = \mathcal{Q}.$
    However, \citet{mccullagh1994does} notices that $d_{\mathcal{F}_{\textup{cf}}}$ would not be a useful metric between two probability measures.
\end{remark}

The following theorem is the main result of this paper which describes the relationship between the ReLU-IPM and H\"{o}lder-IPM.
In Theorem \ref{thm:holder} below, we prove that the H\"{o}lder-IPM is upper bounded by the ReLU-IPM. 
This result implies that the ReLU-IPM can be used as a surrogate for the H\"{o}lder-IPM. That is, we can reduce the H\"{o}lder-IPM by reducing the ReLU-IPM. The proof of Theorem \ref{thm:holder} is presented to Appendix \ref{pf:holder}.

\begin{theorem}[ReLU-IPM bounds H\"{o}lder-IPM] \label{thm:holder}
    Let $d \in \mathbb{N}$ and $\beta > 0$ be given. For some positive constant $c$ that depends on $d$ and $\beta$, we have \\
    (1) if $\beta < \frac{d+3}{2}$, 
    $$
    d_{\mathcal{H}^{\beta}_d}(\mathcal{P},\mathcal{Q}) \leq c \cdot d_{\mathcal{F}_{\textup{ReLU}}}(\mathcal{P},\mathcal{Q})^{\frac{2\beta}{d+3}},
    $$
    and 
    (2) if $\beta > \frac{d+3}{2}$,
    $$
    d_{\mathcal{H}^{\beta}_d}(\mathcal{P},\mathcal{Q}) \leq c \cdot d_{\mathcal{F}_{\textup{ReLU}}}(\mathcal{P},\mathcal{Q}),
    $$
    for any two probability measures $\mathcal{P}$ and $\mathcal{Q}$ in $\mathbb{B}^d$.
\end{theorem}

\begin{remark} 
    We do not consider the case $\beta = \frac{d+3}{2}$, since the upper bound depends on whether $\beta$ is an even integer. Refer to Appendix \ref{pf:holder} for discussions.
\end{remark}

By the definition of IPMs, it is obvious that an IPM with a larger discriminator class has a larger value than IPMs with smaller discriminator classes.
In fact, from $\mathcal{F}_{\textup{ReLU}} \subset \mathcal{H}^{1}_d$ we have
$d_{\mathcal{F}_{\textup{ReLU}}}(\mathcal{P}, \mathcal{Q}) \leq d_{\mathcal{H}_d^{\beta}}(\mathcal{P}, \mathcal{Q})$ for any $\beta \leq 1.$
In contrast, Theorem \ref{thm:holder} provides an opposite result. That is,
a small value of the ReLU-IPM  guarantees a small value of the H\"{o}lder-IPM.
It is remarkable that the IPM utilizing the parametric class $\mathcal{F}_{\textup{ReLU}}$, which is significantly smaller than the nonparametric class $\mathcal{H}^{\beta}_d$, has such a property.

Since the ReLU-IPM uses a parametric discriminator class, it has many desirable statistical properties.
One of such desirable properties is a fast convergence of the empirical ReLU-IPM to its population version
as stated in the following theorem, whose proof is given in Appendix \ref{pf:conv_relu}.

\begin{theorem}[Rate of convergence of the empirical ReLU-IPM] \label{thm:conv_relu}
    Let $\mathcal{P}, \mathcal{Q}, \widehat{\mathcal{P}}_n, \widehat{\mathcal{Q}}_n$ be defined on $\mathbb{B}^d$. Then, for a constant $c$,
    $$
    \Big| d_{\mathcal{F}_{\textup{ReLU}}}(\mathcal{P},\mathcal{Q}) - d_{\mathcal{F}_{\textup{ReLU}}}(\widehat{\mathcal{P}}_n, \widehat{\mathcal{Q}}_n) \Big| \leq \frac{c}{\sqrt{n}}  + \epsilon
    $$
    with probability at least $1 - 2 \exp \left(-\frac{\epsilon^2 n}{16} \right)$. Furthermore, $d_{\mathcal{F}_{\textup{ReLU}}}(\widehat{\mathcal{P}}_n, \widehat{\mathcal{Q}}_n)$ converges almost surely to $d_{\mathcal{F}_{\textup{ReLU}}}(\mathcal{P},\mathcal{Q})$, i.e.,
    $$
    \Big| d_{\mathcal{F}_{\textup{ReLU}}}(\mathcal{P},\mathcal{Q}) - d_{\mathcal{F}_{\textup{ReLU}}}(\widehat{\mathcal{P}}_n, \widehat{\mathcal{Q}}_n) \Big| \overset{a.s.}{\longrightarrow} 0 \quad \text{as} \quad n \to \infty.
    $$
\end{theorem}

Note that the convergence rate is  parametric and not depending on the dimension $d$ of data, which is a highly desirable property. In contrast, the convergence rate of the empirical Wasserstein distance is $\mathcal{O}(n^{-\frac{1}{2}} \log n)$ when $d=1$ but $\mathcal{O}(n^{-\frac{1}{d+1}})$ when $d \geq 2$ \citep{Sriperumbudur2009IPM, sriperumbudur2012}.

\subsection{An algorithm for computing the ReLU-IPM}
\label{sec:com_relu}

For the computation of the ReLU-IPM, we use the gradient projection method. That is, given the current values of parameters $\mu^{\textup{old}}$ and $\boldsymbol{\theta}^{\textup{old}}$, we take the gradients of $\mu^{\textup{old}}$ and $\boldsymbol{\theta}^{\textup{old}}$, and update the parameters to $\mu^{\textup{new}}$ and $\boldsymbol{\theta}^{\textup{new}}$. We then keep $\mu^{\textup{new}}$ fixed but normalize $\boldsymbol{\theta}^{\textup{new}}$ to have the unit norm.
The gradient projection method often gets stuck in a local optimum.
To overcome this difficulty, we randomly select $K$ initial parameters $\{\mathbf{v}_i\}_{i=1}^K :=\{(\mu_i, \boldsymbol{\theta}_i)\}_{i=1}^K$, update each parameter using the gradient projection method until convergence, and return the highest value among the $K$ many ReLU-IPMs.

\begin{algorithm}[H]
\caption{Calculation of ReLU-IPM}
\label{alg_relu}
\begin{algorithmic}[1]
\STATE Initialize $\mathbf{v} = (\mathbf{v}_1, \ldots, \mathbf{v}_K)$
\FOR{$t = 1, 2, \ldots, n_{\textup{epoch}}$}
    \STATE Calculate 
    $L(\mathbf{v}) = \sum_{k=1}^K \left\{ \frac{1}{n} \sum_{i=1}^n f_{\mathbf{v}_k}(\mathbf{X}_i) - \frac{1}{m} \sum_{i=1}^m f_{\mathbf{v}_k}(\mathbf{Y}_i) \right\}^2$
    \STATE $\mathbf{v} \gets \mathbf{v} + \eta \nabla_{\mathbf{v}} L(\mathbf{v})$
\ENDFOR
\RETURN $\max_{1 \leq k \leq K} \left| \frac{1}{n} \sum_{i=1}^n f_{\mathbf{v}_k}(\mathbf{X}_i) - \frac{1}{m} \sum_{i=1}^m f_{\mathbf{v}_k}(\mathbf{Y}_i) \right|$
\end{algorithmic} 
\end{algorithm}


\section{Application 1: Estimation of causal effects using ReLU-IPM}
\label{sec:app3}

Estimating causal effects from observational data has become an important research topic \citep{yao2021survey}. 
The main difficulty in estimating the causal effect is that the covariate distribution of the treated group differs significantly from that of the control group. 
To adjust such systematic biases, various methods have been developed including regression adjustment \citep{hill2011bayesian}, matching \citep{stuart2010matching} and weighting \citep{imai2014covariate}.  
The common goal of these methods is to make the covariate distributions of the treated and control groups be (asymptotically) balanced.

One of the core approaches for covariate balancing is the inverse probability weighting (IPW) \citep{hirano2003efficient}.
The IPW method gives weights to observed samples reciprocally proportional to the propensity scores that are the conditional probabilities of being assigned to the treatment given covariates. 
One issue in using the IPW is that its finite sample property would be inferior  because accurately estimating the propensity score does not automatically mean the accurate estimation of its inverse \citep{li2018balancing}.

Directly estimating the weights for covariate balancing has been received much attention. 
\citet{hainmueller2012entropy} and \citet{imai2014covariate} propose to find the weights that match the sample moments of the covariates.
These methods, however, are asymptotically valid only when the output regression model is linear in covariates.
To resolve this problem, estimating the weights by use of IPMs has also been considered.
\citet{Wong2017sobolev} uses the MMD, and \citet{pmlr-v202-kong23d} provides sufficient conditions on general IPMs for the consistency of the causal effect estimation. 

In this section, we apply the ReLU-IPM to estimate the weights for covariate balancing.
We prove that the estimator of the causal effect is consistent when the true output regression model is H\"{o}lder smooth. In addition, we derive the convergence rate of the estimator of the causal effect.
By analyzing simulated data, we empirically demonstrate that the ReLU-IPM outperforms other IPMs including the SIPM and MMD for estimation of causal effect.

\subsection{Estimation of ATT}

In this study, we focus on estimating the average treatment effect on the treated (ATT), which represents the causal effect of the treatment in terms of how much the treated units benefit from it, viewed retrospectively, in the context of a binary treatment problem.

Let $\mathbf{X} \in \mathcal{X} \subset \mathbb{R}^d$ be a random vector of covariates whose distribution is denoted by $\mathcal{P}$. Let $\pi(\cdot)$ be the propensity score. 
We assume that there exists an $\eta>0$ such that $\eta < \pi(\cdot) < 1-\eta$. Also, let $T \in \{0, 1\}$ be a binary treatment indicator which is generated from $\mathrm{Ber}(\pi(\mathbf{x}))$ conditioned on $\mathbf{X} = \mathbf{x}$, and $Y \coloneqq TY(1) + (1-T)Y(0)$ be an observed outcome where $Y(0)$ and $Y(1)$ are denoted as the potential outcomes under control and treatment, respectively.

Suppose we observe $n$ independent copies $\mathcal{D}^{(n)} \coloneqq \left\{ (\mathbf{X}_i, T_i, Y_i) \right\}_{i=1}^n$ of $(\mathbf{X}, T, Y)$. Let $n_0$ and $n_1$ be the sizes of the control $(T=0)$ and treatment $(T=1)$ group, respectively. Our goal is to estimate  $\mathrm{ATT} \coloneqq \mathbb{E} \left( Y(1) - Y(0) \mid T = 1 \right)$ using $\mathcal{D}^{(n)}.$ 
We consider the weighted estimator of ATT 
$$
\widehat{\mathrm{ATT}}^{\mathbf{w}} = \sum_{i:T_i=1} \frac{1}{n_1} Y_i - \sum_{i:T_i=0} w_i Y_i
$$
parametrized by $\mathbf{w},$
where $\mathbf{w} = (w_1, \dots, w_n)^\top$ is a nonnegative weight vector. 
We use $\mathbf{w}$ such that $\widehat{\mathcal{P}}_{1}(\cdot)=\frac{1}{n_1} \sum_{i: T_i=1} \delta_{\mathbf{X}_i}(\cdot)$
and $\widehat{\mathcal{P}}_{0}^{\mathbf{w}}(\cdot) = \sum_{i: T_i=0} w_i \delta_{\mathbf{X}_i}(\cdot)$ are as similar as possible. 
We propose to measure the similarity between $\widehat{\mathcal{P}}_{1}(\cdot)$ and $\widehat{\mathcal{P}}_{0}^{\mathbf{w}}(\cdot)$ using the ReLU-IPM.
That is, we use
\begin{align} \label{eq:weighted_est}
    \hat{\mathbf{w}} = \underset{\mathbf{w} \in \mathcal{W}^+}{\mathrm{argmin}} \, d_{\mathcal{F}_{\textup{ReLU}}}(\mathcal{P}_{0,n}^{\mathbf{w}}, \mathcal{P}_{1,n}),
\end{align}
where the set of weight vectors $\mathcal{W}^+$ is defined as
$$
\mathcal{W}^+ \coloneqq \left\{ \mathbf{w} =(w_1,\ldots,w_n)^{\top} \in [0,1]^n: \max_{i \in \{1,\ldots,n\}} w_i \leq \frac{K}{n_0}, \sum_{i: T_i=0} w_i = 1, \sum_{i: T_i=1} w_i = 0 \right\}
$$ 
for a positive sufficiently large number $K$ such that $1/\eta^2 \leq K$. 
We refer to $\widehat{\mathrm{ATT}}^{\hat{\mathbf{w}}}$ as the {\it ReLU-CB estimator}.

Let $m_t(\cdot)=\mathbb{E}(Y(t)|\mathbf{X}=\cdot)$ for $t\in \{0,1\}.$
Theorem 4.2 of \cite{pmlr-v202-kong23d} can be used to prove the consistency of the ReLU-CB estimator under regularity conditions, when $m_0$ is smooth.
However we can do more due to nice theoretical properties of the ReLU-IPM.
That is, by combining Theorems \ref{thm:holder} and \ref{thm:conv_relu}, we can derive the convergence rate of the ReLU-CB estimator whose results are presented in Theorem \ref{thm:causal} below.
Particularly, it is proved that the ATT estimator has a parametric convergence rate when the true output regression model is sufficiently smooth. The proof of Theorem \ref{thm:causal} is deferred to Appendix \ref{pf:causal}.

\begin{theorem} \label{thm:causal}
    Suppose that $m_0(\cdot)$ belong to the H\"{o}lder class $\mathcal{H}_d^{\beta}.$ Then,
    \begin{enumerate}
        \item[(i)] if $\beta < \frac{d+3}{2}$, we have
        $$\widehat{\mathrm{ATT}}^{\hat{\mathbf{w}}} - \mathrm{ATT} = \mathcal{O}_p\left(n^{-\frac{\beta}{d+3}}\right),$$

        \item[(ii)] and if $\beta > \frac{d+3}{2}$, we have    $$\widehat{\mathrm{ATT}}^{\hat{\mathbf{w}}} - \mathrm{ATT} = \mathcal{O}_p\left(n^{-\frac{1}{2}}\right).$$
    \end{enumerate}
\end{theorem}


\subsection{Simulation study}
\label{sec:relu-cb}

\subsubsection{Settings}

\begin{paragraph}{Simulation model} 
    We generate simulated datasets using the \textsc{Kang-Schafer} models to analyze the finite sample performance of the ReLU-CB estimator, as is done in previous works \citep{kang2007demystifying, pmlr-v202-kong23d}. 
    The \textsc{Kang-Schafer} models consider the nonlinear outcome regression models and the nonlinear propensity score models for the binary treatment. For comparing the performances of baselines, we calculate the bias and RMSE of the corresponding ATT estimators. Details of the \textsc{Kang-Schafer} models are explained in Appendix \ref{sec:relu_cb_simulations}.
\end{paragraph}

\begin{paragraph}{Baselines}
    For baselines, we consider (i) the stabilized IPW (SIPW) with the linear logistic regression (GLM) \citep{lunceford2004stratification}, (ii) the SIPW with the boosting algorithm (Boost) used by \citet{lee2010improving}, (iii) the SIPW with the covariate balancing propensity score (CBPS) with the linear logistic regression \citep{imai2014covariate}, and (iv) the entropy balancing (EB) method of \citet{hainmueller2012entropy}.
    For CBPS and EB, we match the first moments of $\mathbf{X}$. For baselines using IPMs, we consider five IPMs, including (iv) Wasserstein distance (Wass), (v) sigmoid-IPM (SIPM), (vi) MMD with the RBF kernel (RBF), (vii) MMD with the Sobolev kernel (Sob) of \citet{Wong2017sobolev}, and (viii) H\"{o}lder-IPM using DNN discriminators with $L$-many layers for $L=1, 2$ (H\"{o}l-1, H\"{o}l-2). 
    Details of the baselines are provided in Appendix \ref{sec:impdetail_relu_cb}.
\end{paragraph}

\subsubsection{Results}
Table \ref{table1} shows the results based on 1000 simulated datasets. 
For most cases, ReLU-CB outperforms the other estimators with large margins in terms of both the bias and RMSE.
H\"{o}l-1 and H\"{o}l-2 tend to have larger RMSEs compared to the ReLU-CB, which would be partly because of overfitting. Moreover, the RMSEs of H\"{o}l-1 and H\"{o}l-2 do not decrease as the sample size increases.

GLM and Boost, which use the inverse of the estimated propensity score, are much inferior than those using the covariate balancing, which confirms again that accurate estimation of the propensity score does not guarantee the accurate estimation of its inverse.

Smaller biases of ReLU-CB is interesting since the ReLU-IPM is parametric while the true outcome regression model is highly nonlinear. This would be because the ReLU-IPM behaves similarly to the H\"{o}lder-IPM as proved in Theorem \ref{thm:holder}.

\begin{table*}[t] 
\renewcommand{\arraystretch}{1.2}
\caption{\textbf{\textsc{Kang-Schafer} models.} 
We generate 1000 simulated datasets from the two \textsc{Kang-Schafer} models with two sample sizes, and report the biases and RMSEs.
For each pair of dataset and performance measure (i.e. for each row), the best result is highlighted in bold, while the second best result is underlined.
} 
\label{table1}
\centering
\scalebox{0.75}{ 
\begin{tabular}{c|c|c|cccccccccc|c}
\hline
Model & Meas. & $n$ & GLM & Boost & CBPS & EB & Wass & SIPM & RBF & Sob & H\"{o}l-1 & H\"{o}l-2 & ReLU-CB \\
\hline \hline 
\multirow{4}{*}{1} & \multirow{2}{*}{Bias}  & 200 & -7.820  & -8.375 & -4.745 & -4.806 & -3.934 & -3.567 & \underline{-3.162} & -3.255 & -3.358 & -4.935 & \textbf{-2.448} \\
& & 1000 & -7.410 & -6.489 & -4.496 & -4.494  & -3.607 & -2.757 & -2.844 & -2.700  & \underline{-2.347} & -4.787 & \textbf{-1.819}\\
\cline{2-14}
& \multirow{2}{*}{RMSE} & 200 & 8.719 & 9.204 & 5.354 & 5.395 & 4.619 & 4.611 & 4.356 & \underline{4.285} & 7.995 & 9.255 & \textbf{3.952} \\
& & 1000 & 7.622 & 6.665 & 4.620 & 4.618 & 3.760 & 3.009 & 3.020 & \underline{2.861} & 3.905 & 9.607 & \textbf{2.080}\\
\hline \hline
\multirow{4}{*}{2} & \multirow{2}{*}{Bias}  & 200 & -7.753  & -8.306 & -4.671 & -4.732 & -3.882 & -3.505 & -3.147 & -3.205 & -3.492 & \underline{-2.633} & \textbf{-2.445} \\
& & 1000 & -7.452 & -6.525 & -4.540 & -4.538  & -3.677 & -2.824 & -2.878 & -2.722  & \underline{-2.400} & -2.690 & \textbf{-1.872}\\
\cline{2-14}
& \multirow{2}{*}{RMSE} & 200 & 8.813 & 9.289 & 5.575 & 5.615 & 4.991 & 5.177 & 4.942 & 4.841 & 8.173 & \textbf{4.576} & \underline{4.733} \\
& & 1000 & 7.707 & 6.753 & 4.730 & 4.728 & 3.926 & 3.221 & 3.207 & \underline{3.034} & 4.087 & 4.677 & \textbf{2.376}\\
\hline
\end{tabular}
}
\end{table*}

\subsubsection{More data analysis}
We analyze the ACIC dataset \citep{dorie2019automated} to evaluate the performance of the ReLU-CB estimator, and the results are provided in Appendix \ref{sec:add_experiments_relu_cb}.


\section{Application 2: Fair Representation Learning with ReLU-IPM}
\label{sec:app2}

The notion of group fairness is a fundamental principle in algorithmic fairness, which aims to eliminate biases between sensitive groups (e.g. man vs woman, white vs black, etc.) divided by a given sensitive attribute (e.g. gender, race, etc.).
The most popular criterion for group fairness is Demographic Parity (DP) which focuses on statistical disparities in decisions/outcomes across sensitive groups \citep{calders2009building, feldman2015certifying, ADW19, NEURIPS2020_51cdbd26, pmlr-v115-jiang20a}.
Subsequent fairness metrics, including Equal Opportunity (EqOpp) and Equalized Odds (EO), suggest that fair decisions need to be conditioned on both the label and the sensitive attribute \citep{NIPS2016_9d268236}.
Various algorithms have been proposed to achieve specified levels of group fairness with respect to one of these criteria \citep{zafar2017fairness, donini2018empirical, pmlr-v80-agarwal18a, Madras2018LearningAF, zafar2019fairness, chuang2021fair, pmlr-v162-kim22b}.

Fair Representation Learning (FRL) aims at searching a fair representation space \citep{pmlr-v28-zemel13} in the sense that the distributions of the encoded representation vector for each sensitive group are similar.
Since any prediction model constructed on a fair representation space is automatically group-fair and thus
FRL is considered as a tool to transform given unfair data to be fair.
Numerous algorithms for FRL have been proposed, including \citep{https://doi.org/10.48550/arxiv.1511.00830, 9aa5ba8a091248d597ff7cf0173da151, xie2017controllable, Madras2018LearningAF, Quadrianto_2019_CVPR, gupta2021controllable, zeng2021fair, pmlr-v162-kim22b, guo2022learning, pmlr-v206-deka23a, kong2025fair}, to name just a few.

A popular learning strategy for FRL is the adversarial learning \citep{9aa5ba8a091248d597ff7cf0173da151,Madras2018LearningAF}, where an adversary tries to predict a sensitive attribute on the representation space while the representation space is trained to prevent it. 
\citet{pmlr-v162-kim22b} apply the SIPM to the adversarial learning where the discriminator plays a role of an adversary and prove that any prediction model constructed on the learned fair representation space is also fair as long as the prediction head is infinitely differentiable.
In addition, it is empirically demonstrated that FRL with the SIPM attains better fairness-accuracy trade-offs compared to existing baselines including LAFTR \citep{Madras2018LearningAF}. 

In this section, we consider an FRL algorithm based on the ReLU-IPM. We prove that any prediction model on the learned fair representation space by the ReLU-IPM is guaranteed to be group-fair as long as the prediction model is H\"{o}lder continuous.

\subsection{Preliminaries for Fair Representation Learning}
We consider a binary sensitive attribute.
Let $\mathbf{X} \in \mathcal{X} \subset \mathbb{R}^{d}, Y \in \mathcal{Y},$ and $S \in \{0, 1\}$ represent the $d$-dimensional input vector, the output variable, and the binary sensitive attribute.
We write $\mathcal{P}$ and $\mathcal{P}_{s}$ as the joint distribution of $(\mathbf{X}, Y, S)$ and the conditional distribution of $\mathbf{X} | S = s$, respectively.
Similarly, let $\mathbb{E}$ and $\mathbb{E}_{s}$ be the corresponding expectation operators.

A mapping $h(\mathbf{X}, S)$ from $\mathcal{X} \times \{0, 1\}$ to a certain representation space $\mathcal{Z}$ is called a fair representation if $h(\mathbf{X},S)|S=0 \stackrel{d}\approx h(\mathbf{X},S)|S=1.$
For any prediction head $g:\mathcal{Z}\rightarrow \mathcal{Y},$ we expect that $g\circ h(\mathbf{X},S)|S=0 \stackrel{d}\approx g\circ h(\mathbf{X},S)|S=1$ as long as $h$ is fair.
Thus, a fair representation can be considered as a transformation of data such that the transformed data are considered to be group-fair. 

For a given representation $h,$ let $\mathbb{P}_s^h$ be the conditional distribution of $h(\mathbf{X},S)|S=s.$ 
Most of fair representation learning algorithms consist of two steps. The first step is to choose a deviance measure for the dissimilarity of
$\mathcal{P}_0^h$ and $\mathcal{P}_1^h,$ and the second step is to find a representation that minimizes the deviance measure.
For the deviance measure, \citet{Madras2018LearningAF} uses the KL divergence, \citet{pmlr-v206-deka23a} uses the Maximum Mean Discrepancy (MMD), and \citet{pmlr-v162-kim22b} uses the SIPM.


\subsection{Learning a fair representation by use of the ReLU-IPM}

Let $\mathcal{H},\mathcal{G}$ and $\ell$ be a given set of representations, a given set of prediction heads on $\mathcal{Z}$ and a given loss function on the output space, respectively. We consider the algorithm that learns a fair representation $h$ by minimizing $\mathbb{E}\ell(g\circ h(\mathbf{X},S), Y)$ subject to $d_{\mathcal{F}}(\mathcal{P}_0^h,\mathcal{P}_1^h) \le \delta$ with respect to $h\in \mathcal{H}$ and $g\in \mathcal{G}$ for a prespecified discriminator class $\mathcal{F}$ of functions from $\mathcal{Z}$ to $\mathbb{R}$ and a positive constant $\delta.$ 
For given data, we replace the probability measures and expectations by their empirical counterparts.

A fair representation can be considered as fairly transformed data since it would be expected that $g\circ h$ is fair for a reasonable prediction head $g.$
We say that $f$ is groupwisely $\eta$-fair with respect to a fair function $\phi: \mathbb{R} \rightarrow \mathbb{R}$ \citep{Madras2018LearningAF, chuang2021fair, pmlr-v162-kim22b} if $\Delta \textup{DP}_{\phi}(f) \leq \eta$ where $$\Delta \textup{DP}_{\phi}(f)=\big|\mathbb{E}( \phi \circ f(\mathbf{X},S)|S=0 )-\mathbb{E}( \phi \circ f(\mathbf{X},S)|S=1 )\big|.$$

The existing well-known measures for group fairness can be expressed by choosing a corresponding fair function $\phi.$
For example, the conventional measure $\Delta \textup{DP}$ \citep{calders2009building, feldman2015certifying, donini2018empirical, ADW19, zafar2019fairness} is $\Delta \textup{DP}_{\phi}$ with $\phi(t) = \mathbb{I}(t \ge 1/2).$
Another popularly used measure $\Delta \overline{\textup{DP}}$, which is the disparity between prediction scores \citep{Madras2018LearningAF, chuang2021fair, pmlr-v162-kim22b}, is defined with $\phi(t) = t.$
$\Delta \textup{SDP},$ which is a distributional gap between prediction scores, is defined as $\Delta \textup{SDP} = \mathbb{E}_{\tau \in \textup{Unif}(0, 1)} \Delta \textup{DP}_{\phi_{\tau}},$
where $\phi_{\tau}(t) = \mathbb{I}(t \ge \tau)$ \citep{NEURIPS2020_51cdbd26, pmlr-v115-jiang20a, Silvia_Ray_Tom_Aldo_Heinrich_John_2020, barata2021fair}.
Moreover, as a surrogate of $\Delta \textup{DP},$ $\Delta \textup{DP}_\phi$
with $\phi(t) = (1 - t)_{+}$ and $\phi(t) = \textup{sig}(t)$ have been used \citep{wu2019convexity, kim2022slide, kim2022learning, yao2023understanding}.

A given fair representation $h$ is said to be (groupwisely) $\eta$-fair on a given set $\mathcal{G}$ of prediction heads if $g\circ h$ is $\eta$-fair for all $g\in \mathcal{G}.$
If the IPM with a discriminator class $\mathcal{F}$ is used for the deviance measure of FRL, it is easy to see that the learned fair representation $\hat{h}$ is $\eta$-fair as long as $\mathcal{G} \subset \mathcal{F}.$
However, for the ReLU-IPM, the learned fair representation can be $\eta$-fair on a much larger class of prediction heads than $\mathcal{F}_{\textup{ReLU}}$ thanks to Theorem \ref{thm:holder}, whose results are rigorously stated in the following theorem.

\begin{theorem}[Level of group fairness] \label{thm:relu_frl}
        For any Lipschitz function $\phi$, there exists a constant $c > 0$ such that 
        \begin{itemize}
            \item[(i)]if $\beta > \frac{d+3}{2}$, for every $g \in \mathcal{H}_d^\beta$, we have
            $$
            \Delta \textup{DP}_{\phi}(g \circ h) \le c \cdot d_{\mathcal{F}_{\textup{ReLU}}} ( \mathcal{P}_{0}^{h}, \mathcal{P}_{1}^{h} ),
            $$
            \item[(ii)] if $\beta < \frac{d+3}{2}$, for every $g \in \mathcal{H}_d^\beta$, we have
            $$
            \Delta \textup{DP}_{\phi}(g \circ h) \le c \cdot d_{\mathcal{F}_{\textup{ReLU}}} ( \mathcal{P}_{0}^{h}, \mathcal{P}_{1}^{h} )^{\frac{2\beta}{d+3}}.
            $$
        \end{itemize}
\end{theorem}


\begin{figure}[p]
    \centering
    \includegraphics[width=0.81\textwidth]{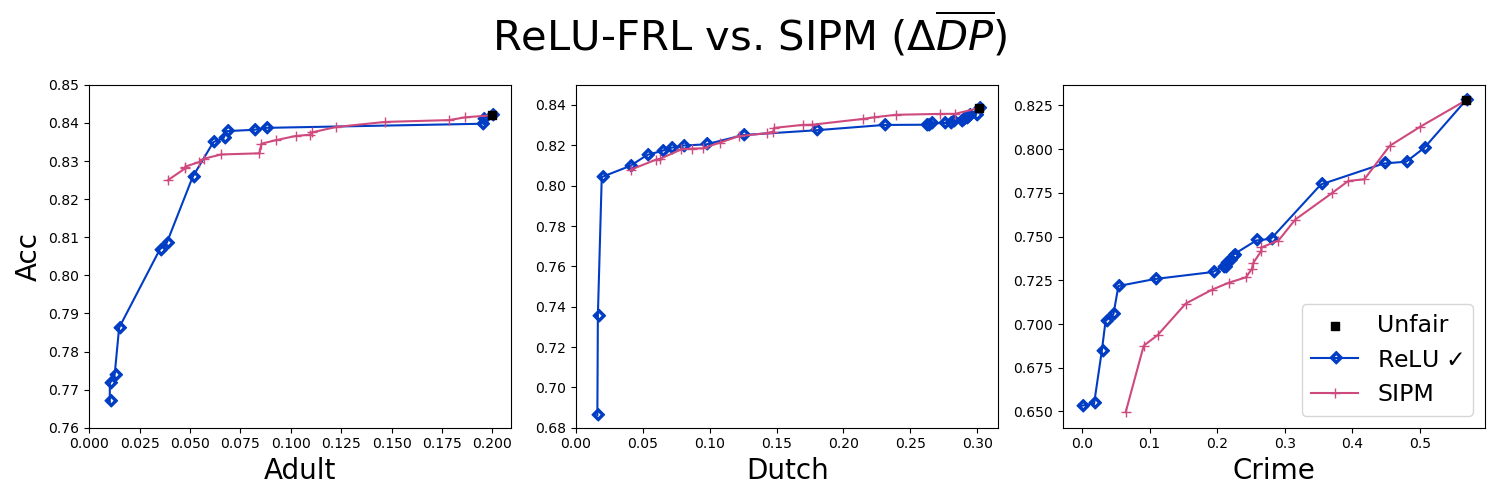}
    \includegraphics[width=0.81\textwidth]{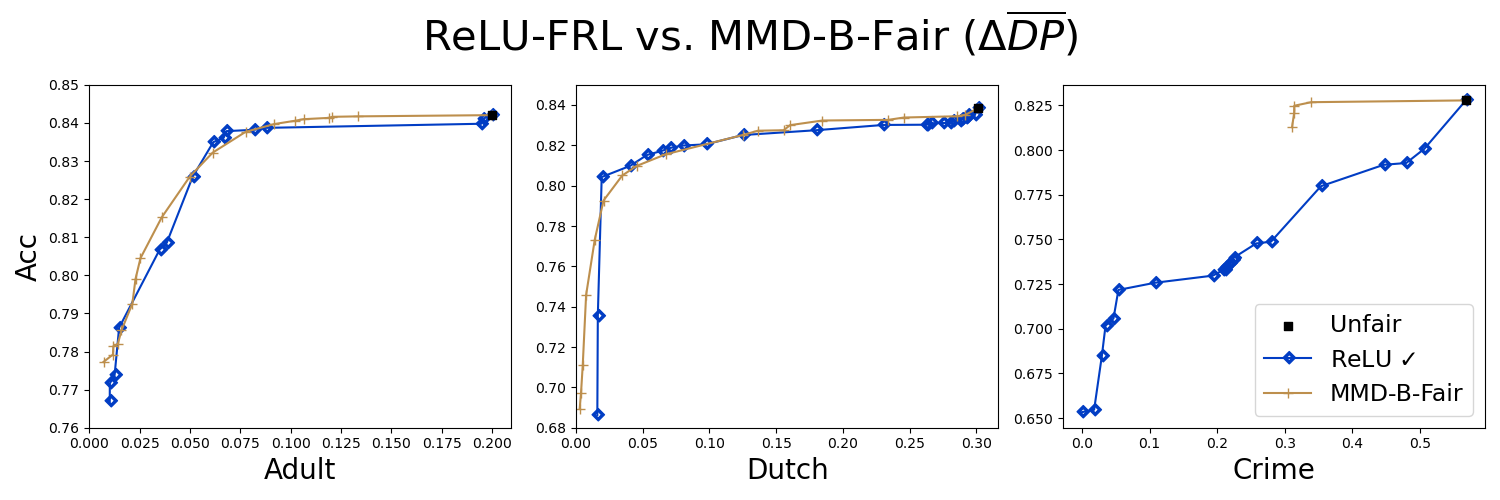}
    \includegraphics[width=0.81\textwidth]{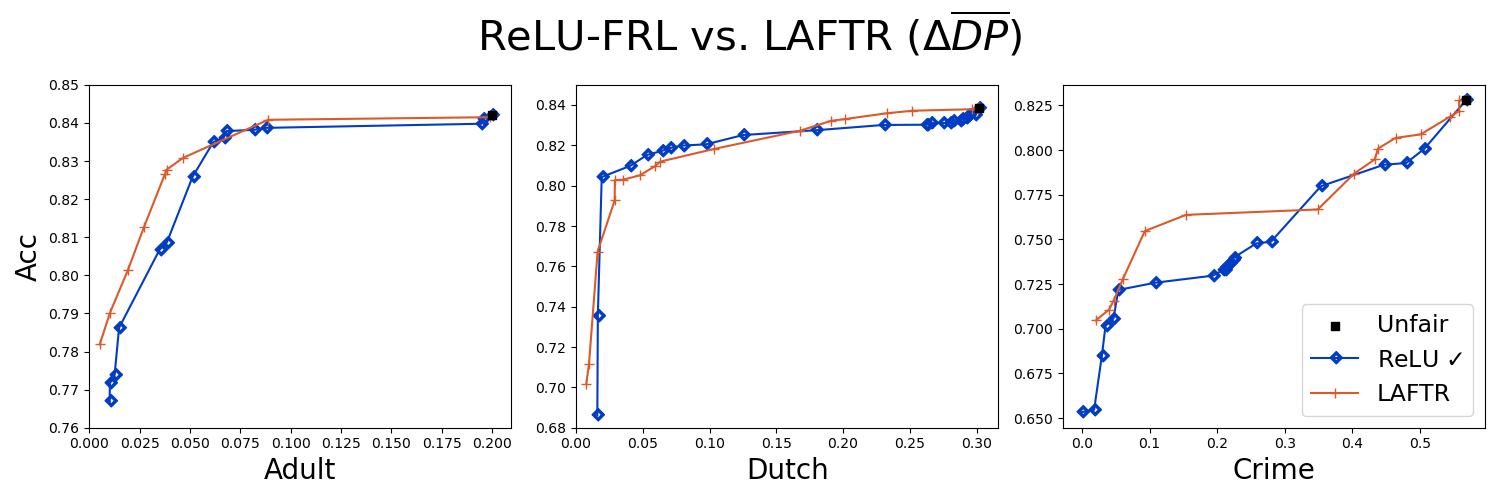}
    \includegraphics[width=0.81\textwidth]{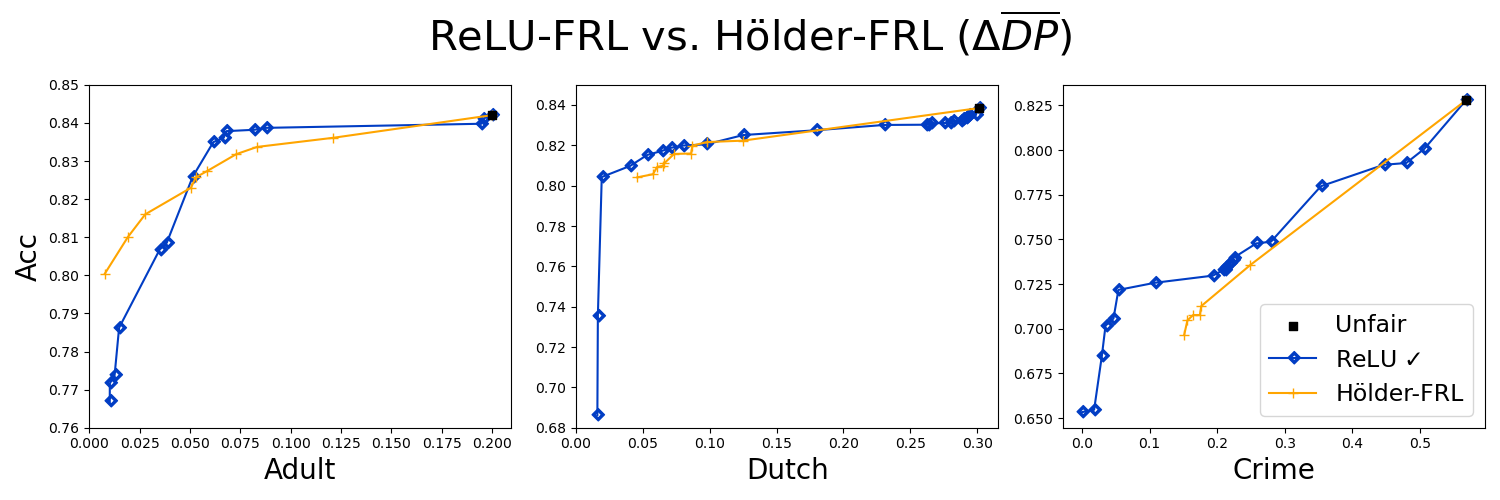}
    \caption{
    \textbf{Single-layered NN (ReLU activation) prediction head}: Pareto-front lines of fairness level $\Delta \overline{\textup{DP}}$ and \texttt{Acc}.
    (Left) \textsc{Adult}, (Center) \textsc{Dutch}, (Right) \textsc{Crime}.
    }
    \label{fig:1_ReLU_MLP_meandp}
\end{figure}

\subsection{Numerical Experiments}
\label{sec:relu-frl}

\subsubsection{Settings}

\begin{paragraph}{Datasets}
    We analyze three real benchmark tabular datasets: \textsc{Adult}, \textsc{Dutch}, and \textsc{Crime}, which are known to have unfair bias with respect to specific sensitive attributes.
    Table \ref{table:datasets} in Appendix \ref{sec:datasets-appendix} provides a summary of key characteristics for these three datasets.
    We randomly split each dataset five times into training and test sets with a ratio of 8:2, and report the average performances on the five test datasets.
\end{paragraph}

\begin{paragraph}{Model architectures}
    For the representation encoder $h,$ we use a two-layered MLP with the ReLU activation.
    For the prediction head $g,$ we consider three models: (i) linear logistic regression model, (ii) a single-layer neural network with the ReLU activation function and 100 nodes, and (ii) a single-layer neural network with the sigmoid activation function and 100 nodes. Considering multiple models for the head $g$ aims to investigate the fairness level of learned fair representations.
\end{paragraph}

\begin{paragraph}{Baseline algorithms}
    For baselines, we consider the following four FRL algorithms:
    (i) SIPM-LFR \citep{pmlr-v162-kim22b} and
    (ii) MMD-B-Fair \citep{pmlr-v206-deka23a} which use the SIPM and MMD to measure the deviance of the distributions of the representation corresponding to the two sensitive groups, respectively,
    (iii) LAFTR \citep{Madras2018LearningAF} which employs an adversarial learning on the representation space,
   and (iv) H\"{o}lder-FRL  which uses the H\"{o}lder-IPM whose discriminators are approximated by DNNs with one  hidden layer. In Appendix \ref{sec:impdetail_relu_frl}, we provide implementation details. 
\end{paragraph}


\subsubsection{Results}

For the comparison of empirical performances, we consider the trade-off between accuracy (\textsc{Acc}) and fairness levels ($\Delta \overline{\textup{DP}}, \Delta \textup{DP}, \Delta \textup{SDP}$), following the approach used in previous works \citep{Madras2018LearningAF, pmlr-v162-kim22b, pmlr-v206-deka23a}. Figure \ref{fig:1_ReLU_MLP_meandp} compares ReLU-FRL with the other baselines through the Pareto-front trade-off graphs between fairness level $\Delta \overline{\textup{DP}}$  and  \textsc{Acc} across three datasets: \textsc{Adult}, \textsc{Dutch}, and \textsc{Crime}.
The Pareto-front lines for the other fairness measures including $\Delta \textup{DP}, \Delta \textup{SDP}$ vs \textsc{Acc} along with different network architectures of the prediction head $g$ are presented in Appendix \ref{sec:add_experiments_relu_frl}.
The results clearly show that ReLU-FRL compares favorably to the other baselines. In particular, ReLU-FRL outperforms the other baselines for the \textsc{Crime} data, even though the margins are not large.
In addition, it is observed that MMD-B-fair has a difficulty in achieving the fairness level below a certain level.


\section{Conclusion}
\label{sec:conc}

In this study, we proposed a new parametric IPM which has desirable theoretical guarantees. In addition, we conducted numerical studies for illustrating the usefulness of the ReLU-IPM across various domains including the covariate balancing in causal estimation and fair representation learning.
Notably, compared to the H\"{o}lder-IPM, the ReLU-IPM has several advantages: (i) it has similar theoretical guarantees, in particular, when the discriminator class is sufficiently smooth, (ii) it is simpler to implement, (iii) it requires fewer hyperparameters, and (iv) it demonstrates comparable or superior performance in numerical experiments. And, in comparison to the SIPM, we have established better theoretical properties of the ReLU-IPM including
(i) parametric convergence rate of the ATT estimator and (ii)  a larger class of prediction heads on the learned fair representation is guaranteed to be group-fair.

There are several possible future works.
(1) For the ATT estimation, we only succeeded in deriving the parametric convergence rate. 
We believe that the proposed estimator is semi-parametric efficient which needs to be proved.
(2) Replacing the ReLU activation in the ReLU-IPM with other activation functions would be useful.
The leaky ReLU activation function would be a promising alternative since there is one more parameter in the activation function and hence we can select the activation function data-adaptively by choosing the activation parameter data-adaptively. (3) The ReLU-IPM could provide a way of explaining how two given distributions differ. By looking at the distributions of $\hat\theta^\top \mathbf{X}$ and $\hat\theta^\top \mathbf{Y},$ where $\hat\theta$ is the parameter at which the ReLU-IPM is calculated,
we could grab some idea of the difference between $\mathcal{P}$ and $\mathcal{Q}.$

\acks{
This work was supported by the Global-LAMP Program of the National Research Foundation of Korea (NRF) grant funded by the Ministry of Education (No.RS-2023-00301976) and the National Research Foundation of Korea(NRF) grant funded by the Korea government(MSIT) (RS-2025-00556079).
}

\clearpage


\appendix
\section{Definitions and Auxiliary lemmas} \label{sec:math_def}

In this section, we provide mathematical tools used in this paper.

\begin{definition}[Bounded differences property]
    A function $f: \mathcal{X}_1 \times \mathcal{X}_2 \times \cdots \times \mathcal{X}_n \rightarrow \mathbb{R}$ satisfies the bounded differences property if substituting the value of the $i$-th coordinate $\mathbf{x}_i$ changes the value of $f$ by at most $c_i$. More formally, if there are constants $c_1, c_2, \ldots, c_n$ such that for all $i \in \{1,\ldots,n\}$, and all $\mathbf{x}_1 \in \mathcal{X}_1, \mathbf{x}_2 \in \mathcal{X}_2, \ldots, \mathbf{x}_{n} \in \mathcal{X}_n$,
    \begin{align*}
    \sup_{\mathbf{x}_i^{\prime} \in \mathcal{X}_i^{\prime}} \Big|f \big( \mathbf{x}_1, & \ldots, \mathbf{x}_{i-1}, \mathbf{x}_i, \mathbf{x}_{i+1}, \ldots, \mathbf{x}_{n} \big) - f\left(\mathbf{x}_1, \ldots, \mathbf{x}_{i-1}, \mathbf{x}_i^{\prime}, \mathbf{x}_{i+1}, \ldots, \mathbf{x}_{n}\right) \Big| \leq c_i .
    \end{align*}
\end{definition}

\begin{definition}[Rademacher random variable] A random variable $X \in \{-1,1\}$ is called a Rademacher random variable if  $\mathbb{P}(X=1) = \mathbb{P}(X=-1)=\frac{1}{2}$.
\end{definition}

\begin{definition}[$L_p(\mathcal{Q})$-norm]
    Let $\mathcal{Q}$ be a measure on a measurable space $(\mathcal{X}, \mathcal{A})$ and $L_p(\mathcal{Q})=\{f: \mathcal{X} \rightarrow \mathbb{R}: \int|f|^p d\mathcal{Q} < \infty \}$. For $f \in L_p(\mathcal{Q})$,  the $L_p(\mathcal{Q})$-norm is defined as
    $\|f\|_{p, \mathcal{Q}} \coloneqq \left(\int |f|^p d\mathcal{Q}\right)^{1/p}$. When $p=2$, we omit the subscript 2. i.e., $\|\cdot\|_{\mathcal{Q}} = \|\cdot\|_{2, \mathcal{Q}}$.
\end{definition}

\begin{definition}[$L_p(\mathcal{Q}_n)$-norm]
    Let $\mathcal{Q}_n$ be the empirical measure of $z_1, \ldots, z_n$. We define the $L_p(\mathcal{Q}_n)$-norm by $\|f\|_{p,\mathcal{Q}_n} \coloneqq \left( \frac{1}{n} \sum_{i=1}^n |f(z_i)|^p \right)^{1/p}$. For $p=2$, we omit the subscript 2, writing $\|\cdot\|_{\mathcal{Q}_n}$ instead of $\|\cdot\|_{2, \mathcal{Q}_n}$.
\end{definition}

\begin{definition}[$\epsilon$-covering number and $\epsilon$-entropy] 
Let $N_p(\mathscr{F}, \|\cdot\|_{\mathcal{Q}}, \epsilon)$ be the $\epsilon$-covering number, defined as the smallest $N$ such that for every $\epsilon > 0$ and a class of functions $\mathscr{F} \subset L_2(\mathcal{Q})$, there exist functions $f_1, \ldots, f_N \in L_2(\mathcal{Q})$ such that $\mathscr{F}$ can be covered by $N$ balls of radius $\epsilon$ centered at  $f_1, \ldots, f_N$ in $L_2(\mathcal{Q})$.
For $\epsilon > 0$ and a class of functions $\mathscr{F} \subset L_2(\mathcal{Q})$, let
$H(\mathscr{F}, \|\cdot\|_{\mathcal{Q}}, \epsilon) = \log N_p(\mathscr{F}, \|\cdot\|_{\mathcal{Q}}, \epsilon)$ be the $\epsilon$-entropy of $\mathscr{F}$ with respect to the $L_2(\mathcal{Q})$-metric.
\end{definition}

\begin{lemma}[McDiarmid's inequality \citep{McDiarmid_1989}] 
\label{lem:McD}    
    Let $f: \mathcal{X}_1 \times \mathcal{X}_2 \times \cdots \times \mathcal{X}_n \rightarrow \mathbb{R}$ satisfy the bounded differences property with bounds $c_1, c_2, \ldots, c_n$. Consider independent random vectors $\mathbf{X}_1, \ldots, \mathbf{X}_{n}$ where $\mathbf{X}_i \in \mathcal{X}_i$ for all $i$. Then, for any $\epsilon>0$,
    \begin{align*}
    & \mathbb{P}\left\{ f\left(\mathbf{X}_1, \ldots, \mathbf{X}_n\right)-\mathbb{E}\left[f\left(\mathbf{X}_1, \ldots, \mathbf{X}_n\right)\right] \geq \epsilon \right\} \leq \exp \left(-\frac{2 \epsilon^2}{\sum_{i=1}^n c_i^2}\right), \\
    & \mathbb{P}\left\{ f\left(\mathbf{X}_1, \ldots, \mathbf{X}_n\right)-\mathbb{E}\left[f\left(\mathbf{X}_1, \ldots, \mathbf{X}_n\right)\right] \leq-\epsilon \right\} \leq \exp \left(-\frac{2 \epsilon^2}{\sum_{i=1}^n c_i^2}\right),
    \end{align*}
    and as an immediate consequence,    
    \begin{align*}
    \mathbb{P}\left\{ \left|f\left(\mathbf{X}_1, \ldots, \mathbf{X}_{n}\right)-\mathbb{E}\left[f\left(\mathbf{X}_1, \ldots, \mathbf{X}_{n}\right)\right]\right| \geq \epsilon \right\} \leq 2 \exp \left(-\frac{2 \epsilon^2}{\sum_{i=1}^n c_i^2}\right) . \qquad \qquad 
    \end{align*}
\end{lemma}

\clearpage

\begin{lemma}[Ledoux-Talagrand contraction inequality \citep{LedouxTalagrand1991}] \label{Ledoux}
    For a compact set $\mathcal{T}$, let $\mathbf{X}_1, \cdots, \mathbf{X}_m$ be i.i.d random vectors whose real-valued components are indexed by $\mathcal{T}$, i.e., $\mathbf{X}_i = \left(\mathbf{X}_{i, s}\right)_{s \in \mathcal{T}}$. Let $\phi: \mathbb{R} \rightarrow \mathbb{R}$ be a 1-Lipschitz function such that $\phi(0)=0$. Let $\epsilon_1, \cdots, \epsilon_m$ be independent Rademacher random variables. Then
    $$
    \mathbb{E}\left[\sup _{s \in \mathcal{T}}\left|\sum_{i=1}^m \epsilon_i \phi\left(\mathbf{X}_{i, s}\right)\right|\right] \leq 2 \mathbb{E}\left[\sup _{s \in \mathcal{T}}\left|\sum_{i=1}^m \epsilon_i \mathbf{X}_{i, s}\right|\right].
    $$
\end{lemma}

\begin{lemma}[Lemma 8.4 of \cite{geer2000empirical}] \label{lem:vandeGeer}
    Let $W_1, \ldots, W_n$ be independent random variables with mean zero and let $z_1, \ldots, z_n$ be a set of points.
    We denote the empirical measure of $z_1,...,z_n$ by $Q_n$.
    Suppose that for some $\alpha \in (0,2)$, $A>0$, R, K, and $\sigma_0$, there exists some function class $\mathscr{F}$ such that (i) $H(\mathscr{F}, \|\cdot\|_{\mathcal{Q}_n}, \epsilon) \leq A \epsilon^{-\alpha}$ holds for all $\epsilon > 0$, (ii) $\sup_{f \in \mathscr{F}} \| f \|_{Q_n} \leq R$, (iii) $\max_{i=1,\ldots,n} K^2 \left( \mathbb{E} e^{|W_i|^2/K^2} -1 \right) \leq \sigma_0^2.$
    Then there exists a constant $c$ depending on $A, \alpha, R, K$, and $\sigma_0$ such that we have 
    $$
    \mathbb{P} \left( \sup_{f \in \mathscr{F}} \frac{\left| \frac{1}{\sqrt{n}} \sum_{i=1}^n W_i f(z_i) \right|}{\|f\|_{\mathcal{Q}_n}^{1-\alpha/2}} \geq T \right)
    \leq c \exp{\left( -\frac{T^2}{c^2} \right)}
    $$
    for all $T \geq c.$
\end{lemma}

\clearpage

\vspace*{-10pt}
\section{Proofs of Theoretical Results} \label{sec:appendix2}


\subsection{Proof of Proposition \ref{prop:discriminative}} \label{pf:prop1}

We begin with the following lemma, which plays a key role in the proof of Proposition \ref{prop:discriminative}.

\begin{lemma} \label{lem:discriminative}
    For two random vectors $\mathbf{X} \sim \mathcal{P}$ and $\mathbf{Y} \sim \mathcal{Q},$ 
    if $d_{\mathcal{F}_{\textup{ReLU}}}(\mathcal{P}, \mathcal{Q}) = 0$, then 
    \begin{align}
        \sup_{\mathbf{c} \in \mathbb{S}^{d-1}} \sup_{t \in \mathbb{R}} \left| \mathbb{P} \left( \mathbf{c}^{\top} \mathbf{X} \leq t \right) - \mathbb{P} \left( \mathbf{c}^{\top} \mathbf{Y} \leq t \right) \right| = 0.
    \end{align}
\end{lemma}

\begin{proof}{\textbf{of Lemma \ref{lem:discriminative}}}
    By the definition of the ReLU-IPM, $d_{\mathcal{F}_{\textup{ReLU}}}(\mathcal{P}, \mathcal{Q})=0$ implies 
    \begin{equation} \label{equaldistance}
        \left| \mathbb{E}(\mathbf{c}^{\top}\mathbf{X})_{+} - \mathbb{E}(\mathbf{c}^{\top}\mathbf{Y})_{+} \right| = 0  
    \end{equation}
    for any $\mathbf{c} \in \mathbb{S}^{d-1}.$
    For given $\mathbf{c} \in \mathbb{S}^{d-1}$, we consider $t \in \mathbb{R}$ such that random variables $\mathbf{c}^{\top} \mathbf{X}$ and $\mathbf{c}^{\top} \mathbf{Y}$ do not have a point mass at $t$.    
    For $\delta>0$, let $\Delta_{\delta}(\mathbf{X}) \coloneqq (\mathbf{c}^{\top} \mathbf{X}-t+\delta)_{+} - (\mathbf{c}^{\top} \mathbf{X}-t)_{+}$. Then, we can write
    \begin{align}
        \Delta_{\delta}(\mathbf{X}) =  (\mathbf{c}^{\top} \mathbf{X}-t+\delta) \cdot \mathbb{I}(t-\delta < \mathbf{c}^{\top} \mathbf{X} \leq t) + \delta \cdot \mathbb{I}(t < \mathbf{c}^{\top} \mathbf{X}). \label{Delta}
    \end{align}
    From (\ref{Delta}), we obtain    
    \begin{align}
        & \left|\mathbb{P} (\mathbf{c}^{\top} \mathbf{X} > t) - \frac{1}{\delta} \mathbb{E} \Delta_{\delta}(\mathbf{X})\right| \nonumber \\
        & = \left|\mathbb{P} (\mathbf{c}^{\top} \mathbf{X} > t) - \mathbb{E} \left( \frac{\mathbf{c}^{\top} \mathbf{X}-t+\delta}{\delta} \cdot \mathbb{I}(t-\delta < \mathbf{c}^{\top} \mathbf{X} \leq t) \right) - \mathbb{P} (\mathbf{c}^{\top} \mathbf{X} > t) \right|, \nonumber \\
        & \leq \mathbb{P}(t-\delta < \mathbf{c}^{\top} \mathbf{X} \leq t), \label{pro_exp1}
    \end{align}
    where we use $0 \leq \mathbf{c}^{\top} \mathbf{X}-t+\delta \leq \delta$ for the last inequality. 
    Similarly, for $\Delta_{\delta}(\mathbf{Y}) \coloneqq (\mathbf{c}^{\top} \mathbf{Y}-t+\delta)_{+} - (\mathbf{c}^{\top} \mathbf{Y}-t)_{+}$,
    \begin{align}
        \left|\mathbb{P} (\mathbf{c}^{\top} \mathbf{Y} > t) - \frac{1}{\delta} \mathbb{E} \Delta_{\delta}(\mathbf{Y})\right| 
        \leq \mathbb{P}(t-\delta < \mathbf{c}^{\top} \mathbf{Y} \leq t). \label{pro_exp2}
    \end{align}
    Therefore, from (\ref{equaldistance}), (\ref{pro_exp1}) and (\ref{pro_exp2}), we get
    \begin{align*}
        \Big| \mathbb{P}(\mathbf{c}^{\top} \mathbf{X} \leq t) - \mathbb{P}(\mathbf{c}^{\top} \mathbf{Y} \leq t) \Big| 
        & = \left| \mathbb{P}(\mathbf{c}^{\top} \mathbf{X} > t) - \mathbb{P}(\mathbf{c}^{\top} \mathbf{Y} > t) \right| \\
        & \leq \left|\frac{1}{\delta} \mathbb{E} \Delta_{\delta}(\mathbf{X})
        - \frac{1}{\delta} \mathbb{E} \Delta_{\delta}(\mathbf{Y}) \right| 
        + \mathbb{P}(t-\delta < \mathbf{c}^{\top} \mathbf{X} \leq t) 
        + \mathbb{P}(t-\delta < \mathbf{c}^{\top} \mathbf{Y} \leq t)\\        
        & \leq \frac{1}{\delta} \left| \mathbb{E}(\mathbf{c}^{\top} \mathbf{X}-t+\delta)_{+} - \mathbb{E}(\mathbf{c}^{\top} \mathbf{X}-t)_{+} - \mathbb{E}(\mathbf{c}^{\top} \mathbf{Y}-t+\delta)_{+} + \mathbb{E}(\mathbf{c}^{\top} \mathbf{Y}-t)_{+} \right| \\
        & \quad + \mathbb{P}(t-\delta < \mathbf{c}^{\top} \mathbf{X} \leq t) 
        + \mathbb{P}(t-\delta < \mathbf{c}^{\top} \mathbf{Y} \leq t)\\
        & = \mathbb{P}(t-\delta < \mathbf{c}^{\top} \mathbf{X} \leq t) 
        + \mathbb{P}(t-\delta < \mathbf{c}^{\top} \mathbf{Y} \leq t).
    \end{align*}
    Since $\delta>0$ is arbitrary, we obtain
    \begin{align*}
        \Big| \mathbb{P}(\mathbf{c}^{\top} \mathbf{X} \leq t) - \mathbb{P}(\mathbf{c}^{\top} \mathbf{Y} \leq t) \Big| 
        & \leq \lim_{\delta \downarrow 0} \mathbb{P}(t-\delta < \mathbf{c}^{\top} \mathbf{X} \leq t) 
        + \lim_{\delta \downarrow 0} \mathbb{P}(t-\delta < \mathbf{c}^{\top} \mathbf{Y} \leq t) = 0.
    \end{align*}

    For the case where either $\mathbf{c}^{\top} \mathbf{X}$ or $\mathbf{c}^{\top} \mathbf{Y}$ has a point mass at $t$, we can construct a sequence $\{t_j\}_{j = 1}^{\infty}$ such that $t_j \downarrow t$ and neither $\mathbf{c}^{\top} \mathbf{X}$ nor $\mathbf{c}^{\top} \mathbf{Y}$ has a point mass at $\{t_j\}_{j = 1}^{\infty}$.
    Since $\mathbb{P}(\mathbf{c}^{\top} \mathbf{X} \leq \cdot)$ and  $\mathbb{P}(\mathbf{c}^{\top} \mathbf{Y} \leq \cdot)$ is right-continuous, 
    \begin{align*}
         \left| \mathbb{P}(\mathbf{c}^{\top} \mathbf{X} \leq t) - \mathbb{P}(\mathbf{c}^{\top} \mathbf{Y} \leq t) \right| = \lim_{j \rightarrow \infty} & \left| \mathbb{P}(\mathbf{c}^{\top} \mathbf{X} \leq t_j) - \mathbb{P}(\mathbf{c}^{\top} \mathbf{Y} \leq t_j) \right| = 0,
    \end{align*}
    and the proof is done.
\end{proof}

\begin{proof}{\textbf{of Proposition \ref{prop:discriminative}}}

\begin{paragraph}{$(\Leftarrow)$}
    It is obvious since for any $f \in \mathcal{F}_{\textup{ReLU}}$, we have 
    $$
    \int f(\mathbf{z})(d\mathcal{P}(\mathbf{z})-d\mathcal{Q}(\mathbf{z})) = 0.
    $$
\end{paragraph}

\begin{paragraph}{$(\Rightarrow)$} 
Lemma \ref{lem:discriminative} implies 
$\mathbf{c}^{\top} \mathbf{X} \stackrel{\textup{d}}{=} \mathbf{c}^{\top} \mathbf{Y}$ for all $\mathbf{c} \in \mathbb{R}^d$, 
which, in turn, implies $\mathcal{P} \equiv \mathcal{Q}$ by the uniqueness of the characteristic function.
\end{paragraph}

\end{proof}


\subsection{Proof of Theorem \ref{thm:holder}}
\label{pf:holder}
We use the following corollary to prove Theorem \ref{thm:holder}.

\begin{lemma}[Corollary 2.4 of \citet{Yang2024}] \label{lem:yang}
    Let $d \geq 3$ and $\beta > 0$. For any $h \in \mathcal{H}_d^{\beta}$,
\begin{enumerate}
    \item[(1)] if $\beta > \frac{d+3}{2}$, then for any sufficiently large L, there exist $f_1, f_2, \ldots, f_L \in \mathcal{F}_{\textup{ReLU}}$ and $a_1, a_2, \ldots, a_L$ such that $\sum_{i=1}^L |a_i| \leq M$ and
    $$
    \left\| h - \sum_{i=1}^L a_i f_i \right\|_{\infty} \lesssim L^{-\frac{d+3}{2d}},
    $$
    where $M$ is a constant depending on $\beta$ and $d$,
    
    \item[(2)] if $\beta < \frac{d+3}{2}$, then for any sufficiently large $L$ and $M$, there exist $f_1, f_2, \ldots, f_L \in \mathcal{F}_{\textup{ReLU}}$ and $a_1, a_2, \ldots, a_L$ such that $\sum_{i=1}^L |a_i| \leq M$ and
    $$
    \left\| h - \sum_{i=1}^L a_i f_i \right\|_{\infty} \lesssim L^{-\frac{\beta}{d}} \vee M^{-\frac{2\beta}{d + 3 - 2\beta}},
    $$
    Thus, the upper bound becomes $\mathcal{O}(L^{-\frac{\beta}{d}})$ holds when $M \gtrsim L^{-\frac{d+3-2\beta}{2d}}.$
    \end{enumerate}
\end{lemma}

\begin{remark}
We do not consider the case of $\beta=\frac{d+3}{2}$ for the sake of simplicity.
In fact, \citet{Yang2024} proves the following approximation property:
    \begin{itemize}
        \item[(i)] if $\beta = \frac{d+3}{2}$ is an even integer, then for any sufficiently large L, there exist $f_1, f_2, \ldots, f_L \in \mathcal{F}_{\textup{ReLU}}$ and $a_1, a_2, \ldots, a_L$ such that $\sum_{i=1}^L |a_i| \leq M$ and
        $$
        \left\| h - \sum_{i=1}^L a_i f_i \right\|_{\infty} \lesssim L^{-\frac{d+3}{2d}},
        $$
        where $M$ is a constant depending on $\beta$ and $d$,
        
        \item[(ii)] if $\beta = \frac{d+3}{2}$ is not an even integer, then for any sufficiently large $L$ and $M \asymp \sqrt{\log L}$, there exist $f_1, f_2, \ldots, f_L \in \mathcal{F}_{\textup{ReLU}}$ and $a_1, a_2, \ldots, a_L$ such that $\sum_{i=1}^L |a_i| \leq M$ and
        $$
        \left\| h - \sum_{i=1}^L a_i f_i \right\|_{\infty} \lesssim L^{-\frac{d+3}{2d}} \sqrt{\log L}.
        $$
    \end{itemize}
\end{remark}

\begin{proof}{\textbf{of Theorem \ref{thm:holder}}}
\begin{paragraph}{(1) $\beta > \frac{d+3}{2}$:}
    By Lemma \ref{lem:yang}, if $d>3$, for all $h \in \mathcal{H}_d^{\beta}$,    
    when $\beta > \frac{d+3}{2}$, there exists a constant $M$ depending on $d$ and $\beta$ such that (i)
    \begin{equation*} \label{upperbound_1}
        \left\|h - \sum_{i=1}^L a_i f_i \right\|_{\infty} \lesssim L^{-\frac{d+3}{2d}}, 
    \end{equation*}
    and (ii) $\sum_{i=1}^L |a_i| \leq M$. Then we have    
    \begin{align}
        & \Bigg| \int h(\mathbf{X}) \left( d\mathcal{P}(\mathbf{X})-d\mathcal{Q}(\mathbf{X}) \right) \Bigg| \nonumber \\ 
        & = \Bigg| \int \left( h(\mathbf{X}) - \sum_{i=1}^L a_i f_i(\mathbf{X}) + \sum_{i=1}^L a_i f_i(\mathbf{X}) \right) \left( d\mathcal{P}(\mathbf{X}) - d\mathcal{Q}(\mathbf{X}) \right) \Bigg| \nonumber \\ 
        & \leq \Bigg| \int \sum_{i=1}^L a_i f_i(\mathbf{X}) \left( d\mathcal{P}(\mathbf{X}) - d\mathcal{Q}(\mathbf{X}) \right) \Bigg| + \Bigg| \int \left( h(\mathbf{X}) - \sum_{i=1}^L a_i f_i(\mathbf{X}) \right) \left( d\mathcal{P}(\mathbf{X}) - d\mathcal{Q}(\mathbf{X}) \right) \Bigg| \nonumber \\
        & \lesssim \left| \int \sum_{i=1}^L a_i f_i(\mathbf{X}) \left( d\mathcal{P}(\mathbf{X}) - d\mathcal{Q}(\mathbf{X}) \right) \right| + L^{-\frac{d+3}{2d}} \nonumber \\
        & \leq \sum_{i=1}^L |a_i| \left| \int f_i(\mathbf{X}) \left( d\mathcal{P}(\mathbf{X}) - d\mathcal{Q}(\mathbf{X}) \right) \right| + L^{-\frac{d+3}{2d}} \nonumber \\
        & \leq M \cdot d_{\mathcal{F}_{\textup{ReLU}}}(\mathcal{P},\mathcal{Q}) + L^{-\frac{d+3}{2d}} \nonumber \\
        & \lesssim d_{\mathcal{F}_{\textup{ReLU}}}(\mathcal{P},\mathcal{Q}),  \nonumber 
    \end{align}
    where the last inequality holds with sufficiently large $L$.
\end{paragraph}

\begin{paragraph}{(2) $\beta < \frac{d+3}{2}$:} 
    Let $d \in \mathbb{N}$ and $\beta > 0$.
    By Lemma \ref{lem:yang} (3), for all $h \in \mathcal{H}_d^{\beta}$, there exist $a_1, \ldots, a_{L} \in \mathbb{R}$ and $f_1,\ldots,f_L \in \mathcal{F}_{\textup{ReLU}}$ such that $\|h - \sum_{i=1}^L a_i f_i\|_{\infty} \lesssim L^{-\frac{\beta}{d}} \vee M^{-\frac{2\beta}{d-2\beta+3}}$ and $\sum_{i=1}^L |a_i| \leq M$. Then we have
    
    \begin{align}
        & \Bigg| \int h(\mathbf{X}) \left( d\mathcal{P}(\mathbf{X})-d\mathcal{Q}(\mathbf{X}) \right) \Bigg| \nonumber \\ 
        & = \Bigg| \int \left( h(\mathbf{X}) - \sum_{i=1}^L a_i f_i(\mathbf{X}) + \sum_{i=1}^L a_i f_i(\mathbf{X}) \right) \left( d\mathcal{P}(\mathbf{X}) - d\mathcal{Q}(\mathbf{X}) \right) \Bigg| \nonumber \\ 
        & \leq \Bigg| \int \sum_{i=1}^L a_i f_i(\mathbf{X}) \left( d\mathcal{P}(\mathbf{X}) - d\mathcal{Q}(\mathbf{X}) \right) \Bigg| + \Bigg| \int \left( h(\mathbf{X}) - \sum_{i=1}^L a_i f_i(\mathbf{X}) \right) \left( d\mathcal{P}(\mathbf{X}) - d\mathcal{Q}(\mathbf{X}) \right) \Bigg| \nonumber \\
        & \lesssim \left| \int \sum_{i=1}^L a_i f_i(\mathbf{X}) \left( d\mathcal{P}(\mathbf{X}) - d\mathcal{Q}(\mathbf{X}) \right) \right| + \left( L^{-\frac{\beta}{d}} \vee M^{-\frac{2 \beta}{d-2\beta+3}} \right) \nonumber \\
        & \leq \sum_{i=1}^L |a_i| \left| \int f_i(\mathbf{X}) \left( d\mathcal{P}(\mathbf{X}) - d\mathcal{Q}(\mathbf{X}) \right) \right| + \left( L^{-\frac{\beta}{d}} \vee M^{-\frac{2 \beta}{d-2\beta+3}} \right) \nonumber \\
        & \leq M \cdot d_{\mathcal{F}_{\textup{ReLU}}}(\mathcal{P},\mathcal{Q}) + \left( L^{-\frac{\beta}{d}} \vee M^{-\frac{2 \beta}{d-2\beta+3}} \right). \nonumber 
    \end{align}
    Thus, we can complete the proof if we choose $M = d_{\mathcal{F}_{\textup{ReLU}}}(\mathcal{P},\mathcal{Q})^{-\frac{d-2\beta+3}{d+3}}$ for sufficiently large $L$.
\end{paragraph} 

\end{proof}


\subsection{Proof of Theorem \ref{thm:conv_relu}} \label{pf:conv_relu}

\begin{proof}
    We consider the case of $n \neq m$.
    To prove Theorem \ref{thm:conv_relu}, we will show
    \begin{align}
        \mathbb{P} \Bigg( \Big| d_{\mathcal{F}_{\textup{ReLU}}}(\mathcal{P},\mathcal{Q}) & - d_{\mathcal{F}_{\textup{ReLU}}}(\widehat{\mathcal{P}}_n, \widehat{\mathcal{Q}}_m) \Big| - 2 \left(\frac{1}{\sqrt{n}} + \frac{1}{\sqrt{m}} \right) > \epsilon \Bigg) \leq 2 \exp \left(-\frac{\epsilon^2 nm}{8(n+m)} \right). \label{theo1-0}
    \end{align}
    Let us begin with the absolute term on the left-hand side of (\ref{theo1-0}).
    \begin{align}
        \Big| d_{\mathcal{F}_{\textup{ReLU}}}(\mathcal{P},\mathcal{Q}) - d_{\mathcal{F}_{\textup{ReLU}}}(\widehat{\mathcal{P}}_n, \widehat{\mathcal{Q}}_m) \Big| & = \left| \sup_{f \in \mathcal{F}_{\textup{ReLU}}} \Big| \mathbb{E}_{\mathbf{X} \sim \mathcal{P}} f - \mathbb{E}_{\mathbf{Y} \sim \mathcal{Q}} f \Big| - \sup_{f \in \mathcal{F}_{\textup{ReLU}}} \bigg| \frac{1}{n} \sum_{i=1}^n f(\mathbf{X}_i) - \frac{1}{m} \sum_{i=1}^m f(\mathbf{Y}_i) \bigg| \right| \nonumber \\
        & \leq  \sup_{f \in \mathcal{F}_{\textup{ReLU}}} \left| \mathbb{E}_{\mathbf{X} \sim \mathcal{P}} f - \mathbb{E}_{\mathbf{Y} \sim \mathcal{Q}} f - \frac{1}{n} \sum_{i=1}^n f(\mathbf{X}_i) + \frac{1}{m} \sum_{i=1}^m f(\mathbf{Y}_i) \right| \nonumber \\
        & \coloneqq G(\mathcal{P},\mathcal{Q},\mathbf{X}^n,\mathbf{Y}^m),\label{theo1-g}
    \end{align}
    where $\mathbf{X}^n=(\mathbf{X}_1,\ldots,\mathbf{X}_n)$ and $\mathbf{Y}^m$ is defined similarly.
    Let $\mathbf{X}_{(i)}^{\prime}$ be a dataset that differs only for the $i$-th component of $\mathbf{X}^n$, for $i = 1,\ldots,n$, and $\mathbf{Y}_{(j)}^{\prime}$ is defined similarly  for $j = 1,\ldots,m$. 
    Then we have
    
    \begin{align*}
        G (\mathcal{P},\mathcal{Q},\mathbf{X}^n, \mathbf{Y}^m) & - G(\mathcal{P},\mathcal{Q}, \mathbf{X}_{(i)}^{\prime}, \mathbf{Y}^m) \\
        & \leq \sup_{f \in \mathcal{F}_{\textup{ReLU}}} \left| \mathbb{E}_{\mathbf{X} \sim \mathcal{P}} f - \mathbb{E}_{\mathbf{Y} \sim \mathcal{Q}} f - \frac{1}{n} \sum_{i=1}^n f(\mathbf{X}_i) + \frac{1}{m} \sum_{i=1}^m f(\mathbf{Y}_i) \right| \\
        & \quad - \sup_{f \in \mathcal{F}_{\textup{ReLU}}} \left| \mathbb{E}_{\mathbf{X} \sim \mathcal{P}} f - \mathbb{E}_{\mathbf{Y} \sim \mathcal{Q}} f - \frac{1}{n} \left( \sum_{k \neq i} f(\mathbf{X}_k) + f(\mathbf{X}_i^{\prime}) \right) + \frac{1}{m} \sum_{i=1}^m f(\mathbf{Y}_i) \right| \\
        & \leq \sup_{f \in \mathcal{F}_{\textup{ReLU}}} \left| - \frac{1}{n} \sum_{i=1}^n f(\mathbf{X}_i) + \frac{1}{n} \left( \sum_{k \neq i} f(\mathbf{X}_k) + f(\mathbf{X}_i^{\prime}) \right) \right| \\
        & = \sup_{f \in \mathcal{F}_{\textup{ReLU}}} \left| \frac{1}{n} f(\mathbf{X}_i^{\prime}) - \frac{1}{n} f(\mathbf{X}_i) \right| \\
        & \leq \frac{2}{n} \sup_{f \in \mathcal{F}_{\textup{ReLU}}} \|f\|_{\infty} \\
        & \leq \frac{4}{n},
    \end{align*}
    where the last inequality holds since $\left\| f \right\|_2 \leq \left\| \cdot ^{\top}\boldsymbol{\theta}+\mu \right\|_2 = | \cdot ^{\top}\boldsymbol{\theta}+\mu| \leq 2$, for all $ \cdot \in \mathbb{B}^d$, $\boldsymbol{\theta} \in \mathbb{S}^{d-1}$ and $\mu \in [-1,1]$. Similarly, $G(\mathcal{P},\mathcal{Q},\mathbf{X}^n,\mathbf{Y}^m) - G(\mathcal{P},\mathcal{Q},\mathbf{X}^n,\mathbf{Y}_{(j)}^{\prime}) \leq 4/m$.
    Applying Lemma \ref{lem:McD} (McDiarmid's inequality) with $c_i = 4/n$ ($i=1,\ldots,n$) for $\mathbf{X}$ and $c_j = 4/m$ ($j=1,\ldots,m$) for $\mathbf{Y}$ (that is, a denominator in the exponent of Lemma \ref{lem:McD} of $16(n+m)/(nm)$), we obtain
    \begin{align} \label{theo1-1}
        \mathbb{P} ( G(\mathcal{P},\mathcal{Q},\mathbf{X}^n,\mathbf{Y}^m) - \mathbb{E}_{\mathbf{X},\mathbf{Y}} & G(\mathcal{P},\mathcal{Q},\mathbf{X}^n, \mathbf{Y}^m) > \epsilon ) \leq 2 \exp \left( - \frac{\epsilon^2 nm}{8(n+m)} \right).
    \end{align}
The proof would be done if we derive the upper bound of the expectation of $G(\mathcal{P},\mathcal{Q},\mathbf{X}^n,\mathbf{Y}^m)$:
    \begin{align*}
        \mathbb{E}_{\mathbf{X},\mathbf{Y}} G(\mathcal{P},\mathcal{Q},\mathbf{X}^n,\mathbf{Y}^m) \leq 4\left( 1+\sqrt{\frac{2}{\pi}} \right) \left( \frac{1}{\sqrt{n}} + \frac{1}{\sqrt{m}} \right)
    \end{align*}    
    
    Let $\varepsilon_1,\ldots, \varepsilon_n$ be independent Radamacher random variables (i.e
    $\mathbb{P}(\varepsilon_i=-1)=\mathbb{P}(\varepsilon_i=1)=1/2.$ The empirical Radamacher average (Rademacher complexity) of a function class $\mathcal{F}$ is defined as
    $$
    \widehat{\mathcal{R}}_n(\mathcal{F}) \coloneqq \mathbb{E}_{\varepsilon} \left[ \sup_{f \in \mathcal{F}} \Big| \frac{1}{n} \sum_{i=1}^n \varepsilon_i f(\mathbf{X}_i) \Big| \Bigg| \mathbf{X}_1, \ldots, \mathbf{X}_{n} \right],
    $$
    where $\mathbb{E}_{\varepsilon}$ denotes the expectation over all $\varepsilon_i$. 
    Since $\left\{ \sum_{i=1}^n \varepsilon_i \right\}$ is a simple symmetric random walk, its expected absolute distance after $n$ steps is given by \citep{Grünbaum1960Projection, frederick1961probability, Hughes1995Random, KönigSchüttTomczakJaegermann1999}
    \begin{align} \label{eq:expect_abs_rademacher}
        \mathbb{E}\left|\sum_{i=1}^n \varepsilon_i \right| \leq \sqrt{\frac{2n}{\pi}}.
    \end{align}
    Thus, by the Talagrand’s Contraction Lemma \citep{LedouxTalagrand1991} (see Lemma \ref{Ledoux} of Appendix \ref{sec:math_def}), we have an upper bound of the empirical Rademacher average for the discriminator class $\mathcal{F}_{\textup{ReLU}}$ as
    \begin{align*}
        \widehat{\mathcal{R}}_n(\mathcal{F}_{\textup{ReLU}})  
        & = \mathbb{E} \sup_{ (\boldsymbol{\theta}, \mu) \in \mathbb{S}^{d-1} \times [-1,1] } \left| \frac{1}{n} \sum_{i=1}^n \varepsilon_i \left( \boldsymbol{\theta}^{\top} \mathbf{X}_i + \mu \right)_{+} \right| \\ 
        & \leq \frac{2}{n} \mathbb{E} \sup_{ (\boldsymbol{\theta}, \mu) \in \mathbb{S}^{d-1} \times [-1,1] } \left| \sum_{i=1}^n \varepsilon_i \left( \boldsymbol{\theta}^{\top} \mathbf{X}_i + \mu \right) \right| && \tag{by Lemma \ref{Ledoux}} \\   
        & \leq \frac{2}{n} \mathbb{E} \sup_{ (\boldsymbol{\theta}, \mu) \in \mathbb{S}^{d-1} \times [-1,1] } \left\{ \left| \sum_{i=1}^n \boldsymbol{\theta}^{\top} (\varepsilon_i \mathbf{X}_i) \right| + \left| \sum_{i=1}^n \varepsilon_i \mu \right| \right\} && \tag{by the triangle inequality} \\   
        & \leq \frac{2}{n} \mathbb{E} \sup_{ (\boldsymbol{\theta}, \mu) \in \mathbb{S}^{d-1} \times [-1,1] } \left\{ \left\|\boldsymbol{\theta}^{\top}\right\|_2 \left\| \sum_{i=1}^n \varepsilon_i \mathbf{X}_i \right\|_2 + |\mu| \left| \sum_{i=1}^n \varepsilon_i \right| \right\}  && \tag{by Cauchy-Schwarz ineq.} \\   
        & \leq \frac{2}{n} \mathbb{E} \left\{ \left\| \sum_{i=1}^n \varepsilon_i \mathbf{X}_i \right\|_2 + \left|\sum_{i=1}^n \varepsilon_i \right| \right\}\\
        & \leq \frac{2}{n} \left\{ \mathbb{E} \sqrt{ \left\langle \sum_{i=1}^n \varepsilon_i \mathbf{X}_i, \sum_{i=1}^n \varepsilon_i \mathbf{X}_i \right\rangle } + \sqrt{\frac{2n}{\pi}} \right\} && \tag{by (\ref{eq:expect_abs_rademacher})} \\   
        & \leq \frac{2}{n} \left\{ \mathbb{E} \sqrt{ \sum_{i=1}^n \sum_{j=1}^n \varepsilon_i \varepsilon_j \left\langle \mathbf{X}_i, \mathbf{X}_j \right\rangle } + \sqrt{\frac{2n}{\pi}} \right\}  && \tag{by the linearity of inner product} \\   
        & \leq \frac{2}{n} \left\{ \sqrt{ \, \mathbb{E} \sum_{1 \leq i,j \leq n} \varepsilon_i \varepsilon_j \left\langle \mathbf{X}_i, \mathbf{X}_j \right\rangle } + \sqrt{\frac{2n}{\pi}} \right\}  && \tag{by Jensen's ineq.} \\   
        & = \frac{2}{n} \left\{ \sqrt{ \sum_{1 \leq i,j \leq n} \mathbb{E} (\varepsilon_i \varepsilon_j) \left\langle \mathbf{X}_i, \mathbf{X}_j \right\rangle } + \sqrt{\frac{2n}{\pi}} \right\} \\   
        & = \frac{2}{n} \left\{\sqrt{ \sum_{i=1}^n \|\mathbf{X}_i\|_2^2 } + \sqrt{\frac{2n}{\pi}} \right\} && \tag{by properties of $\varepsilon$} \\   
        & \leq \frac{2}{\sqrt{n}} \left( 1 + \sqrt{\frac{2}{\pi}} \right),
    \end{align*}
    where the expectation is taken over $\{\varepsilon_i\}_{i=1}^n$.

    Finally, we will derive an upper bound of $\mathbb{E}_{\mathbf{X},\mathbf{Y}} G(\mathcal{P},\mathcal{Q},\mathbf{X}^n,\mathbf{Y}^m)$ in terms of the empirical Radamacher averages
    by use of the symmetrization technique \citep{van1996}.
        Let $\mathbf{X}^{\prime,n} \coloneqq (\mathbf{X}_1^{\prime}, \ldots, \mathbf{X}_n^{\prime})$ be an independent sample whose distribution is the same as
        that of $\mathbf{X}^n,$ and define $\mathbf{Y}^{\prime,m}$ similarly.   
    Then we have
    \begin{align}
        \mathbb{E}_{\mathbf{X}, \mathbf{Y}} & G(\mathcal{P},\mathcal{Q},\mathbf{X}^n,\mathbf{Y}^m) \nonumber \\
        & = \mathbb{E}_{\mathbf{X}^n, \mathbf{Y}^m} \sup _{f \in \mathcal{F}_{\textup{ReLU}}} \left| \mathbb{E}_{\mathbf{X} \sim \mathcal{P}} f(\mathbf{X}) - \mathbb{E}_{\mathbf{Y} \sim \mathcal{Q}} f(\mathbf{Y}) - \frac{1}{n} \sum_{i=1}^n f(\mathbf{X}_i) + \frac{1}{m} \sum_{j=1}^m f(\mathbf{Y}_j) \right| \nonumber \\ 
        & = \mathbb{E}_{\mathbf{X}^n, \mathbf{Y}^m} \sup _{f \in \mathcal{F}_{\textup{ReLU}}} \left| \mathbb{E}_{\mathbf{X}^{\prime} \sim P} \left( \frac{1}{n} \sum_{i=1}^n f(\mathbf{X}_i^{\prime}) \right) - \frac{1}{n} \sum_{i=1}^n f(\mathbf{X}_i) - \mathbb{E}_{\mathbf{Y}^{\prime} \sim Q} \left( \frac{1}{m} \sum_{j=1}^m f(\mathbf{Y}_j^{\prime}) \right) + \frac{1}{m} \sum_{j=1}^m f(\mathbf{Y}_j) \right| \nonumber \\
        & \leq \mathbb{E}_{\mathbf{X}^n, \mathbf{Y}^m, \mathbf{X}^{\prime, n}, \mathbf{Y}^{\prime, m}} \sup _{f \in \mathcal{F}_{\textup{ReLU}}} \left| \frac{1}{n} \sum_{i=1}^n f(\mathbf{X}_i^{\prime}) - \frac{1}{n} \sum_{i=1}^n f(\mathbf{X}_i) - \frac{1}{m} \sum_{j=1}^m f(\mathbf{Y}_j^{\prime}) + \frac{1}{m} \sum_{j=1}^m f(\mathbf{Y}_j) \right| &&\tag{by Jensen's ineq.} \nonumber \\
        & = \mathbb{E}_{\mathbf{X}^n, \mathbf{Y}^m, \mathbf{X}^{\prime, n}, \mathbf{Y}^{\prime, m}, \varepsilon, \varepsilon^{\prime}} \sup _{f \in \mathcal{F}_{\textup{ReLU}}}\left| \frac{1}{n} \sum_{i=1}^n \varepsilon_i\left(f(\mathbf{X}_i^{\prime}) - f(\mathbf{X}_i)\right) + \frac{1}{m} \sum_{j=1}^m \varepsilon_j^{\prime}\left(f(\mathbf{Y}_j^{\prime}) - f(\mathbf{Y}_j)\right) \right| \nonumber \\        
        & \leq \mathbb{E}_{\mathbf{X}^n, \mathbf{X}^{\prime, n}, \varepsilon} \sup _{f \in \mathcal{F}_{\textup{ReLU}}} \left| \frac{1}{n} \sum_{i=1}^n \varepsilon_i \left( f(\mathbf{X}_i^{\prime}) - f(\mathbf{X}_i)\right) \right| + \mathbb{E}_{\mathbf{Y}^m, \mathbf{Y}^{\prime, m}, \varepsilon} \sup _{f \in \mathcal{F}_{\textup{ReLU}}} \left| \frac{1}{m} \sum_{j=1}^m \varepsilon_j\left(f(\mathbf{Y}_j^{\prime}) - f(\mathbf{Y}_j) \right) \right|  &&\tag{by Triangle ineq.} \nonumber \\
        & = \mathbb{E}_{\mathbf{X}^n, \mathbf{X}^{\prime, n}} \mathbb{E}_{\varepsilon} \left[ \sup _{f \in \mathcal{F}_{\textup{ReLU}}} \left| \frac{1}{n} \sum_{i=1}^n \varepsilon_i \left( f(\mathbf{X}_i^{\prime}) - f(\mathbf{X}_i)\right) \right| \Big| \mathbf{X}_1, \ldots, \mathbf{X}_{n} \right] \nonumber \\
        & \qquad \qquad + \mathbb{E}_{\mathbf{Y}^m, \mathbf{Y}^{\prime, m}} \mathbb{E}_{\varepsilon} \left[\sup _{f \in \mathcal{F}_{\textup{ReLU}}} \left| \frac{1}{m} \sum_{j=1}^m \varepsilon_j\left(f(\mathbf{Y}_j^{\prime}) - f(\mathbf{Y}_j) \right) \right| \Big| \mathbf{Y}_1, \ldots, \mathbf{Y}_{m} \right] \nonumber \\
        & \leq 2\left\{ \widehat{\mathcal{R}}_n(\mathcal{F}_{\textup{ReLU}}) + \widehat{\mathcal{R}}_m(\mathcal{F}_{\textup{ReLU}}) \right\} \hspace{8cm} \nonumber \\
        & \leq 4\left( 1+\sqrt{2/\pi} \right) \left( 1/\sqrt{n} + 1/\sqrt{m} \right).  \hspace{8cm} \label{theo1-3}
    \end{align}
    
 Combining (\ref{theo1-1}) with (\ref{theo1-3}), we complete the proof of Theorem \ref{thm:conv_relu}. In addition, the strong consistency of $d_{\mathcal{F}_{\textup{ReLU}}}(\widehat{\mathcal{P}}_n, \widehat{\mathcal{Q}}_n)$ can be proved by use of the Borel-Cantelli lemma.
\end{proof}

\subsection{Proof of Theorem \ref{thm:causal}} \label{pf:causal}

We begin by deriving the covering number of $\mathcal{F}_{\textup{ReLU}}$ to demonstrate that the assumption in Lemma 8.4 of \citet{geer2000empirical} (Lemma \ref{lem:vandeGeer} in Appendix \ref{sec:math_def}) holds, which is necessary for proving Theorem \ref{thm:causal}.

\begin{lemma}[covering number of ReLU-IPM] 
    For any $\epsilon>0$, we have
    $$
    \mathcal{N}(\mathcal{F}_{\textup{ReLU}}, \|\cdot\|_{1,\mathcal{P}}, \epsilon) \leq \mathcal{O} \left(1+\frac{1}{\epsilon}\right)^{d+1}.
    $$
\end{lemma}
\begin{proof}
    There exists $\epsilon/2$-cover $V_{\epsilon}$ of $\mathbb{S}^{d-1}$ in the metric $\|\cdot\|_2$ with cardinality $|V_{\epsilon}| = \mathcal{O}((1+1/\epsilon)^{d})$.
    Hence, for any $\boldsymbol{\theta} \in \mathbb{S}^{d-1}$, we can choose $\Tilde{\boldsymbol{\theta}} \in V_{\epsilon}$ such that $\|\boldsymbol{\theta} -\Tilde{\boldsymbol{\theta}}\|_2 \leq \epsilon/2$. In addition, there exists $\epsilon/2$-cover $U_{\epsilon}$ of $[-1,1]$ in the metric $\|\cdot\|_2$ with cardinality $|U_{\epsilon}| = \mathcal{O}(1+1/\epsilon)$.
    Hence, for any $\mu \in [-1,1]$, we can choose $\Tilde{\mu} \in U_{\epsilon}$ such that $|\mu-\Tilde{\mu}| \leq \epsilon/2$.
    Therefore, letting $\Tilde{f}(x) = (\tilde{\boldsymbol{\theta}}^{\top}\mathbf{x} + \tilde{\mu})_{+}$, we have, for all $\mathbf{x} \in \mathbb{B}^d$,
    \begin{align*}
        \left| f(\mathbf{x}) - \Tilde{f}(\mathbf{x}) \right| & = \left| (\boldsymbol{\theta}^{\top}\mathbf{x} + \mu)_{+} - (\tilde{\boldsymbol{\theta}}^{\top}\mathbf{x} + \tilde{\mu})_{+} \right| \\
        & \leq \left| \boldsymbol{\theta}^{\top}\mathbf{x} + \mu - (\tilde{\boldsymbol{\theta}}^{\top}\mathbf{x} + \tilde{\mu}) \right| && \text{(by Lips. conti. of ReLU)} \\
        & \leq \big\| \boldsymbol{\theta} - \tilde{\boldsymbol{\theta}} \big\|_2 \cdot \big\| \mathbf{x} \big\|_2 + \left| \mu - \tilde{\mu} \right| && \text{(by Cauchy-Schwarz ineq.)} \\
        & \leq \epsilon.
    \end{align*}
\end{proof}

The following lemma is also needed to prove Theorem \ref{thm:causal}.

\begin{lemma} \label{lemma_causal_true_conv}
    Let $u(\cdot) \coloneqq \frac{\pi(\cdot)}{1-\pi(\cdot)}$.
    For $\{(\mathbf{X}_i, T_i)\}_{i=1}^n$, 
    we define
    $\tilde{\mathbf{w}} \coloneqq (\tilde{w}_{1}, \dots, \tilde{w}_{n})^{\top} \in \mathcal{W}^{+}$ as 
    \begin{flalign*}
         && \tilde{w}_{i} \coloneqq \frac{\mathbb{I}(T_i = 0) u(\mathbf{X}_i)}{ \sum_{i : T_i = 0} {u(\mathbf{X}_i)}}, && i \in [n]
    \end{flalign*}
    Then, we have
    $$d_{\mathcal{F}_{\textup{ReLU}}}\left( \mathbb{P}_{0,n}^{\tilde{\mathbf{w}}}, \mathbb{P}_{1,n} \right) = \mathcal{O}_p(n^{-1/2}).$$
\end{lemma}

\begin{proof}{\textbf{of Lemma \ref{lemma_causal_true_conv}}} 
For given $(\mathbf{X}_1,\ldots, \mathbf{X}_n)=(\mathbf{x}_1,\ldots,\mathbf{x}_n),$ 
    we can write 
    \begin{align}
        d_{\mathcal{F}_{\textup{ReLU}}} \left( \mathbbm{P}_{0,n}^{\tilde{\mathbf{w}}_{i}}, \mathbbm{P}_{1,n} \right) 
        & = \sup_{f\in \mathcal{F}_{\textup{ReLU}}} 
        \left| \sum_{i=1}^n \left( \tilde{w}_{i} - \frac{\mathbb{I}(T_i = 1)}{n_1} \right) f(\mathbf{x}_i) \right| \nonumber \\
        & \leq \sup_{f\in \mathcal{F}_{\textup{ReLU}}} 
        \left| \sum_{i=1}^n \left( \tilde{w}_{i} - \frac{u(\mathbf{x}_i)\mathbb{I}(T_i = 0)}{n_1}  \right) f(\mathbf{x}_i) \right| \label{w0_ipm_eq_1} \\
        & \quad + \sup_{f\in \mathcal{F}_{\textup{ReLU}}} 
        \left| \sum_{i=1}^n \left( \frac{u(\mathbf{x}_i)\mathbb{I}(T_i = 0)}{n_1} - \frac{\mathbb{I}(T_i = 1)}{n_1} \right) f(\mathbf{x}_i) \right|. \label{w0_ipm_eq_2}
    \end{align}    
        
    For (\ref{w0_ipm_eq_1}), we have
    \begin{align*}
        \frac{n_1}{n} \times (\ref{w0_ipm_eq_1})  
        & = \frac{\sum_{i=1}^n \mathbb{I}(T_i = 1)}{n} \cdot \sup_{f\in \mathcal{F}_{\textup{ReLU}}} 
        \left| \sum_{i=1}^n \left(  \frac{\mathbb{I}(T_i = 0) u(\mathbf{x}_i)}{ \sum_{j=1}^n \mathbb{I}(T_j = 0) u(\mathbf{x}_j)} - \frac{u(\mathbf{x}_i)\mathbb{I}(T_i = 0)}{\sum_{j=1}^n \mathbb{I}(T_j = 1)}  \right) f(\mathbf{x}_i) \right| \\
        & = \left| \frac{\sum_{i=1}^n \mathbb{I}(T_i = 1)}{n} - \frac{\sum_{j=1}^n \mathbb{I}(T_j = 0) u(\mathbf{x}_j)}{n} \right| \cdot
        \sup_{f\in \mathcal{F}_{\textup{ReLU}}} 
        \left| \frac{\sum_{i=1}^n \mathbb{I}(T_i = 0) u(\mathbf{x}_i)  f(\mathbf{x}_i)}{\sum_{j=1}^n   \mathbb{I}(T_j = 0) u(\mathbf{x}_j) } \right| &&\tag{by Hoeffding's ineq.} \\
        & \leq 2 \cdot \left| \frac{\sum_{i=1}^n \mathbb{I}(T_i = 1)}{n} - \frac{\sum_{i=1}^n \mathbb{I}(T_i = 0) u(\mathbf{x}_i)}{n}   \right| \\
        & = \mathcal{O}_p \left( n^{-1/2} \right),
    \end{align*}
    where $\mathcal{O}_p$ is defined in probability with respect to $\{T_i\}_{i=1}^n$
    conditional on $(\mathbf{X}_1,\ldots, \mathbf{X}_n)=(\mathbf{x}_1,\ldots,\mathbf{x}_n).$
    Here, the last inequality holds by the inequality $\left\| f(z) \right\|_2 \leq | z^{\top}\boldsymbol{\theta}+\mu| \leq 2$.
    
    For (\ref{w0_ipm_eq_2}), let $W_i \coloneqq u(\mathbf{x}_i)\mathbb{I}(T_i = 0) - \mathbb{I}(T_i = 1).$ Then, we have $\mathbb{E}(W_i) = 0$ and $\max_{i \in [n]} K^2 (\mathbb{E} e^{|W_i|^2 / K^2 } -1) \leq \sigma_0^2$, for $K=\eta^{-1}$ and $\sigma_0^2 = K^2(e-1)$. Hence, by Theorem \ref{thm:conv_relu} and Lemma 8.4 of \cite{geer2000empirical} (see Lemma \ref{lem:vandeGeer} in Appendix \ref{sec:math_def}), there exists a constant $c>0$ such that 
    \begin{align*}
        \mathbb{P} \left( \sup_{f\in \mathcal{F}_{\textup{ReLU}}} \frac{\frac{1}{\sqrt{n}} \sum_{i=1}^n W_i f(\mathbf{x}_i) }{\sqrt{\frac{1}{n}\sum_{i=1}^n f(\mathbf{x}_i)^2}} \geq c^{\frac{3}{2}} \, \middle| \, (\mathbf{X}_1,\ldots, \mathbf{X}_n)=(\mathbf{x}_1,\ldots,\mathbf{x}_n)  
        \right) \leq c e^{-c},
    \end{align*}
    and hence
    \begin{align*}
        \mathbb{P} \left( \sup_{f\in \mathcal{F}_{\textup{ReLU}}} \frac{1}{\sqrt{n}} \sum_{i=1}^n W_i f(\mathbf{x}_i) \geq 2 \, c^{\frac{3}{2}} \, \middle| \, (\mathbf{X}_1,\ldots, \mathbf{X}_n)=(\mathbf{x}_1,\ldots,\mathbf{x}_n)
        \right) \leq c e^{-c}.
    \end{align*}
    Then, we have
    \begin{align*}
        \frac{n_1}{n} \times (\ref{w0_ipm_eq_2}) 
        & = \sup_{f\in \mathcal{F}_{\textup{ReLU}}}
        \left| \sum_{i=1}^n \left( \frac{u(\mathbf{x}_i)\mathbb{I}(T_i = 0)}{n} - \frac{\mathbb{I}(T_i = 1)}{n} \right) f(\mathbf{x}_i) \right| \\
        & = \frac{1}{n}\sup_{f\in \mathcal{F}_{\textup{ReLU}}}
        \left| \sum_{i=1}^n W_i f(\mathbf{x}_i) \right| \\
        & = \mathcal{O}_p \left( n^{-1/2} \right),
    \end{align*}
    where $\mathcal{O}_p$ is defined in probability with respect to $\{T_i\}_{i=1}^n$
    conditional on $(\mathbf{X}_1,\ldots, \mathbf{X}_n)=(\mathbf{x}_1,\ldots,\mathbf{x}_n).$
    
    To sum up, we get
    \begin{align*}
        d_{\mathcal{F}_{\textup{ReLU}}}\left( \mathbb{P}_{0,n}^{\tilde{\mathbf{w}}}, \mathbb{P}_{1,n} \right) = \mathcal{O}_p \left( n^{-1/2} \right)
    \end{align*}
    conditional on $(\mathbf{X}_1,\ldots, \mathbf{X}_n)=(\mathbf{x}_1,\ldots,\mathbf{x}_n).$
    Since it holds for every $\{\mathbf{x}_i\}_{i=1}^n \in \mathcal{X}^n$, we obtain the assertion and complete the proof.
\end{proof}


\begin{proof}{\textbf{of Theorem \ref{thm:causal}}}
    We can decompose the error of $\widehat{\mathrm{ATT}}^{\hat{\mathbf{w}}}$ as 
    \begin{align}
        \widehat{\operatorname{ATT}}^{\hat{\mathbf{w}}} - \operatorname{ATT}
        & = \sum_{i : T_i = 1} \frac{Y_i}{n_1} - \sum_{i : T_i = 1} \frac{\widehat{w}_i Y_i}{n_1} - SATT + (SATT-ATT) \nonumber \\
        & = \sum_{i : T_i = 1} \frac{Y_i}{n_1} - \sum_{i : T_i = 1} \frac{\widehat{w}_i Y_i}{n_1} - \frac{1}{n_1} \sum_{i : T_i = 1} \left( m_1(\mathbf{X}_i) - m_0(\mathbf{X}_i)\right) + (SATT-ATT) \nonumber \\
        & = \frac{1}{n_1} \sum_{i : T_i = 1} m_0 (\mathbf{X}_i) - \sum_{i : T_i = 0} \widehat{w}_{i} m_0 (\mathbf{X}_i) \label{decomp1} \\
        & \quad + \sum_{i : T_i = 1} \frac{Y_i - m_1(\mathbf{X}_i)}{n_1} - \sum_{i : T_i = 0} \widehat{w}_{i} (Y_i - m_0 (\mathbf{X}_i)) \label{decomp2} \\
        & \quad + (\operatorname{SATT} - \operatorname{ATT}). \label{decomp3}
    \end{align}
    
    By the definition of $\hat{\mathbf{w}}$ and Lemma \ref{lemma_causal_true_conv},
    \begin{align}
        d_{\mathcal{F}_{\textup{ReLU}}} \left( \mathbb{P}_{0,n}^{\hat{\mathbf{w}}}, \mathbb{P}_{1,n} \right) \leq d_{\mathcal{F}_{\textup{ReLU}}} \left( \mathbb{P}_{0,n}^{\tilde{\mathbf{w}}}, \mathbb{P}_{1,n} \right) = \mathcal{O}_p \left( n^{-1/2} \right).
    \end{align}    
    By Theorem \ref{thm:holder}, from the assumption $m_0 \in \mathcal{H}_d^{\beta}$, we have (i) if $\beta \geq \frac{d+3}{2}$, 
    $$
    (\ref{decomp1}) \leq d_{\mathcal{H}_d^{\beta}} \left( \mathbb{P}_{0,n}^{\hat{\mathbf{w}}}, \mathbb{P}_{1,n} \right) \lesssim d_{\mathcal{F}_{\textup{ReLU}}} \left( \mathbb{P}_{0,n}^{\hat{\mathbf{w}}}, \mathbb{P}_{1,n} \right) = \mathcal{O}_p \left( n^{-1/2} \right),
    $$ 
    and (ii) if $\beta < \frac{d+3}{2}$, 
    $$
    (\ref{decomp1}) \leq d_{\mathcal{H}_d^{\beta}} \left( \mathbb{P}_{0,n}^{\hat{\mathbf{w}}}, \mathbb{P}_{1,n} \right) \lesssim d_{\mathcal{F}_{\textup{ReLU}}} \left( \mathbb{P}_{0,n}^{\hat{\mathbf{w}}}, \mathbb{P}_{1,n} \right)^{\frac{2\beta}{d+3}} = \mathcal{O}_p \left( n^{-\frac{\beta}{d+3}} \right).
    $$
    
    Since $\{\widehat{w}_{i}\}_{i=1}^n$ is a measurable function of $\{(\mathbf{X}_i, T_i)\}_{i=1}^n$, we have
    \begin{align*}
    \mathbb{E} \left(\sum_{i : T_i = 1} \frac{Y_i - m_1(\mathbf{X}_i)}{n_1} - \sum_{i : T_i = 0} \widehat{w}_{i} (Y_i - m_0(\mathbf{X}_i)) \Bigg| \{(\mathbf{X}_i, T_i)\}_{i=1}^n  \right)=0.
    \end{align*}
    and hence $$\mathbb{E}\left[ (\ref{decomp2})\right] = 0.$$
    Also, 
    \begin{align*}
        \operatorname{Var}\left[ (\ref{decomp2})\right] 
        & = \mathbb{E} \left( \operatorname{Var} \left(\sum_{i : T_i = 1} \frac{Y_i - m_1(\mathbf{X}_i)}{n_1} 
        -  \sum_{i : T_i = 0} \widehat{w}_{i} (Y_i - m_0(\mathbf{X}_i)) \Big| \{(\mathbf{X}_i, T_i)\}_{i=1}^n  \right) \right)\\
        +&  \operatorname{Var} \left( \mathbb{E} \left(\sum_{i : T_i = 1} \frac{Y_i - m_1(\mathbf{X}_i)}{n_1} 
        -  \sum_{i : T_i = 0} \widehat{w}_{i} (Y_i - m_0(\mathbf{X}_i)) \Big| \{(\mathbf{X}_i, T_i)\}_{i=1}^n \right) \right)\\
        = & \mathbb{E} \left( \sum_{i : T_i = 1} \operatorname{Var} \left( \frac{Y_i(1)}{n_1} \Big| \{(\mathbf{X}_i, T_i)\}_{i=1}^n \right)
          + \sum_{i : T_i = 0} \operatorname{Var} \left( \widehat{w}_{i} Y_i(0) \big| \{(\mathbf{X}_i, T_i)\}_{i=1}^n \right) \right)\\
        = & \mathbb{E} \left( \sum_{i : T_i = 1} \frac{\operatorname{Var} \left( Y_i(1) | \{(\mathbf{X}_i, T_i)\}_{i=1}^n \right)}{n_1^2} 
          + \sum_{i : T_i = 0} \widehat{w}_{i}^2 \operatorname{Var} \left( Y_i(0) \big| \{(\mathbf{X}_i, T_i)\}_{i=1}^n \right) \right)\\
        = & O(n^{-1}),
    \end{align*}
    where we use the fact that 
    $\max_{i\in[n]} \widehat{w}_{i} < (n_0 \eta^2)^{-1}.$ 
    Hence, by the Chebyshev's inequality, we get 
    $$(\ref{decomp2}) = O_p(n^{-1/2}).$$
    Since
    $$(\ref{decomp3}) = O_p(n^{-1/2})$$
    holds by the Central Limit Theorem, we obtain the assertion of Theorem \ref{thm:causal}.
\end{proof}


\subsection{Proof of Theorem \ref{thm:relu_frl}} \label{pf:relu_frl}

\begin{proof}
For every $g \in \mathcal{H}_d^\beta$,
\begin{align*}
    \Delta \textup{DP}_{\phi}(g \circ h)
    & = \big| \mathbb{E}_{0} \left( \phi ( g \circ h(\mathbf{X}, 0)) \right) - \mathbb{E}_{1} \left(\phi ( g \circ h (\mathbf{X}, 1)) \right)\big|\\
    & \leq C \cdot \big| \mathbb{E}_{0} \left( g( h(\mathbf{X}, 0)) \right) - \mathbb{E}_{1} \left( g(h( (\mathbf{X}, 1)) \right)\big| &&\tag{by def. of Lipchitz function}\\
    & \leq C \cdot \sup_{g \in \mathcal{H}_d^\beta} \left| \int g(\mathbf{z}) (d\mathcal{P}_0^h(\mathbf{z})-d\mathcal{P}_1^h(\mathbf{z})) \right|,
\end{align*}
where $C$ is the Lipschitz constant of $\phi$.
By Theorem \ref{thm:holder}, we have
\begin{align*}
    \sup_{g \in \mathcal{H}_d^\beta} \left| \int g(\mathbf{z}) (d\mathcal{P}_0^h(\mathbf{z})-d\mathcal{P}_1^h(\mathbf{z})) \right| \leq c \cdot d_{\mathcal{F}_{\textup{ReLU}}}(\mathcal{P}_0^h, \mathcal{P}_1^h)^{\frac{2\beta}{d+3}},
\end{align*}
for $\beta < \frac{d+3}{2}$, and
\begin{align*}
    \sup_{g \in \mathcal{H}_d^\beta} \left| \int g(\mathbf{z}) (d\mathcal{P}_0^h(\mathbf{z})-d\mathcal{P}_1^h(\mathbf{z})) \right| \leq c \cdot d_{\mathcal{F}_{\textup{ReLU}}}(\mathcal{P}_0^h, \mathcal{P}_1^h)
\end{align*}
for $\beta > \frac{d+3}{2}$,
where $c$ is some positive constant that only depends on $d$ and $\beta$. 
\end{proof}

\clearpage
\vspace*{-10pt}
\section{Experimental Details for ReLU-CB} \label{sec:ex_settings_relu_cb}

\subsection{Simulation models} \label{sec:relu_cb_simulations}
    For each unit $i=1,\dots,n$, we consider transformations $\mathbf{X}_i = (X_{i1}, X_{i2},X_{i3},X_{i4})^{\top}$ of latent variables $\mathbf{Z}_i = (Z_{i1}, Z_{i2},Z_{i3},Z_{i4})^{\top}$ independently generated from $\mathcal{N}(0,I_4)$, which are given as
    \begin{align*}
    \begin{split}
    X_{i1} &= \exp(Z_{i1}/2), \\
    X_{i2} &= Z_{i1} / (1+\exp(Z_{i1})) + 10, \\
    X_{i3} &= (Z_{i1}Z_{i3}/25 + 0.6)^3, \\
    X_{i4} &= (Z_{i2}+Z_{i4}+20)^2,
    \end{split}
    \end{align*}    
    and outcomes $Y_i$ are generated from the following regression model,
    \begin{align*}
    Y_i = 210 + 27.4 Z_{i1} + 13.7 Z_{i2} +13.7 Z_{i3} + 13.7 Z_{i4} + \epsilon_i,
    \end{align*}
    where $\epsilon_i \sim N(0,1)$.
    Note that $\operatorname{ATT}=0$ since $m_0(\cdot) = m_1(\cdot).$    
    For the propensity score, we generate the binary treatment indicators $T_i \in \{ 0,1 \}$ from 
    \begin{align*}
    \mathbb{P}(T=1|\mathbf{Z}_i) = \tau \cdot (-Z_{i1} + 0.5 Z_{i2} -0.25 Z_{i3} -0.1 Z_{i4}).
    \end{align*}
    We consider two models (Model 1 and Model 2), each selecting $\tau = \{1, 10\}$ respectively.

\subsection{Baselines and implementation details}\label{sec:impdetail_relu_cb}

We use R (ver. 4.0.2) or Python (ver. 3.9) wherever necessary, and use NVIDIA TITAN Xp GPUs. 

\begin{itemize}     
    \item GLM and Boosting: 
    The SIPW estimator is a refined IPW estimator for the ATT using normalized weights such as
    $$
    \sum_{i:T_i=1} \frac{1}{n_1} Y_i - \frac{\sum_{i:T_i=0} \hat{\pi}(\mathbf{X}_i)(1- \hat{\pi}(\mathbf{X}_i))^{-1} Y_i}{\sum_{i:T_i=0} \hat{\pi}(\mathbf{X}_i)(1- \hat{\pi}(\mathbf{X}_i))^{-1}},
    $$
    where $\hat{\pi}(\cdot)$ is an estimated propensity score. That is, the weights for the control groups are given by the quantities $\hat{\pi}(\mathbf{X}_i)(1- \hat{\pi}(\mathbf{X}_i))^{-1}/\{ \sum_{i:T_i=0} \hat{\pi}(\mathbf{X}_i)(1- \hat{\pi}(\mathbf{X}_i))^{-1} \}$.
For GLM, we use the linear logistic regression $\hat{\pi}_{\hat{\mathbf{\beta}}}(\mathbf{x}) = \{1 + \exp(-\mathbf{x}^{\top} \hat{\mathbf{\beta}} )\}^{-1}$, where the regression coefficient $\hat{\mathbf{\beta}}$ is estimated by the maximum likelihood estimator using the \texttt{GLM} package in R.
For Boosting, we use the the \texttt{twang} package in R \citep{ridgeway2017toolkit}  with 20,000 iterations and the shrinkage parameter of 0.0005.

    \item CBPS \citep{imai2014covariate}:
    Let $\pi_{\beta}(\mathbf{X}_i)$ be a parametric propensity score model, i.e., $\pi_{\beta}(\mathbf{X}_i) \coloneqq \mathbb{P}(\mathbf{T}_i = 1 | \mathbf{X}_i)$.
    \citet{imai2014covariate} proposes to estimate
    $\mathbf{\beta} \in \mathbb{R}^d$ 
    by balancing the moments of the treated and control groups, 
    \begin{align}
    \frac{1}{n} \sum_{i : T_i = 0} \frac{\pi_{\mathbf{\beta}} (\mathbf{X}_i)}{1-\pi_{\mathbf{\beta}} (\mathbf{X}_i) } \mathbf{\psi} \left(\mathbf{X}_i\right) = \frac{1}{n} \sum_{i : T_i = 1} \mathbf{\psi}\left(\mathbf{X}_i\right),
    \label{eq:CBPS-eq}
    \end{align}
    where $\mathbf{\psi}\left(\cdot\right) : \mathcal{X} \to \mathbb{R}^p$ is a pre-specified transformation.
    That is, they estimate the ATT by
        \begin{equation*}
        \sum_{i : T_i = 1} \frac{1}{n_1} Y_i- \frac{\sum_{i : T_i = 0} \pi_{\hat{\mathbf{\beta}}} (\mathbf{X}_i)\left(1-\pi_{\hat{\mathbf{\beta}}} (\mathbf{X}_i)\right)^{-1} Y_i}{\sum_{i : T_i = 0} \pi_{\hat{\mathbf{\beta}}} (\mathbf{X}_i)\left(1-\pi_{\hat{\mathbf{\beta}}} (\mathbf{X}_i)\right)^{-1}},
    \end{equation*}
    where $\hat{\mathbf{\beta}}$ is the solution of the equation (\ref{eq:CBPS-eq}).
    In this study, we use the linear logistic regression model for $\pi_{\beta}(\mathbf{X}_i),$ set $p=d$ and $\mathbf{\psi}(\mathbf{x})=\mathbf{x},$
    and implement CBPS using \texttt{CBPS} package in R \citep{fong2022package} with the default parameters.
    
    \item EB \citep{hainmueller2012entropy}:
    \citet{hainmueller2012entropy} solves
    \begin{align*}
    \underset{\bm{w}}{\operatorname{minimize}} & \sum_{i : T_i = 0} w_i \log w_i\\
    \text { subject to } & \sum_{i : T_i = 0} w_i \bm{\psi}\left(\bm{X}_i\right)=\frac{1}{n_1} \sum_{i : T_i = 1} \bm{\psi}\left(\bm{X}_i\right), \quad \sum_{i : T_i = 0} w_i=1, \quad
    w_i>0,
    \end{align*}
    where $\bm{w} = (w_1 , \dots, w_n)^{\top}$ and $\bm{\psi}\left(\cdot\right) : \mathcal{X} \to \mathbb{R}^p$ is a pre-specified transformation.
    That is, they estimate the ATT by
    \begin{equation*}
        \sum_{i : T_i = 1} \frac{Y_i}{n_1}-\sum_{i : T_i = 0} w_i Y_i. 
    \end{equation*}
    We set $p=d$ and $\mathbf{\psi}(\mathbf{x})=\mathbf{x},$ and
    implement EB using \texttt{EB} package in R \citep{hainmueller2022package} with the default parameters.
    
    \item IPMs: We find $\hat{\mathbf{w}}$ by applying the gradient descent and gradient ascent algorithms iteratively until convergence.
    To be more specific, suppose that the discriminator class $\mathcal{F}$ is parameterized by $\boldsymbol{\lambda} \in \Lambda.$ 
    That is, $\mathcal{F}=\{f(\cdot:\boldsymbol{\lambda}):\boldsymbol{\lambda} \in \Lambda\}.$
    Then, we estimate $\hat{\mathbf{w}}$ by minimizing $\ell_n(\boldsymbol{\lambda}_{\mathbf{w}}, \mathbf{w})$  with respect to $\mathbf{w},$ where 
    $$
	\ell_n(\boldsymbol{\lambda}, \mathbf{w}) \coloneqq \left\{ \sum_{i:T_i=1}\frac{f(\mathbf{x}_i; \boldsymbol{\lambda})}{n_1} - \sum_{i:T_i=0}  w_i f(\mathbf{x}_i; \boldsymbol{\lambda}) \right\}^2
	$$
and $\boldsymbol{\lambda}_{\mathbf{w}}={\rm argmax}_{\boldsymbol{\lambda}} \ell_n(\boldsymbol{\lambda}, \mathbf{w}).$
To get $\hat{\mathbf{w}},$
we update $\mathbf{w}$ by applying the gradient descent algorithm to $\ell_n(\boldsymbol{\lambda}, \mathbf{w})$ with $\boldsymbol{\lambda}$ being fixed and then update  $\boldsymbol{\lambda}$ by applying the gradient ascent algorithm with $\mathbf{w}$ being fixed.
At each gradient descent update, we repeat $N_{\textup{epo}}^{\textup{adv}}$ gradient ascent updates, and
we repeat this gradient descent-ascent updates $N_{\textup{epo}}$ many times. Let us denote the learning rates in the gradient descent and ascent steps as $\textup{lr}$ and $\textup{lr}^{\textup{adv}}$, respectively.
In addition, we use the ensemble technique explained in Section \ref{sec:com_relu}  with $K=100$ in the gradient ascent step.

    \begin{enumerate}
        \item[(1)] Wasserstein distance \citep{pmlr-v202-kong23d}: 
        We compute the Wasserstein distance using techniques suggested by \citet{arjovsky2017wasserstein} and \citet{gulrajani2017improved}. 
        Specifically, given points $\mathbf{x}_t$ and $\mathbf{x}_c$ from the treated and control groups respectively, we generate 100 samples of $\mathbf{x}(\lambda) \coloneqq \lambda \mathbf{x}_c + (1-\lambda) \mathbf{x}_t$ with $\lambda \sim \textup{Unif}(0,1)$. We then compute the gradient of the discriminator with respect to the interpolated samples and penalize its norm to encourage it to approach 1. We set the regularization parameter to 0.3.
        For the discriminators, we use a 2-layer neural network with 100 hidden nodes at each layer and Leaky ReLU activation.
        We set $N_{\textup{epo}}=1000$ and $N_{\textup{epo}}^{\textup{adv}}=5$ and
        use the Adam optimizer with the learning rates $\textup{lr} = 0.03$ and $\textup{lr}^{\textup{adv}} = 0.3$ in the gradient descent and ascent steps, respectively.
        For additional stability, we clip the weights and biases of $f(\cdot ; \boldsymbol{\lambda})$ to $0.1$ after each gradient ascent step.

        \item[(2)] MMD: 
        For $\mathbf{w} = (w_1, \dots, w_n)^\top \in \mathcal{W}^{+}$, the square of the MMD with a kernel $k: \mathbb{R}^d \times \mathbb{R}^d \to \mathbb{R}$ is given as 
        \begin{align}
        d_{\mathcal{H}_{k}}(\widehat{\mathcal{P}}_{0}^{\mathbf{w}},  \widehat{\mathcal{P}}_{1} )^2 & = 
        \sum_{i : T_i = 0 } \sum_{j : T_j = 0 } w_i w_j k(\mathbf{X}_i , \mathbf{Y}_j) 
        - 2 \sum_{i : T_i = 0 } \sum_{j : T_j = 1 } \frac{w_i}{n_1} k(\mathbf{X}_i , \mathbf{Y}_j) \nonumber \\
        & + \sum_{i : T_i = 1 } \sum_{j : T_j = 1 } \frac{1}{n_1^2}  k(\mathbf{X}_i , \mathbf{Y}_j).
        \end{align}
        
        We consider two kernels in the experiments as follows: (i) RBF kernel \citep{pmlr-v202-kong23d}, (ii) Sobolev kernel \citep{Wong2017sobolev}.
        For (i) the RBF kernel $k_{\sigma}(\mathbf{x},\mathbf{y}) = \exp\left(-\|\mathbf{x} - \mathbf{y}\|_2^2 / \sigma^2\right)$, we use Adam optimizer with the learning rate 0.03 and $N_{\textup{epo}} = 1000$ for gradient descent steps, and set $\sigma = 10$.
        For (ii) the Sobolev kernel
        $$
        k_s(\mathbf{x}, \mathbf{y}) = \prod_{j=1}^d \left\{ 1 + k_1(\mathbf{x}_j)k_1(\mathbf{y}_j) + k_2(\mathbf{x}_j)k_2(\mathbf{y}_j) - k_4(|\mathbf{x}_j - \mathbf{y}_j|) \right\},
        $$
        where $k_1(t) = t - 0.5$, $k_2(t) = \frac{k_1(t)^2 - \frac{1}{12}}{2}$, and $k_4(t) = \frac{k_1(t)^4 - \frac{k_1(t)^2}{2} + \frac{7}{240}}{24}$,
        we use Adam optimizer with $\textup{lr} = 0.03$ and $N_{\textup{epo}} = 1000$ for gradient descent steps.

        \item[(3)] SIPM \citep{pmlr-v202-kong23d}:
        We set $N_{\textup{epo}} = 1000$ and $N_{\textup{epo}}^{\textup{adv}} = 3$ and 
        use the Adam optimizer with $\textup{lr} = 0.1$ for the gradient descent step and the SGD optimizer with $\textup{lr}^{\textup{adv}} = 1.0$ for the gradient ascent step.

        \item[(4)] H\"{o}lder-IPM:
        We use a fully connected neural network with $L=\{1,2\}$ hidden layers and $d$ many nodes at each layer, where $d$ is the input dimension. By clipping the output value and the weights, we control 
        the sup-norm of neural networks being less than or equal to 1 and the maximum absolute values of the weights bounded by $\sqrt{d}$, respectively.
        We set $N_{\textup{epo}} = 1000$ and $N_{\textup{epo}}^{\textup{adv}} = 1$ and 
        use the Adam optimizer with $\textup{lr} = 0.01$ for the gradient descent step and the SGD optimizer with $\textup{lr}^{\textup{adv}} = 0.01$ for the gradient ascent step.

        \item[(5)] ReLU-IPM: We set $N_{\textup{epo}} = 1000$ and $N_{\textup{epo}}^{\textup{adv}} = 1$ and 
        use the Adam optimizer with $\textup{lr} = 0.05$ for the gradient descent step and the SGD optimizer with $\textup{lr}^{\textup{adv}} = 1.0$ for the gradient ascent step.
    \end{enumerate}
\end{itemize}


\subsection{Additional Experiments} \label{sec:add_experiments_relu_cb}
We present the results of performance comparison based on additional real analysis in this section.

\begin{table*}[t] 
\renewcommand{\arraystretch}{1.0}
\caption{
\textbf{\textsc{ACIC 2016} datasets.}
For each pair of dataset and performance measure (i.e. for each row), the best result is highlighted in bold, while the next best result is underlined.
We omit the columns of Boost and EB because its solution often does not converge
(Wass: Wasserstein distance, RBF: MMD with RBF kernel, Sob: MMD with Sobolev kernel, Hölder IPM with L=1 and L=2:  H\"{o}l-1 and H\"{o}l-2, respectively).
}
\label{table_ACIC}
\centering
\small
\scalebox{1}{ 
\begin{tabular}{c|c|cccccccc|c}
\hline
Dataset & Measure & GLM & CBPS & Wass & SIPM & RBF & Sob & H\"{o}l-1 & H\"{o}l-2 & ReLU-CB \\
\hline \hline 
\multirow{2}{*}{1} & Bias & \textbf{0.022} & 0.091 & 0.052 & 0.071 & 0.093 & 0.153 & -0.071 & -0.489 & \underline{0.036} \\
& RMSE & 0.314 & 0.237 & \underline{0.159} & 0.212 & 0.160 & 0.249 & 0.941 & 0.956 & \textbf{0.122}\\
\cline{1-11}
\multirow{2}{*}{2} & Bias & -0.318 & 0.029 & \underline{0.024} & 0.045 & 0.028 & 0.048 & \textbf{-0.020} & -0.139 & 0.028 \\
& RMSE & 2.687 & 0.452 & 0.337 & 0.364 & \textbf{0.256} & 0.314 & 0.832 & 0.631 & \underline{0.269} \\
\cline{1-11}
\multirow{2}{*}{3} & Bias & \textbf{-0.017} & 0.134 & 0.076 & 0.100 & 0.135 & 0.210 & 0.099 & -0.619 & \underline{0.050} \\
& RMSE & 0.588 & 0.287 & \underline{0.164} & 0.237 & 0.188 & 0.288 & 1.068 & 0.936 & \textbf{0.127} \\
\cline{1-11}
\multirow{2}{*}{4} & Bias & 0.356 & 0.350 & 0.260 & 0.333 & 0.208 & 0.308 & 0.225 & \textbf{-0.123} & \underline{0.177} \\
& RMSE & 0.789 & 0.766 & 0.533 & 0.748 & \underline{0.424} & 0.599 & 1.012 & 0.714 & \textbf{0.375} \\
\cline{1-11}
\multirow{2}{*}{5} & Bias & \textbf{-0.076} & 0.128 & 0.085 & 0.137 & 0.115 & 0.163 & 0.094 & -0.183 & \underline{0.077} \\
& RMSE & 2.210 & 0.840 & 0.543 & 0.404 & \underline{0.369} & 0.440 & 1.224 & 0.946 & \textbf{0.299} \\
\hline
\multirow{2}{*}{6} & Bias & 0.432 & 0.467 & 0.372 & 0.403 & 0.406 & 0.449 & \underline{0.277} & \textbf{-0.249} & 0.337 \\
& RMSE & 0.712 & 0.739 & \underline{0.608} & 0.700 & 0.700 & 0.698 & 1.006 & 0.739 & \textbf{0.537}\\
\cline{1-11}
\multirow{2}{*}{7} & Bias & 0.100 & 0.102 & 0.094 & 0.106 & 0.123 & 0.111 & \textbf{0.009} & -0.196 & \underline{0.088} \\
& RMSE & 0.368 & 0.354 & \underline{0.317} & 0.365 & 0.383 & 0.327 & 1.042 & 0.679 & \textbf{0.269} \\
\cline{1-11}
\multirow{2}{*}{8} & Bias & 0.248 & 0.273 & 0.224 & 0.249 & 0.245 & 0.249 & 0.137 & \textbf{-0.188} & \underline{0.200} \\
& RMSE & 0.625 & 0.648 & \underline{0.545} & 0.599 & 0.606 & 0.609 & 1.080 & 0.664 & \textbf{0.489} \\
\cline{1-11}
\multirow{2}{*}{9} & Bias & 0.294 & 0.304 & 0.243 & 0.302 & 0.319 & 0.300 & 0.246 & \textbf{-0.203} & \underline{0.211} \\
& RMSE & 0.534 & 0.525 & \underline{0.432} & 0.505 & 0.516 & 0.490 & 0.866 & 0.636 & \textbf{0.383} \\
\cline{1-11}
\multirow{2}{*}{10} & Bias & 0.420 & 0.421 & 0.341 & 0.390 & 0.410 & 0.392 & 0.405 & \textbf{-0.179} & \underline{0.293} \\
& RMSE & 0.783 & 0.812 & \underline{0.641} & 0.736 & 0.742 & 0.723 & 1.044 & 0.746 & \textbf{0.550} \\
\hline
\end{tabular}
}
\end{table*}

\subsubsection{Semi-synthetic experiments}
The \textsc{ACIC 2016} datasets include covariates, simulated treatment, and simulated response variables for the causal inference challenge held at the 2016 Atlantic Causal Inference Conference \citep{dorie2019automated}. For each of the 77 conditions (simulation settings), treatment and response data were simulated 100 times from real-world data corresponding to 4,802 individuals and 58 covariates. Among 77 simulation settings, we selected the first five and the last five ones (total 10 simulation settings), and analyzed 100 simulated datasets for each simulation setting.

The results in Table \ref{table_ACIC} indicate that ReLU demonstrates superior performance compared to other methods in terms of both performance measures. While GLM yields better results in terms of Bias, Wass and RBF exhibit better performance in terms of RMSE. Interestingly, H\"{o}l-1 and H\"{o}l-2 shows smaller Bias values but larger RMSE values.


\clearpage
\vspace*{-10pt}
\section{Experimental Details for ReLU-FRL} \label{sec:ex_settings_relu_fl}

\subsection{Datasets}\label{sec:datasets-appendix}

\begin{itemize}
    \item \textsc{Adult}:
    The Adult income dataset \citep{Dua:2019} can be downloaded from the UCI repository\footnote{\url{https://archive.ics.uci.edu/ml/datasets/adult}}.
    We follow the pre-processing of the implementation of IBM's AIF360 \citep{aif360-oct-2018} \footnote{\url{https://aif360.readthedocs.io/en/stable/}}.
    
    \item \textsc{Dutch}:
    The Dutch census dataset can be downloaded from the public GitHub of \citet{Le_Quy_2022} \footnote{\url{https://github.com/tailequy/fairness\_dataset/tree/main/experiments/data/dutch.csv}}.
    We follow the pre-processing of \citet{Le_Quy_2022} in its official GitHub\footnote{\url{https://github.com/tailequy/fairness\_dataset/tree/main/experiments/data/}}.
    
    \item \textsc{Crime}:
    The communities and crime dataset can be downloaded from the UCI repository\footnote{\url{https://archive.ics.uci.edu/dataset/183/communities+and+crime}}.
    We follow the pre-processing of \citet{jiang2022generalized} in its official GitHub\footnote{\url{https://github.com/zhimengj0326/GDP/blob/master/tabular/data\_load.py}}.
\end{itemize}

Table \ref{table:datasets} below presents the basic information of the three datasets used in our experiments.

\begin{table}[ht]
    \centering
    \caption{
    The description of the datasets: \textsc{Adult}, \textsc{Dutch}, and \textsc{Crime}.
    }
    \label{table:datasets}
    \vskip 0.1in
    \resizebox{\textwidth}{!}{%
    \begin{tabular}{c|c|c|c}
        \toprule
         Dataset & \textsc{Adult} & \textsc{Dutch} & \textsc{Crime} \\
         \midrule \midrule
         $\mathbf{X}$ & Demographic attributes & Demographic attributes & Demographic attributes \\
         $S$ & Gender & Gender & \% of black population (binarized) \\
         $Y$ & Outcome over $\$50k$ & High-level occupation & \% of crimes (binarized) \\
         $d$ & 101 & 58 & 121 \\
         Train data size & 30,136 & 48,336 & 1,794 \\
         Test data size & 15,086 & 12,084 & 200 \\
         Batch size & 1024 & 1024 & 256 \\
         \bottomrule
    \end{tabular}
    }
\end{table}

\subsection{Implementation Details}\label{sec:impdetail_relu_frl}

\begin{paragraph}{Baseline algorithms}
    For given dataset $\{(\mathbf{x}_i, s_i, y_i)\}_{i=1}^n$,  the training objective is to minimize
    \begin{align*}
        \frac{1}{n} \sum_{i=1}^n \ell (g \circ h(\mathbf{x}_i, s_i), y_i) \quad \textup{ subject to } \quad d(\mathcal{P}_{0, n}^{h}, \mathcal{P}_{1, n}^{h}) \leq \delta,
    \end{align*}
    with respect to $g\in \mathcal{G}$ and $h\in \mathcal{H},$
    where $\ell$ is a given task-specific loss function and $d$ is a given deviance measure between the empirical distributions $\mathcal{P}_{0, n}^{h}$ and $\mathcal{P}_{1, n}^{h}$ of $\{h(\mathbf{x}_i,0) : s_i = 0\}$ and 
    $\{h(\mathbf{x}_i,1) : s_i = 1\}$, respectively. We use the cross-entropy loss for the $\ell$.
    The Lagrangian form of this constrained optimization can be expressed as
    \begin{align} \label{eq:obj_ftn_frl}
        \frac{1}{n} \sum_{i=1}^n \ell (g \circ h(\mathbf{x}_i, s_i), y_i) + \lambda d(\mathcal{P}_{0, n}^{h}, \mathcal{P}_{1, n}^{h})
    \end{align}
    for some $\lambda>0$.
    In our experimental results, we present the  Pareto-front trade-off graphs illustrating the relationship between the fairness levels (\(\Delta \overline{\textup{DP}}, \Delta \textup{DP}, \Delta \textup{SDP}\)) and the binary classification accuracy (\textsc{Acc}) by varying the values of $\lambda$.

    For the deviance measure, along with the ReLU-IPM, we consider SIPM \citep{pmlr-v162-kim22b},
    H\"{o}lder-IPM \citep{Wang2023manifold} approximating H\"{o}lder discriminators by  one hidden layer neural networks, MMD-B-Fair \citep{pmlr-v206-deka23a} and 
    the Jensen–Shannon divergence (LAFTR, \cite{Madras2018LearningAF}).
    For LAFTR, we use a single-layered MLP with hidden size 100 and ReLU activation.

\end{paragraph}

\begin{paragraph}{Training hyperparameters}
    We train all models for $200$ epochs in total with Adam optimizer \citep{kingma2014adam}.
    The learning rate is initialized by $0.001$ and scheduled by multiplying $0.95$ at each epoch.
    For the optimization we update the discriminator for $2$ iterations of the gradient ascent  per a single update of the model parameters (i.e. parameters in the representation) by the gradient descent.
    The learning rate is set to $0.001$ and not scheduled.
    For ReLU-IPM and SIPM, we use the ensemble technique explained in Section \ref{sec:com_relu}  with $K=100$ in the gradient ascent step.
    The reported results are assessed on the test dataset with the final models trained after the $200$ epochs.
    For H\"{o}lder-IPM, we use $100$ nodes at the hidden layer.
\end{paragraph}


\subsection{Results of Additional Experiments} \label{sec:add_experiments_relu_frl}

Figures \ref{fig:1_ReLU_MLP_dp} and \ref{fig:1_ReLU_MLP_sdp} presents the Pareto-front trade-off graphs of $\Delta \textup{DP}$ and $\Delta \textup{SDP}$ for comparing the ReLU-FRL with the single layered neural network head with the ReLU activation function to other baselines. The results are similar to those in Figure \ref{fig:1_ReLU_MLP_meandp}.

Figure \ref{fig:linear_meandp} to Figure  \ref{fig:1_Sigmoid_MLP_sdp}
draw the Pareto-front trade-off graphs for the liner head and the single layered neural network head with the sigmoid activation function. The results indicate that the ReLU-FRL works well regardless of the choice of the prediction head.

\begin{figure}[p]
    \centering
    \includegraphics[width=0.81\textwidth]{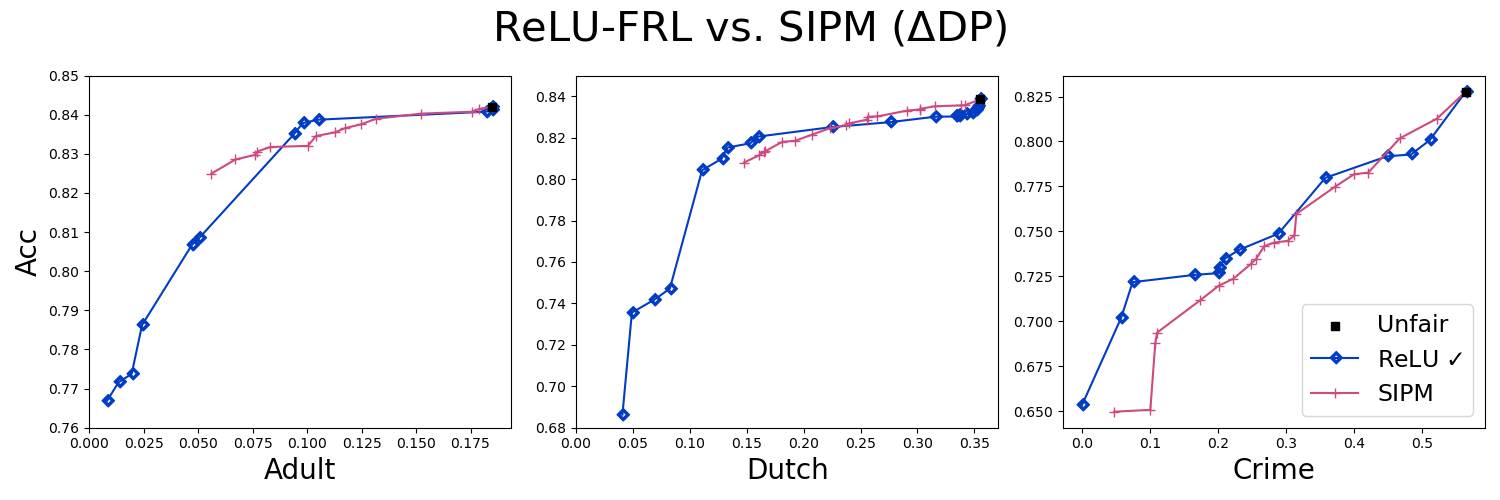}
    \includegraphics[width=0.81\textwidth]{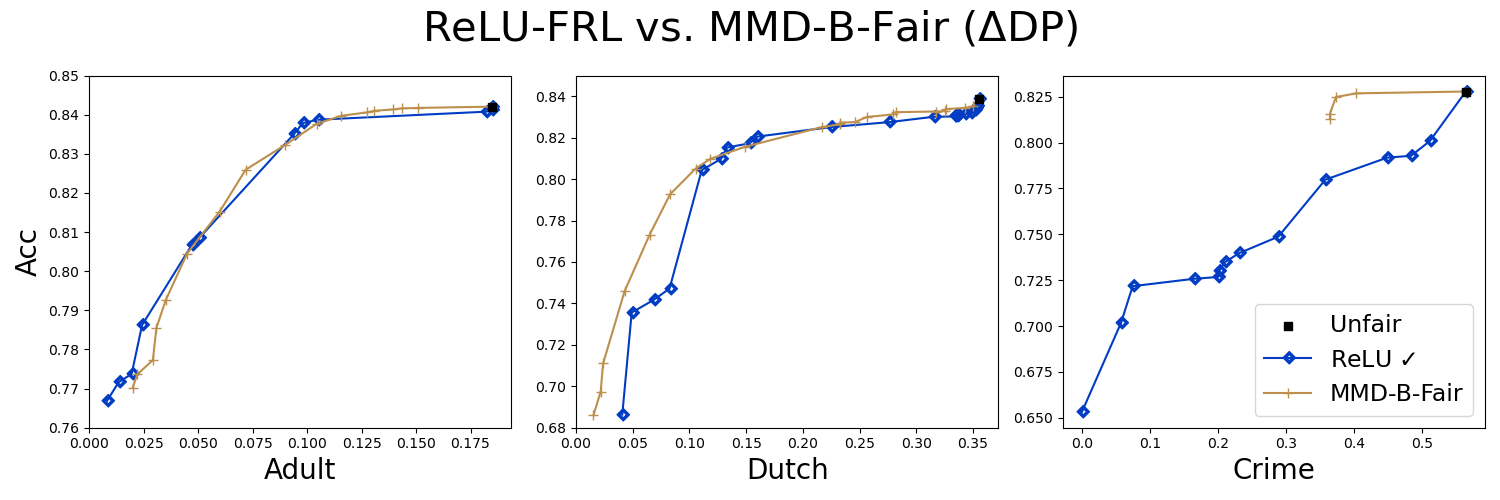}
    \includegraphics[width=0.81\textwidth]{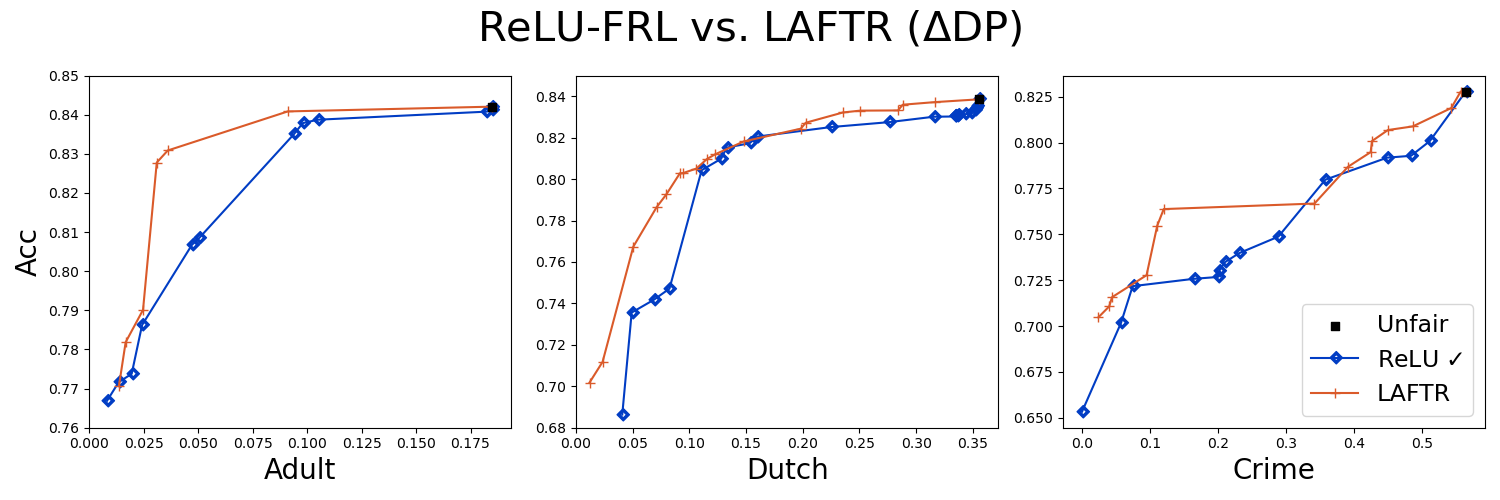}
    \includegraphics[width=0.81\textwidth]{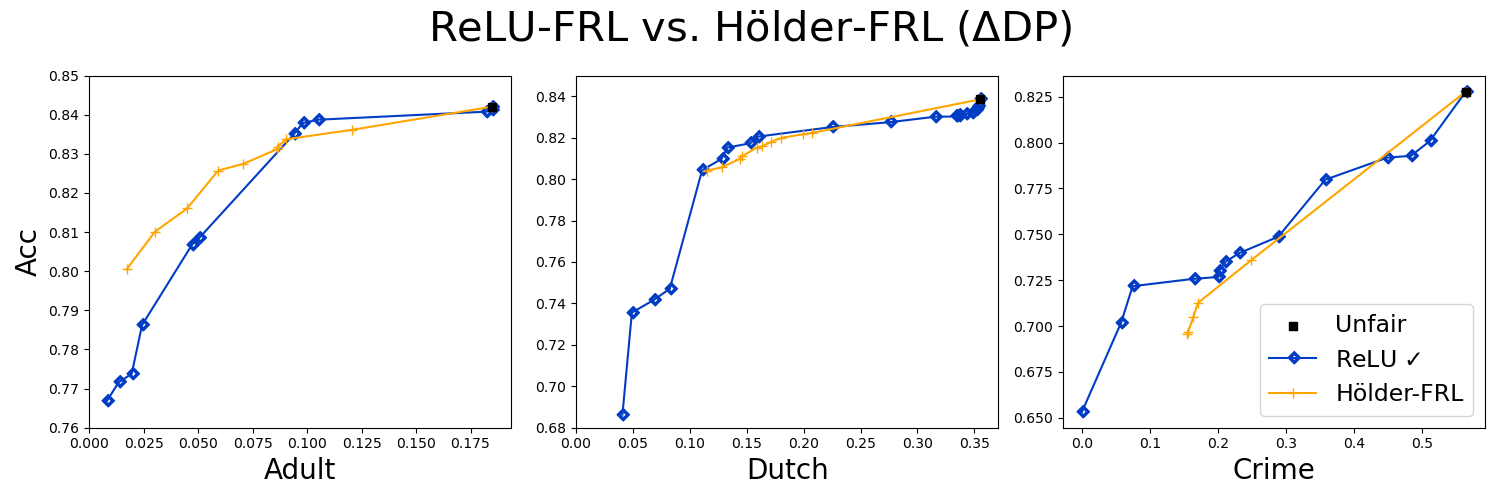}
    \caption{
    \textbf{Single-layered NN (ReLU activation) prediction head}: Pareto-front lines of fairness level $\Delta \textup{DP}$ and \texttt{Acc}.
    (Left) \textsc{Adult}, (Center) \textsc{Dutch}, (Right) \textsc{Crime}.
    }
    \label{fig:1_ReLU_MLP_dp}
\end{figure}

\begin{figure}[p]
    \centering
    \includegraphics[width=0.81\textwidth]{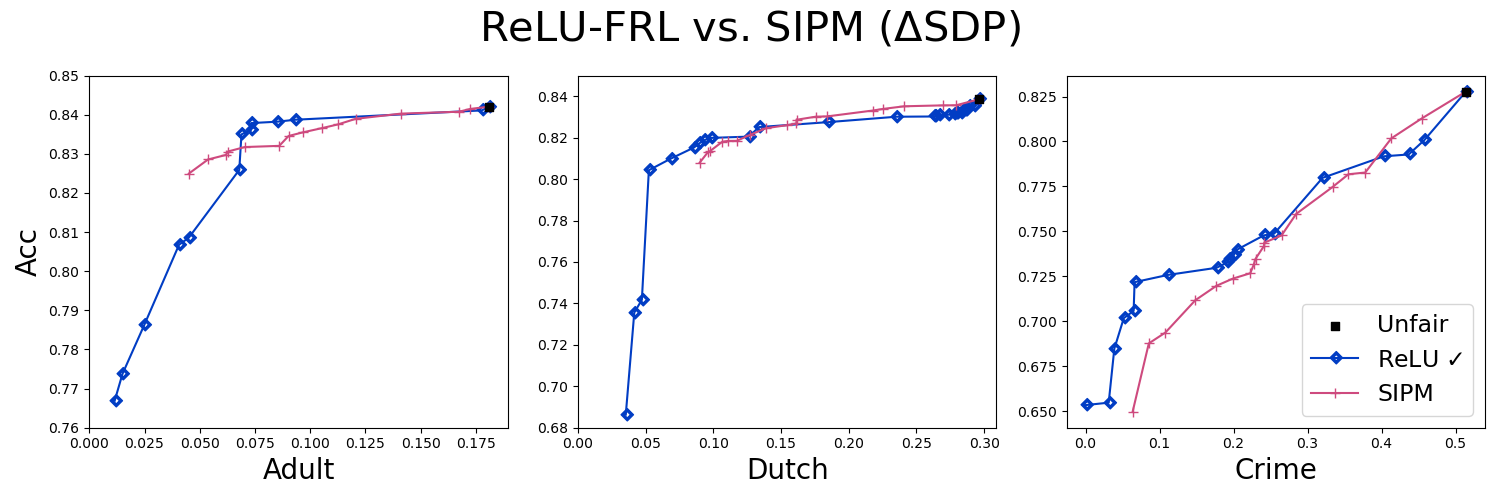}
    \includegraphics[width=0.81\textwidth]{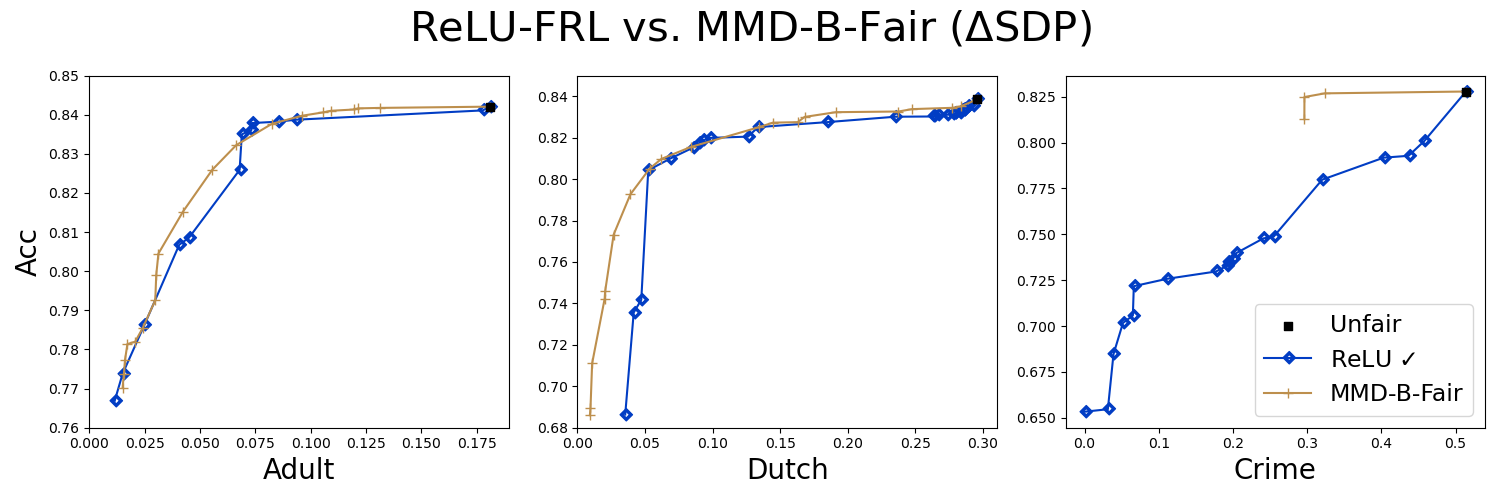}
    \includegraphics[width=0.81\textwidth]{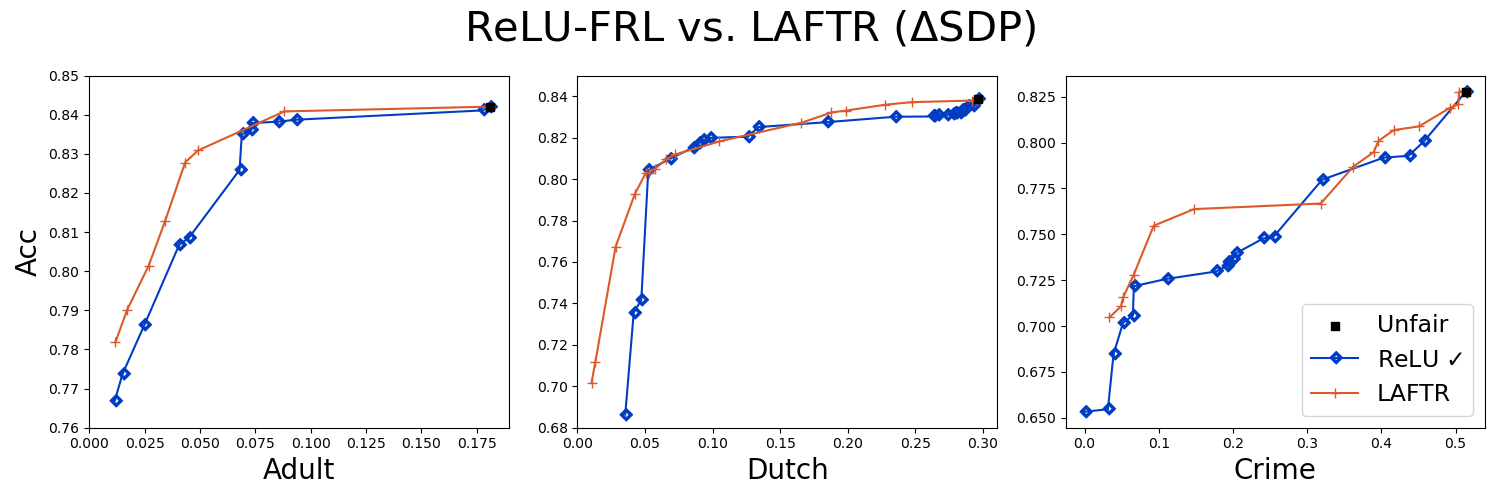}
    \includegraphics[width=0.81\textwidth]{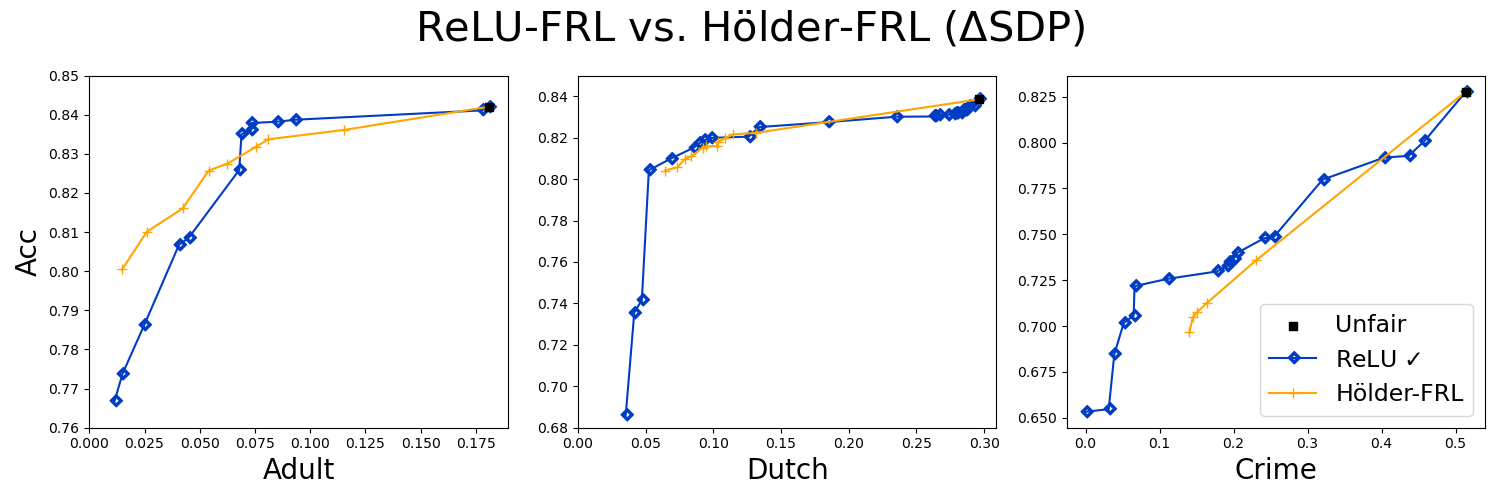}
    \caption{
    \textbf{Single-layered NN (ReLU activation) prediction head}: Pareto-front lines of fairness level $\Delta \textup{SDP}$ and \texttt{Acc}.
    (Left) \textsc{Adult}, (Center) \textsc{Dutch}, (Right) \textsc{Crime}.
    }
    \label{fig:1_ReLU_MLP_sdp}
\end{figure}


\begin{figure}[p]
    \centering
    \includegraphics[width=0.81\textwidth]{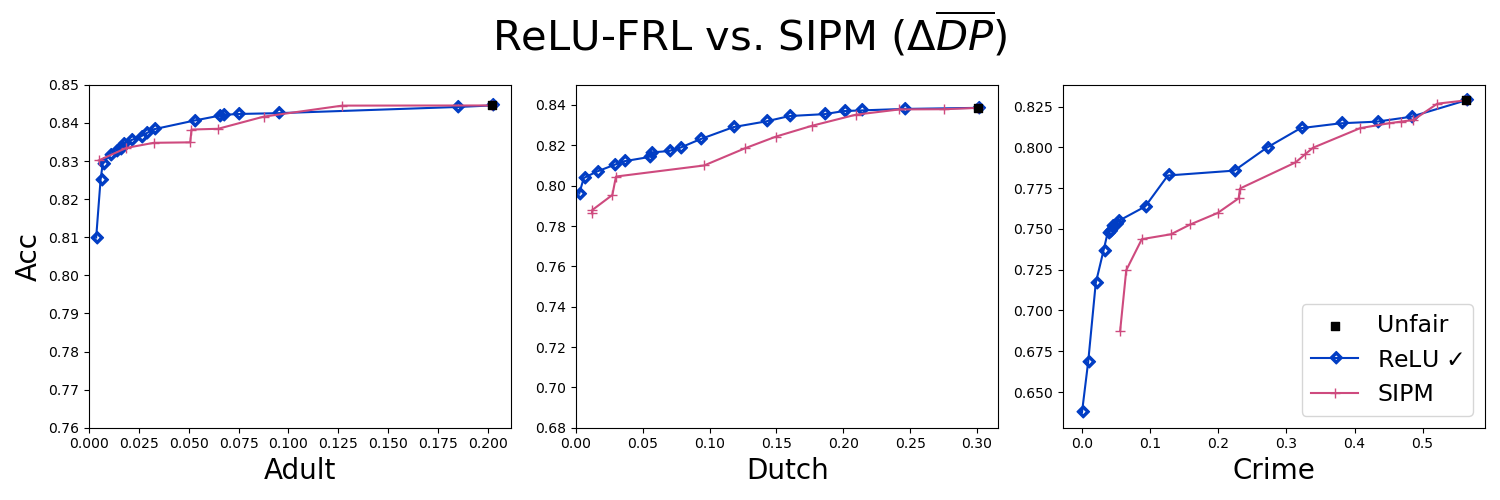}
    \includegraphics[width=0.81\textwidth]{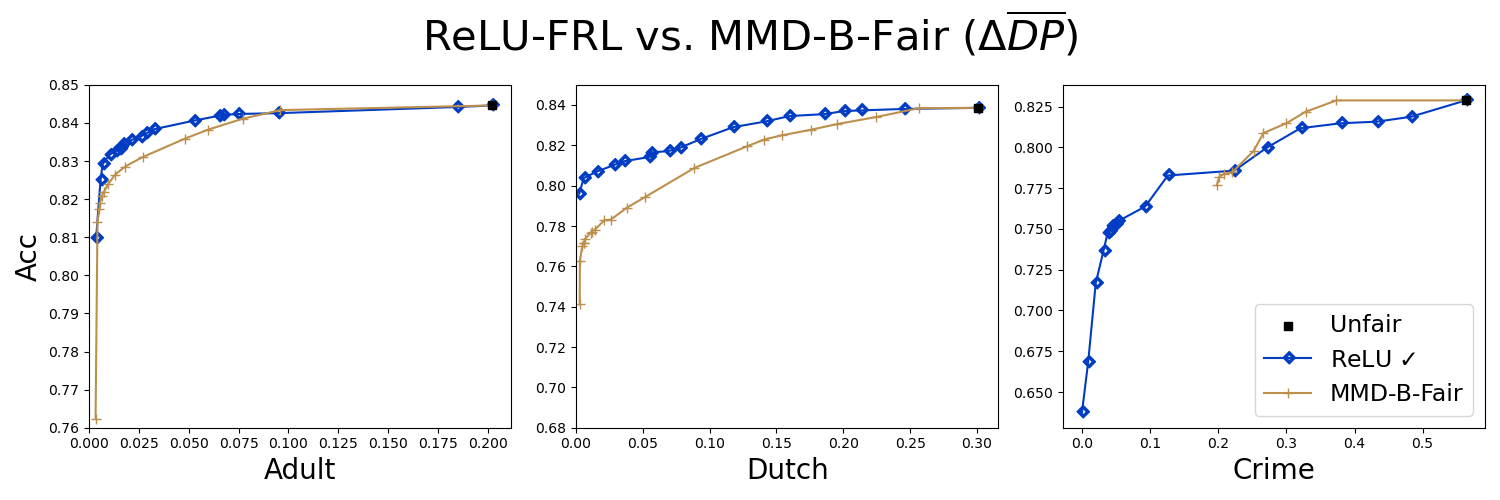}
    \includegraphics[width=0.81\textwidth]{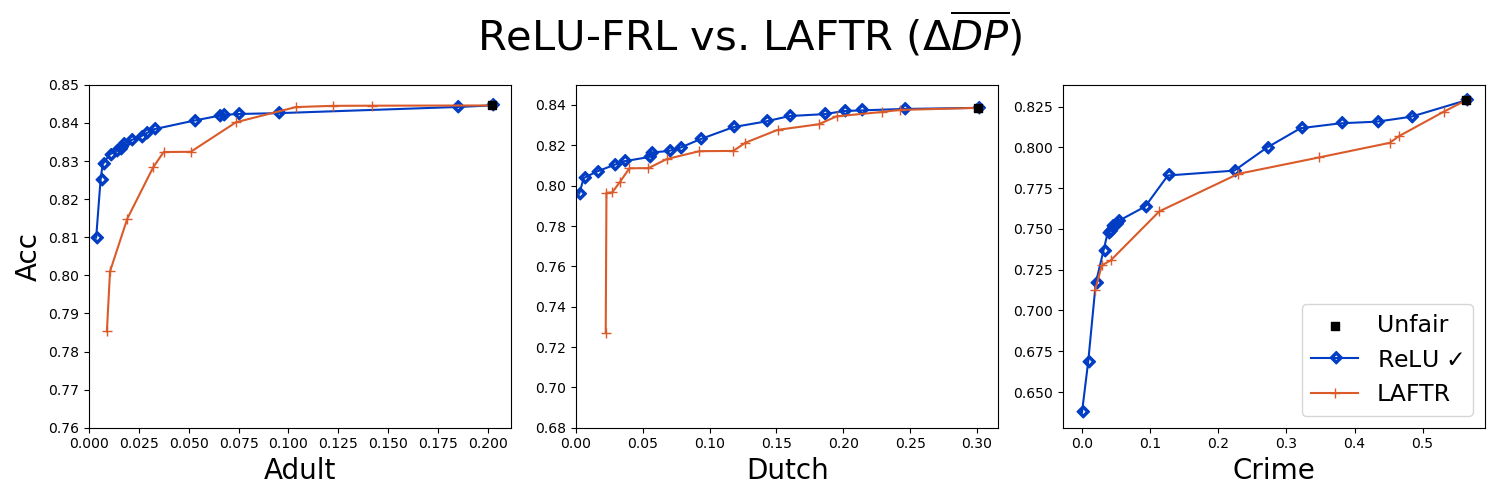}
    \includegraphics[width=0.81\textwidth]{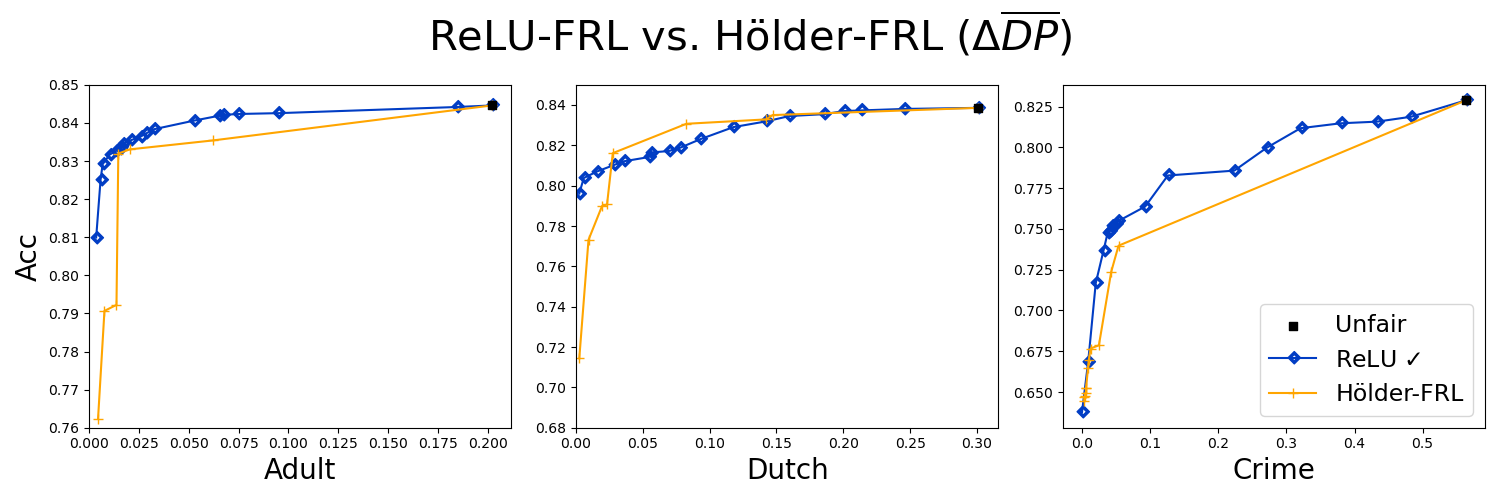}
    \caption{
    \textbf{Linear prediction head}: Pareto-front lines of fairness level $\Delta \overline{\textup{DP}}$ and \texttt{Acc}.
    (Left) \textsc{Adult}, (Center) \textsc{Dutch}, (Right) \textsc{Crime}.
    }
    \label{fig:linear_meandp}
\end{figure}

\begin{figure}[p]
    \centering
    \includegraphics[width=0.81\textwidth]{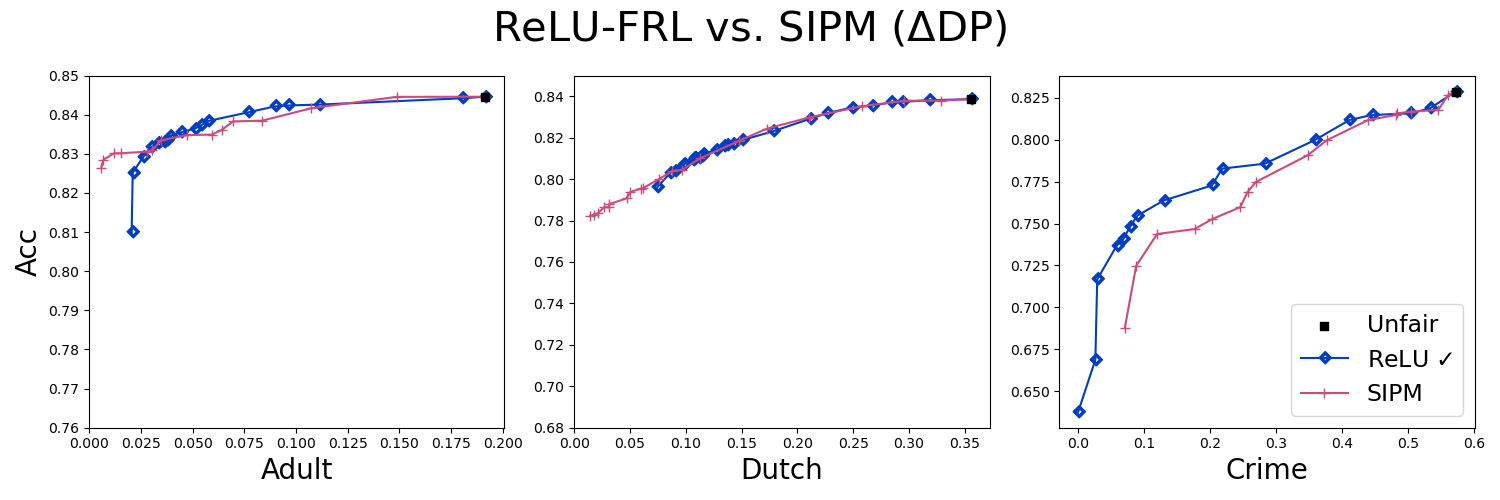}
    \includegraphics[width=0.81\textwidth]{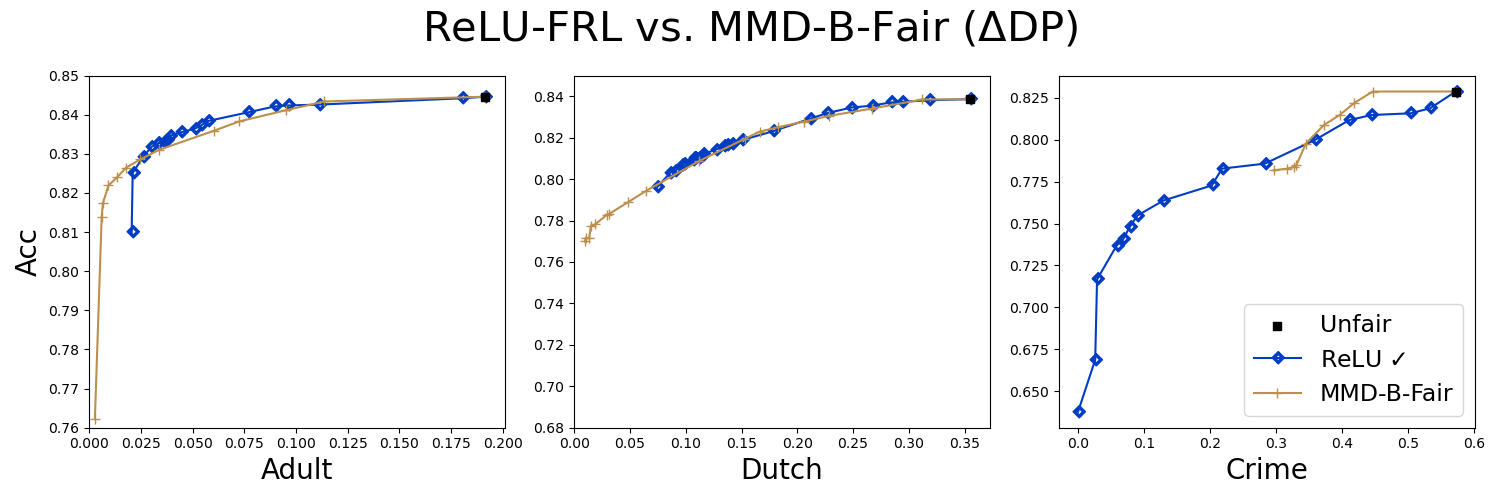}
    \includegraphics[width=0.81\textwidth]{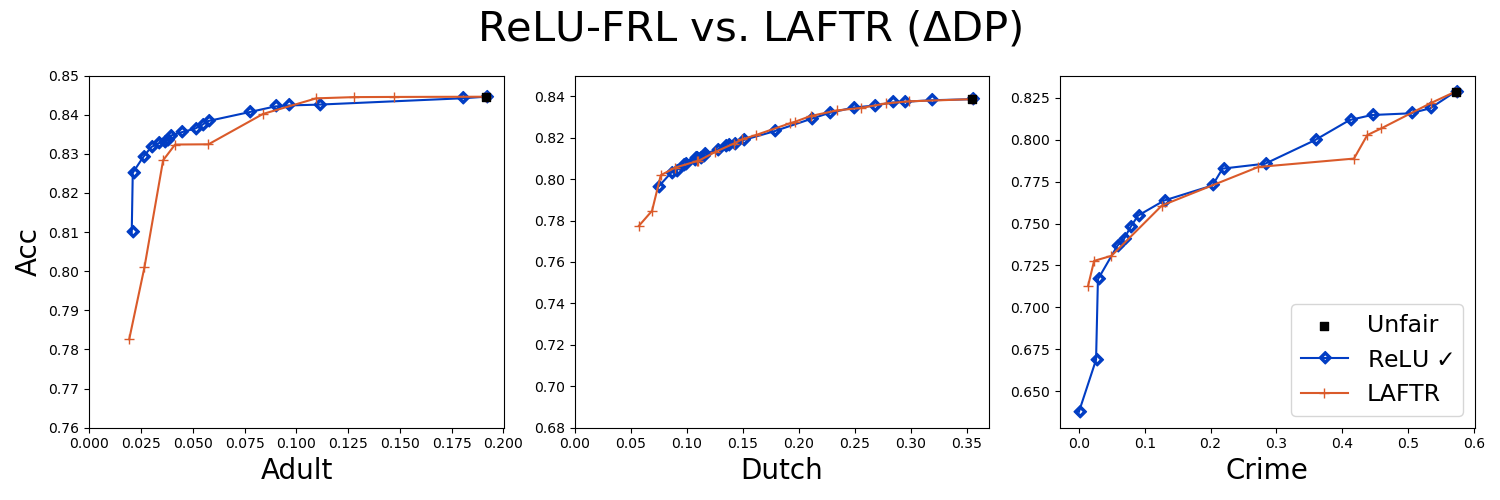}
    \includegraphics[width=0.81\textwidth]{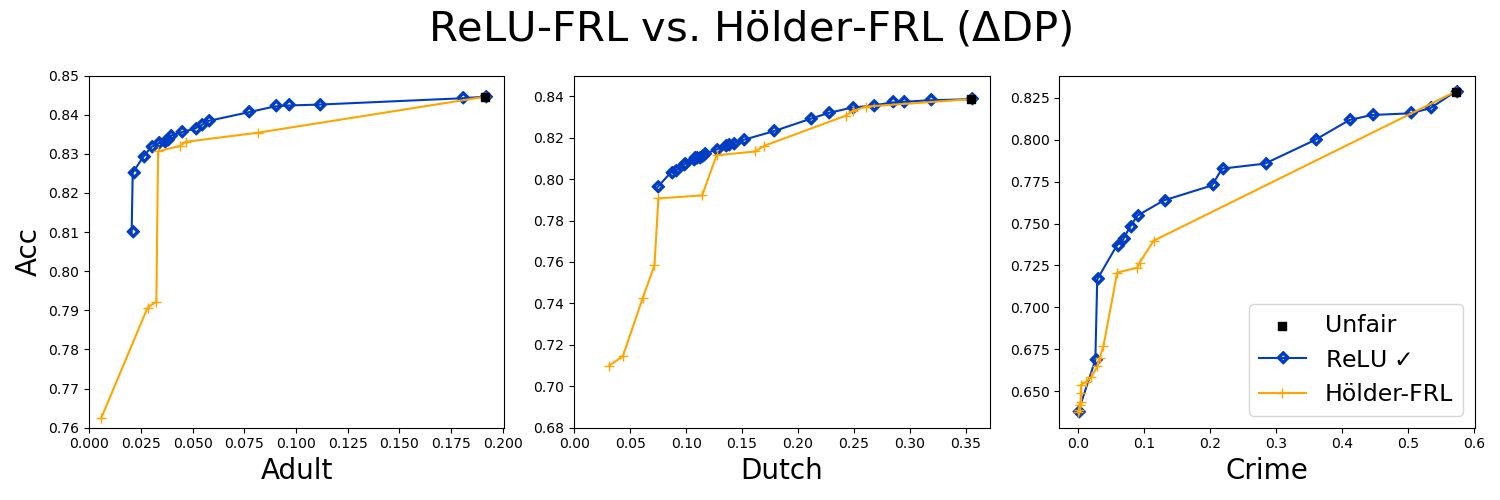}
    \caption{
    \textbf{Linear prediction head}: Pareto-front lines of fairness level $\Delta \textup{DP}$ and \texttt{Acc}.
    (Left) \textsc{Adult}, (Center) \textsc{Dutch}, (Right) \textsc{Crime}.
    }
    \label{fig:linear_dp}
\end{figure}

\begin{figure}[p]
    \centering
    \includegraphics[width=0.81\textwidth]{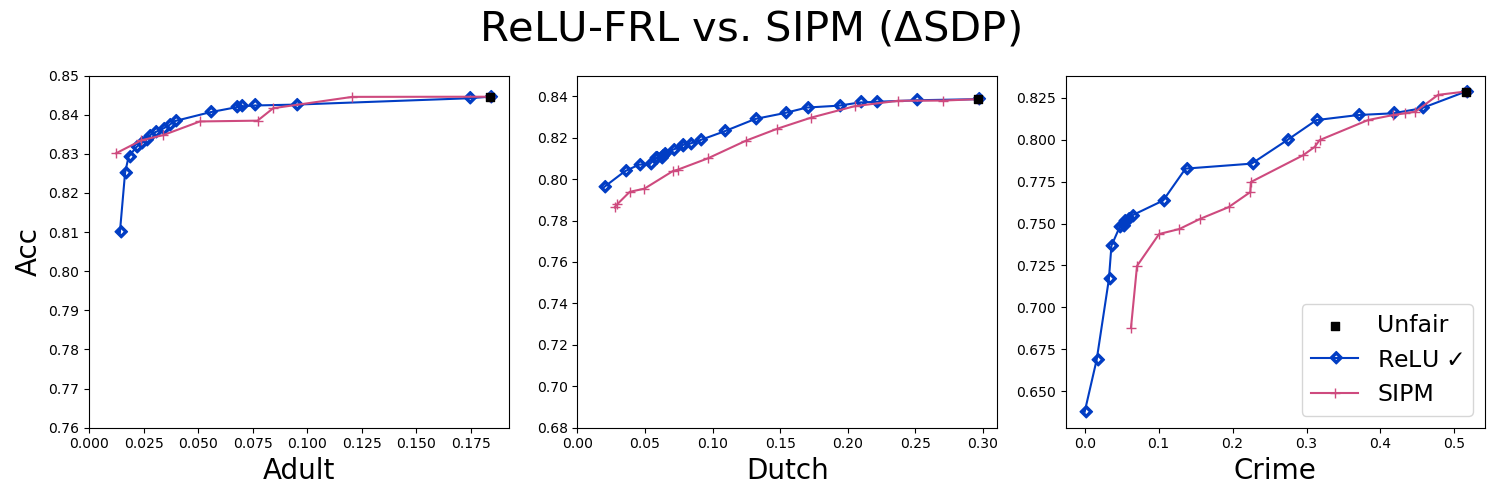}
    \includegraphics[width=0.81\textwidth]{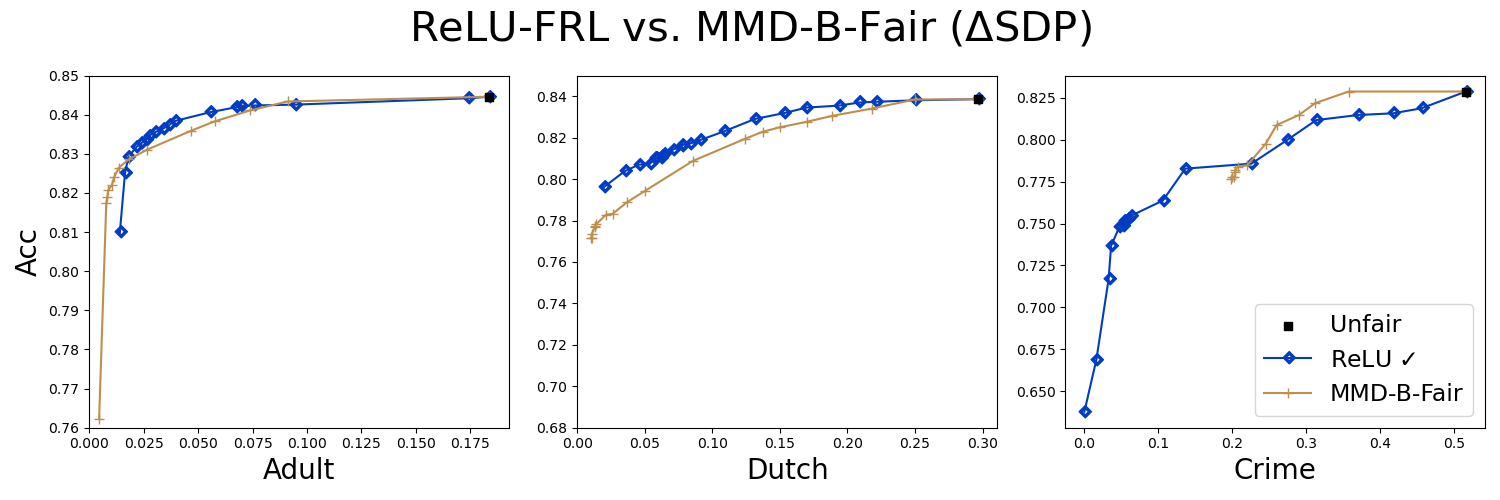}
    \includegraphics[width=0.81\textwidth]{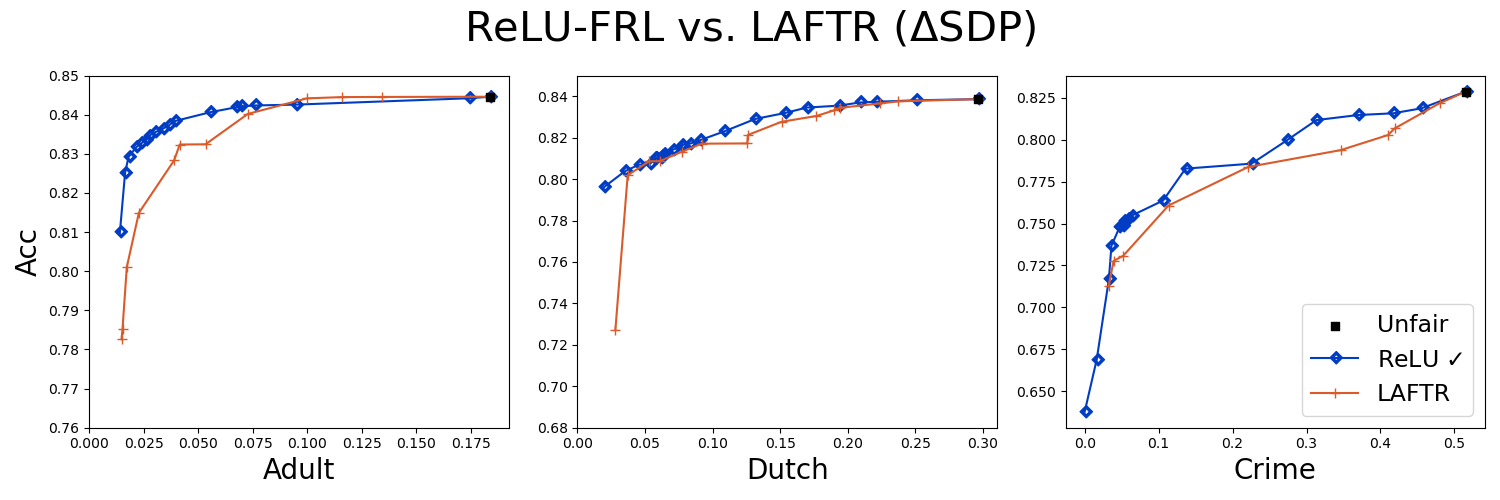}
    \includegraphics[width=0.81\textwidth]{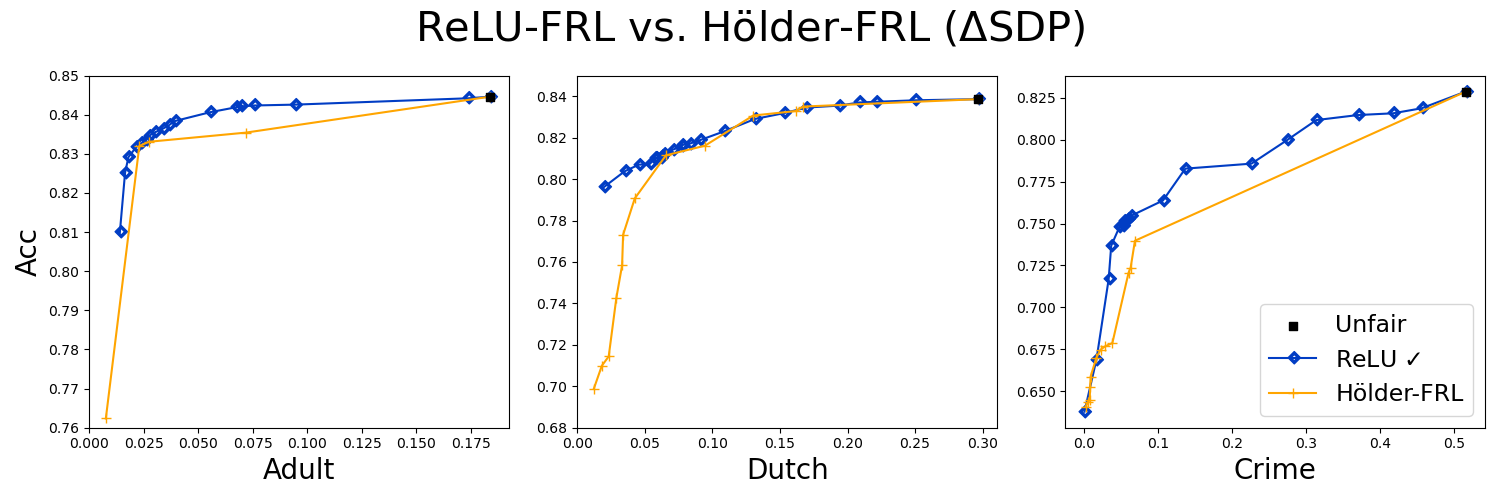}
    \caption{
    \textbf{Linear prediction head}: Pareto-front lines of fairness level $\Delta \textup{SDP}$ and \texttt{Acc}.
    (Left) \textsc{Adult}, (Center) \textsc{Dutch}, (Right) \textsc{Crime}.
    }
    \label{fig:linear_meandp}
\end{figure}


\begin{figure}[p]
    \centering
    \includegraphics[width=0.81\textwidth]{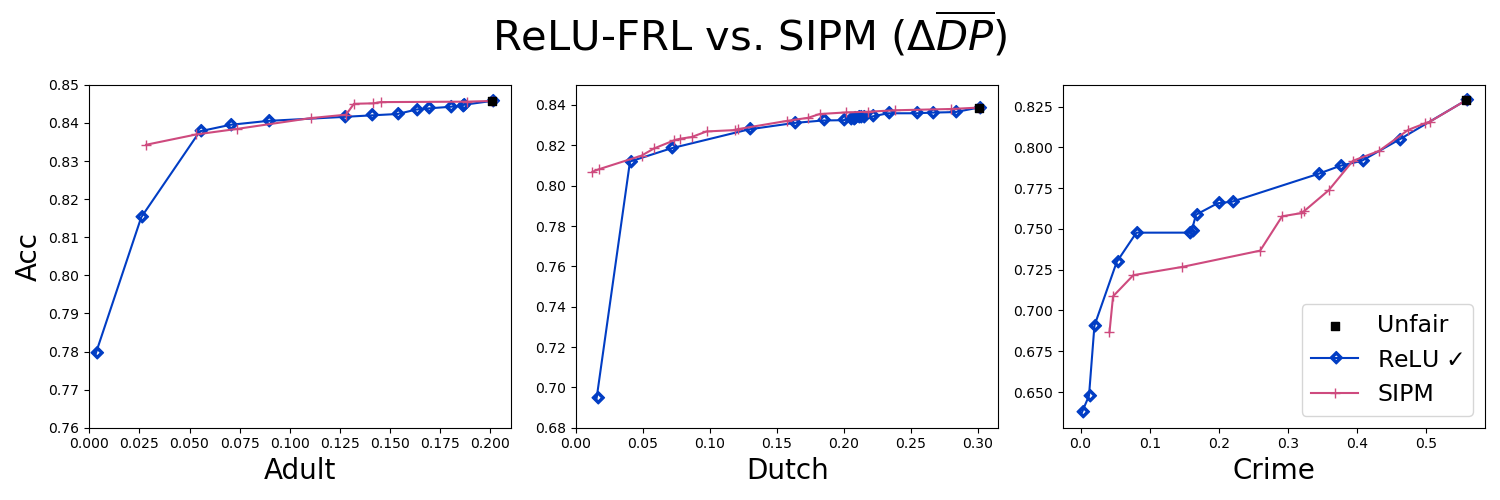}
    \includegraphics[width=0.81\textwidth]{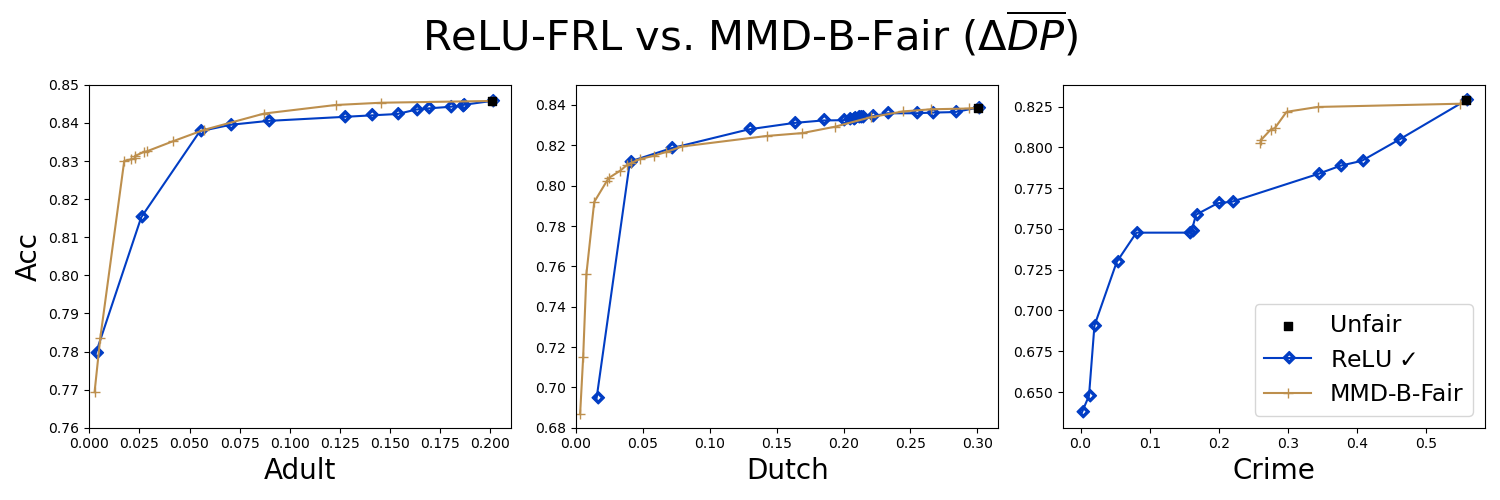}
    \includegraphics[width=0.81\textwidth]{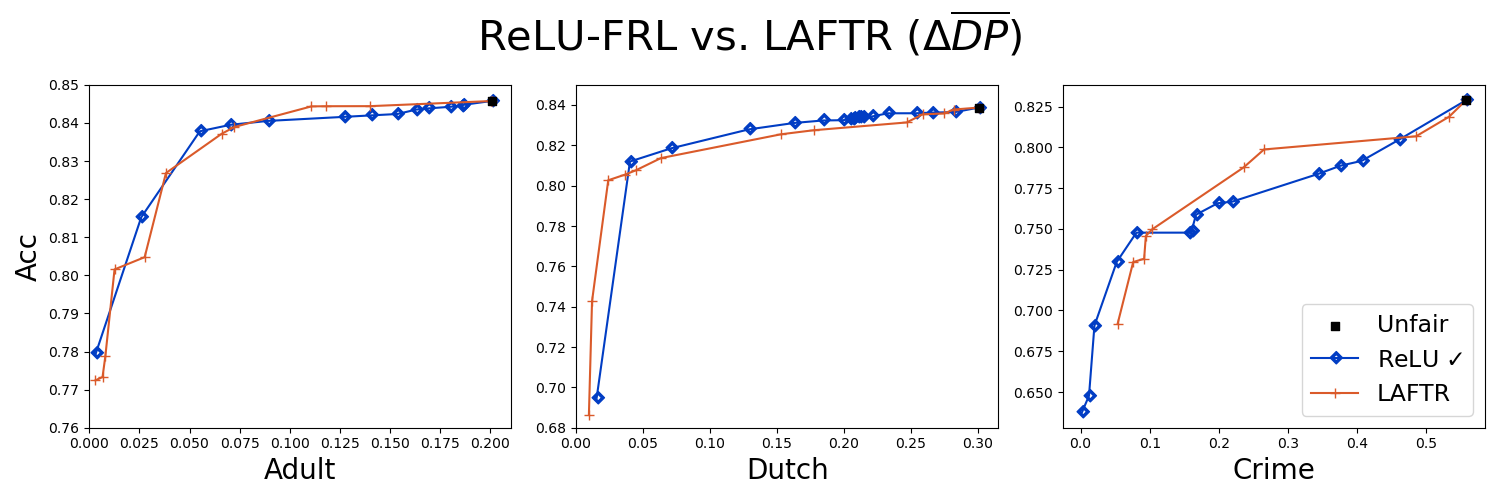}
    \includegraphics[width=0.81\textwidth]{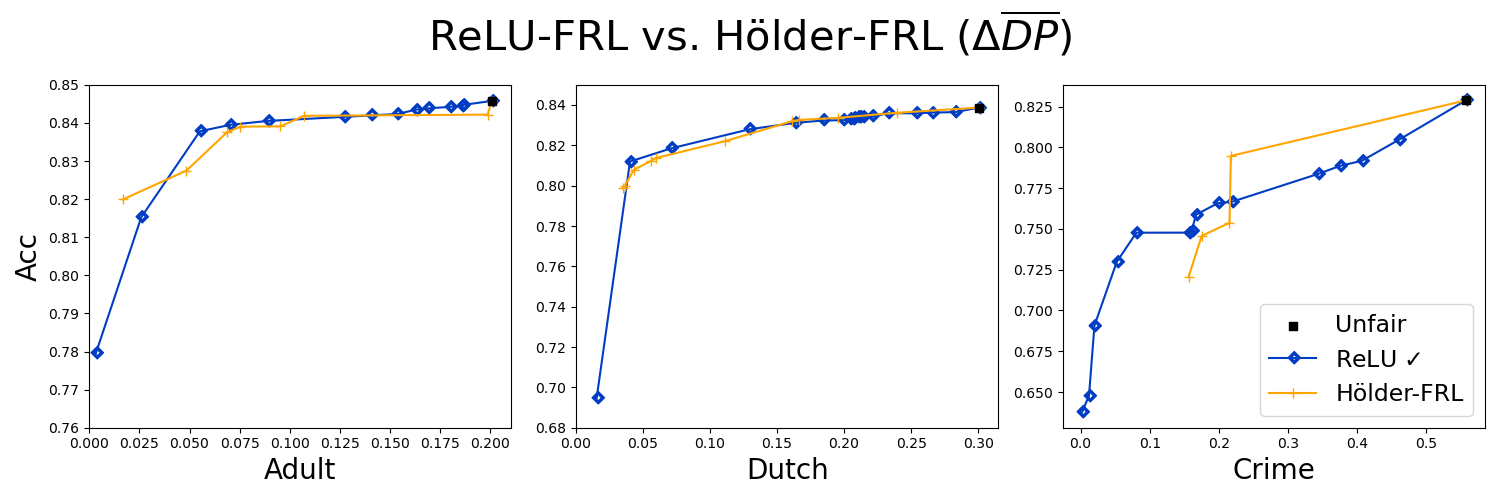}
    \caption{
    \textbf{Single-layered NN (sigmoid activation) prediction head}: Pareto-front lines of fairness level $\Delta \overline{\textup{DP}}$ and \texttt{Acc}.
    (Left) \textsc{Adult}, (Center) \textsc{Dutch}, (Right) \textsc{Crime}.
    }
    \label{fig:1_Sigmoid_MLP_meandp}
\end{figure}

\begin{figure}[p]
    \centering
    \includegraphics[width=0.81\textwidth]{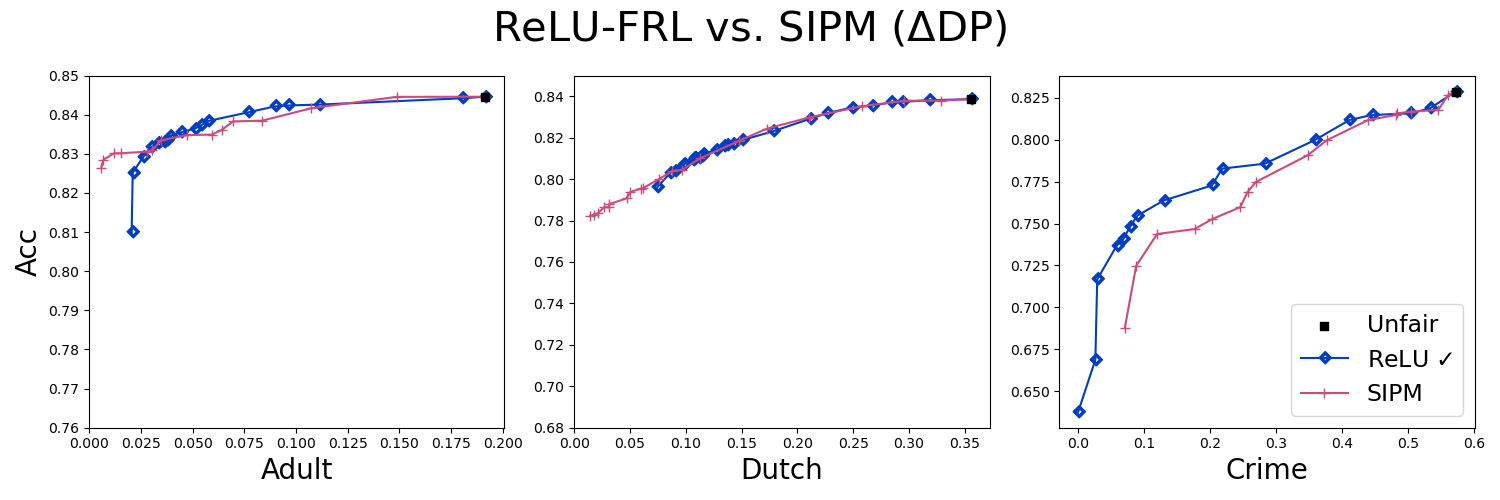}
    \includegraphics[width=0.81\textwidth]{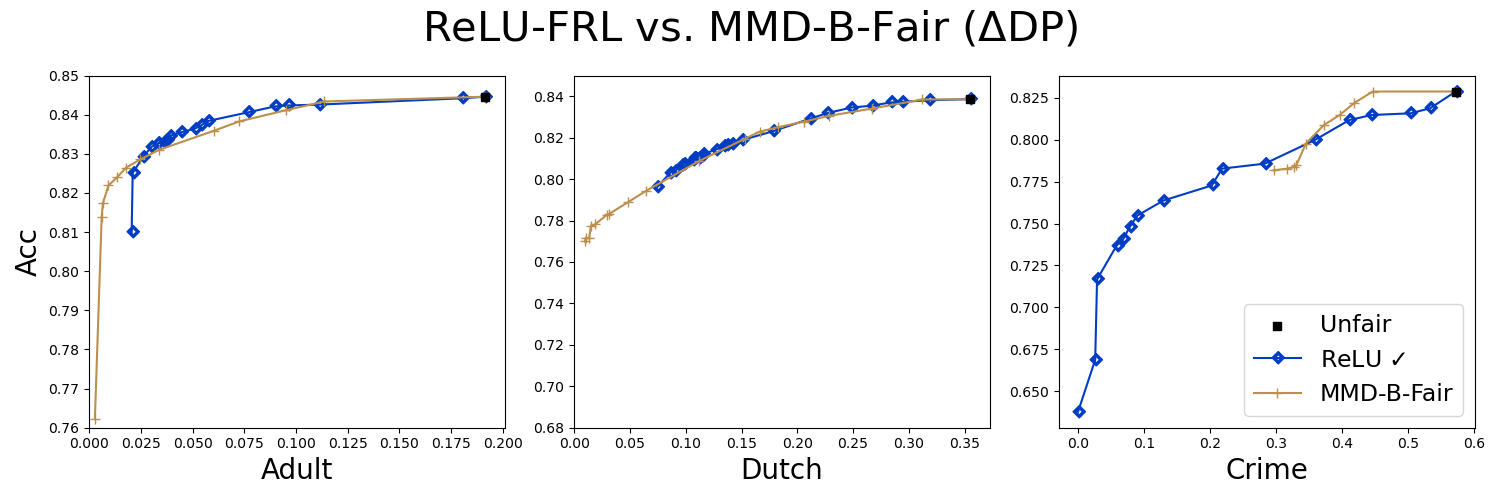}
    \includegraphics[width=0.81\textwidth]{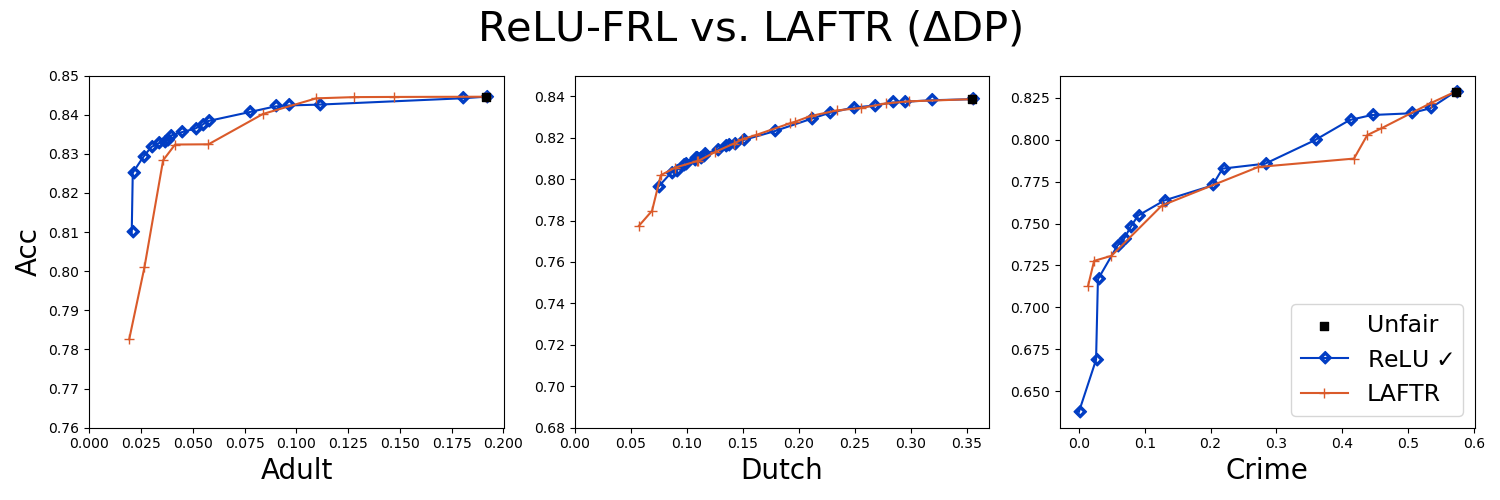}
    \includegraphics[width=0.81\textwidth]{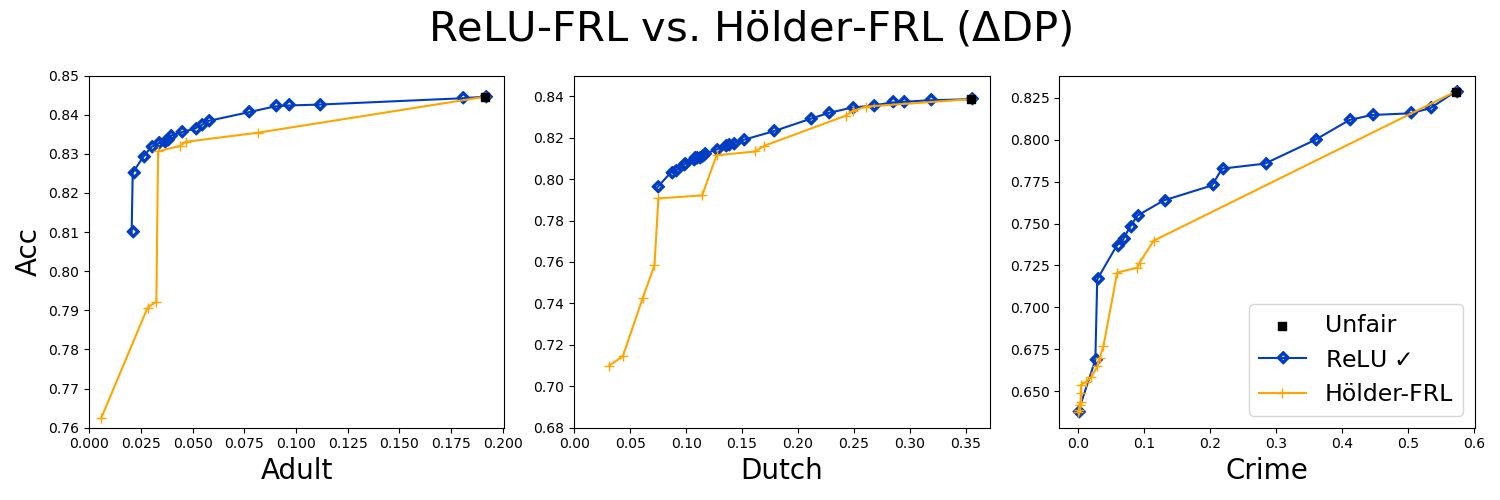}
    \caption{
    \textbf{Single-layered NN (sigmoid activation) prediction head}: Pareto-front lines of fairness level $\Delta \textup{DP}$ and \texttt{Acc}.
    (Left) \textsc{Adult}, (Center) \textsc{Dutch}, (Right) \textsc{Crime}.
    }
    \label{fig:1_Sigmoid_MLP_dp}
\end{figure}

\begin{figure}[p]
    \centering
    \includegraphics[width=0.81\textwidth]{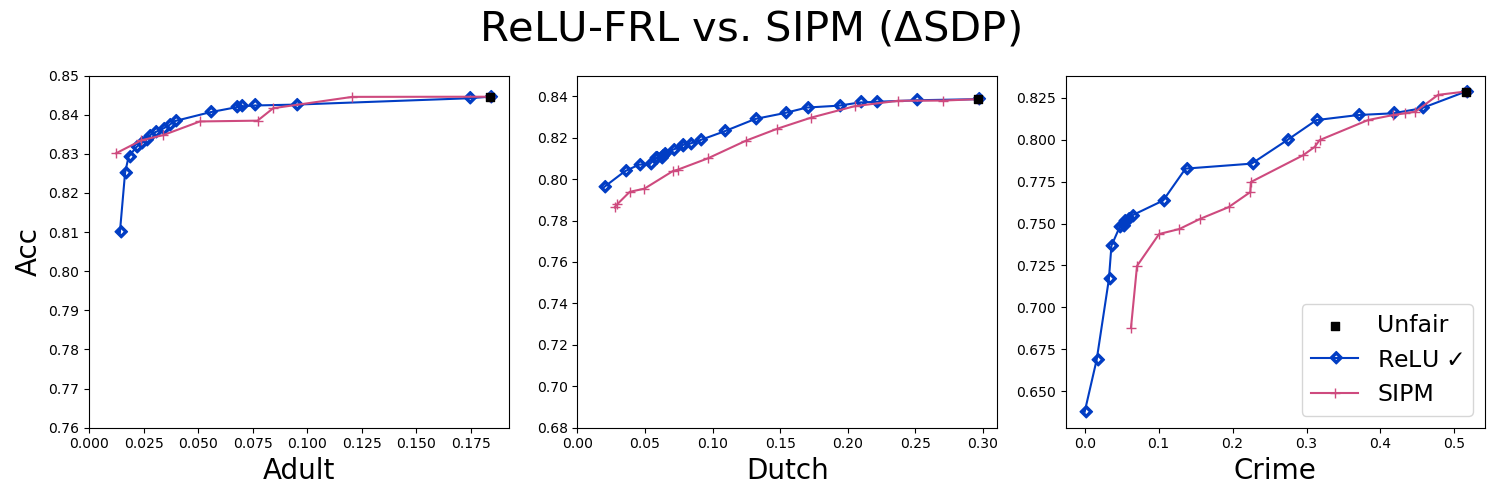}
    \includegraphics[width=0.81\textwidth]{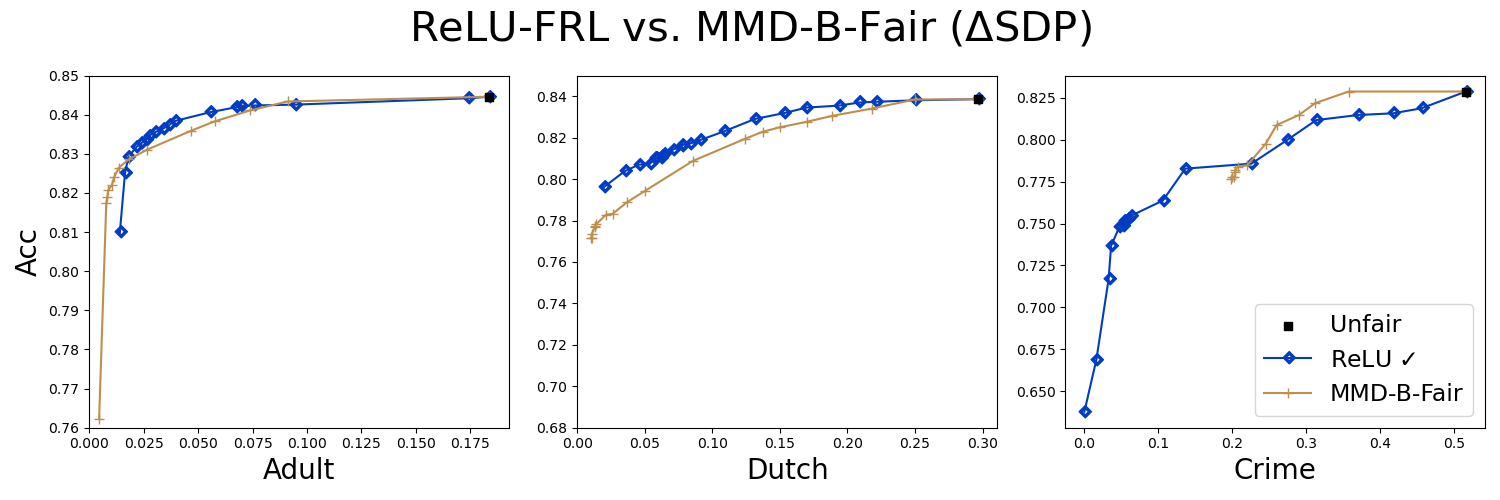}
    \includegraphics[width=0.81\textwidth]{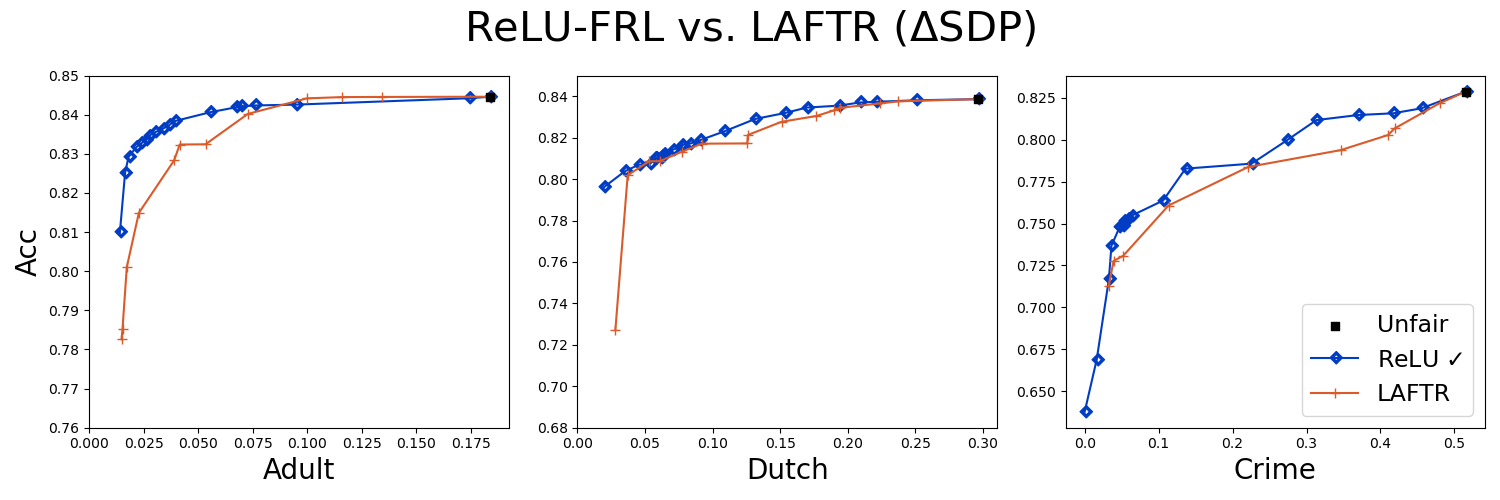}
    \includegraphics[width=0.81\textwidth]{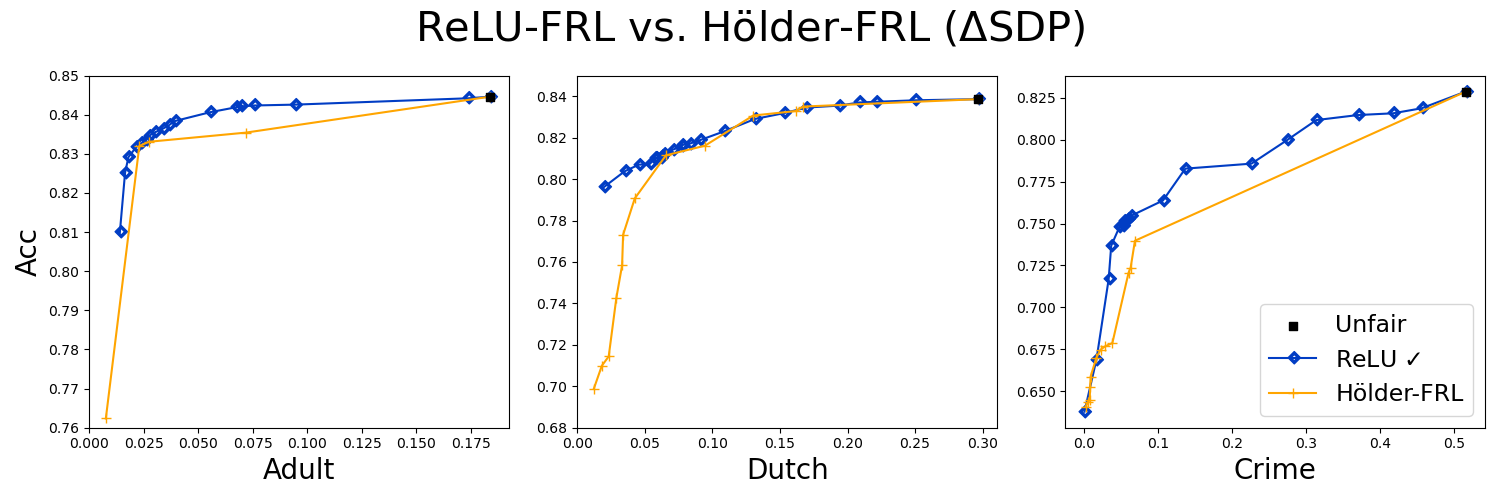}
    \caption{
    \textbf{Single-layered NN (sigmoid activation) prediction head}: Pareto-front lines of fairness level $\Delta \textup{SDP}$ and \texttt{Acc}.
    (Left) \textsc{Adult}, (Center) \textsc{Dutch}, (Right) \textsc{Crime}.
    }
    \label{fig:1_Sigmoid_MLP_sdp}
\end{figure}

\newpage
\vskip 0.2in
\bibliography{reference}

\end{document}